 \documentclass[mnsc]{informs3_no_remark}
\usepackage{macros/packages-or}
\usepackage{macros/editing-macros}
\usepackage{macros/statistics-macros-or}
\usepackage{macros/formatting}

\RequirePackage{tgtermes}
\RequirePackage{newtxtext}
\RequirePackage{newtxmath}
\RequirePackage{bm}
\RequirePackage{endnotes}

\OneAndAHalfSpacedXII 

\usepackage{endnotes}
\let\footnote=\endnote

\usepackage{algorithm}
\usepackage{algpseudocode}
\usepackage{tikz}

\usepackage{natbib}
 \bibpunct[, ]{(}{)}{,}{a}{}{,}%
 \def\bibfont{\small}%
 \def\bibsep{\smallskipamount}%
 \def\newblock{\ }%

\usepackage{hyperref}

\usepackage{pgfplotstable}
\usepackage{graphicx}
\usepackage{subcaption}
\usepackage{float} 
\usepackage{tabularx}
\usepackage{overpic}
\usepackage{tikz,tcolorbox}
\usepackage{rotating}
\usepackage{psfrag}
\usepackage{enumitem}

\EquationsNumberedThrough    

\TheoremsNumberedThrough     
\ECRepeatTheorems  %

 \MANUSCRIPTNO{}

\begin{document}


\RUNAUTHOR{}

\RUNTITLE{Data-Driven Stochastic Modeling Using Autoregressive Sequence Models}

\TITLE{Data-Driven Stochastic Modeling Using Autoregressive Sequence Models: Translating Event Tables to \\ Queueing Dynamics}

\ARTICLEAUTHORS{%
\AUTHOR{Daksh Mittal}
\AFF{Columbia University, Graduate School of Business, \EMAIL{dm3766@gsb.columbia.edu}}

\AUTHOR{Shunri Zheng}
\AFF{Columbia University, Graduate School of Business,  \EMAIL{sz3091@columbia.edu}}

\AUTHOR{Jing Dong}
\AFF{Columbia University, Graduate School of Business,  \EMAIL{jing.dong@gsb.columbia.edu}}

\AUTHOR{Hongseok Namkoong}
\AFF{Columbia University, Graduate School of Business,  \EMAIL{namkoong@gsb.columbia.edu}}
} 

\ABSTRACT{%
While queueing network models are powerful tools for analyzing service systems, they traditionally require substantial human effort and domain expertise to construct. To make this modeling approach more scalable and accessible, we propose a data-driven framework for queueing network modeling and simulation based on autoregressive sequence models trained on event-stream data. Instead of explicitly specifying arrival processes, service mechanisms, or routing logic, our approach learns the conditional distributions of event types and event times, recasting the modeling task as a problem of sequence distribution learning. We show that Transformer-style architectures can effectively parameterize these distributions, enabling automated construction of high-fidelity simulators. As a proof of concept, we validate our framework on event tables generated from diverse queueing networks, showcasing its utility in simulation, uncertainty quantification, and counterfactual evaluation. Leveraging advances in artificial intelligence and the growing availability of data, our framework takes a step toward more automated, data-driven modeling pipelines to support broader adoption of queueing network models across service domains.



}%




\KEYWORDS{Stochastic modeling, Autoregressive sequence models, Simulation, Uncertainty Quantification} 

\maketitle

\section{Introduction}
\label{sec:introduction}


Queueing network models are powerful mathematical frameworks for understanding and managing congestion in a wide range of service systems, such as call centers~\cite{GansKoMa03, Koole10}, healthcare delivery systems~\cite{GreenSaSe06,liu2014panel}, and ride-sharing platforms~\cite{banerjee2015pricing}. By capturing the stochastic nature of demand and service processes in these resource-constrained environments, queueing models enable us to analyze system performance under uncertainty, evaluate trade-offs between service quality and resource utilization, and test the impact of different operational policies. When closed-form analysis is infeasible, stochastic simulation provides a flexible and scalable tool to estimate and optimize key performance metrics. Such simulation models also serve as the computational backbone of `digital twins', especially in congestion-prone settings like hospitals, manufacturing facilities, or logistics hubs. These virtual replicas of real-world service systems support continuous monitoring, forecasting, and what-if analysis in dynamic and data-rich environments \citep{tao2018digital}.

Despite their capabilities, the broader adoption of queueing network models 
remains limited due to challenges related to accessibility and scalability. In particular, developing an appropriate queueing model for a complex service system is a highly non-trivial task that requires substantial training and expertise in stochastic modeling. For instance, modeling patient flow in hospitals has been the focus of many studies~\citep{armony2015patient, shi2016models, dong2020queueing}, each grappling with the intricacies of capturing system-specific dynamics and operational constraints. In this work, we investigate how increasingly widespread machine learning (ML) engineering infrastructures can be leveraged to lower the technical barriers to building queueing network models across different service systems.

Modern information systems have converted what were once anecdotal and fragmented logs into rich, structured event tables that systematically capture sequences of discrete events (transitions) along with their timestamps ~\cite{GansKoMa03, BertsimasMi07}. For example, with the widespread adoption of electronic health records, detailed event tables can now be readily extracted to represent patient encounters across various stages of care (see, e.g., Figure \ref{fig:introducing_event_table}). 
These high-resolution data streams invite a fundamental rethinking of how we build queueing-network simulators. Instead of hand-crafting a structural model and coding a discrete-event simulator, we can now learn a system’s dynamics directly and entirely from the data.

\begin{figure}[t]
    \centering
    \includegraphics[width=\textwidth]{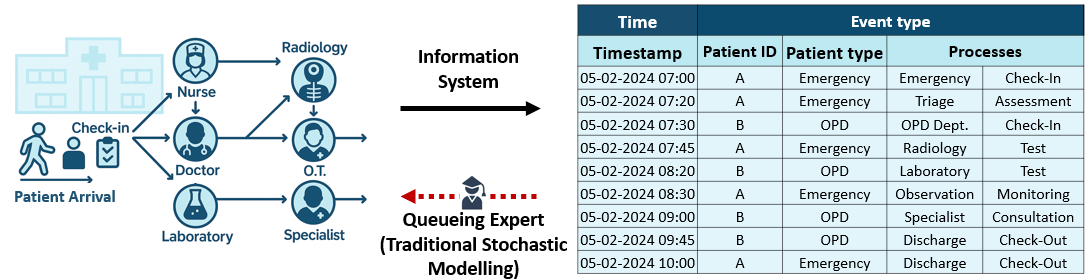}
    \caption{Example of a hospital information system that records operational data as event tables.}
    \label{fig:introducing_event_table}
\end{figure}

We propose a new approach termed {\em data-driven stochastic modeling with autoregressive sequence models}: a generative framework that maps an event history to a distribution over the next event type and its occurrence time.
In particular, inspired by the success of autoregressive sequence models such as Transformers in modeling sequential data like natural language~\cite{VaswaniEtAl17, BrownEtAl20}, we investigate how these architectures can be adapted to the event-stream setting that dominates operations management. Once trained, the model functions as a black-box simulator, generating realistic trajectories without requiring explicit knowledge of queue lengths, service disciplines, or routing rules.

\begin{figure}[t]
\centering
\includegraphics[width=\textwidth]{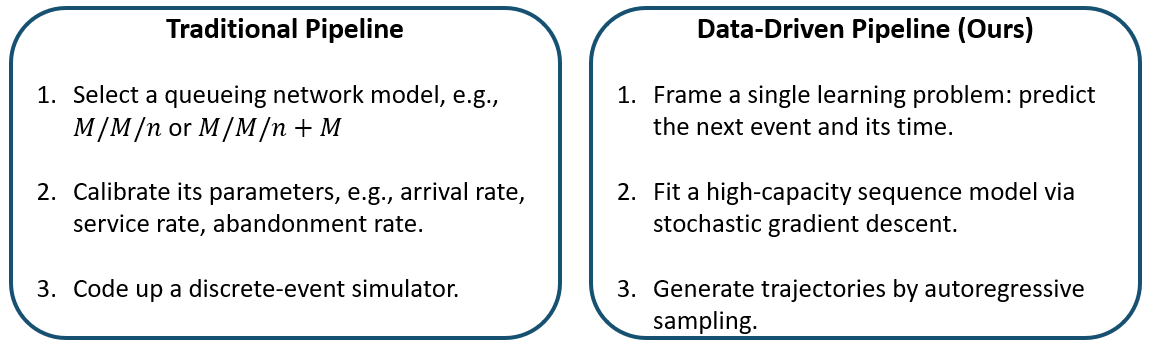}
\caption{Comparison between traditional and data-driven modeling pipelines.}
\label{fig:pipeline_comparison}
\end{figure}

Figure \ref{fig:pipeline_comparison} compares our proposed framework to the classic simulation methods. Note that the classical methods parameterize the system by structure, while we parameterize it by its predictive distribution over the next event. 
This shift yields three practical benefits:
\begin{enumerate}
\item {\bf Expressive modeling.} The learned simulator naturally captures non-Markovian and cross-resource dependencies that are hard to encode analytically.

\item {\bf Built-in uncertainty quantification.} Autoregressive sampling produces full predictive distributions, enabling rigorous assessment of uncertainty.

\item {\bf Versatility.} A single trained model can support what-if analysis, policy evaluation, and policy optimization without retraining.
\end{enumerate}

Prior data-driven work (e.g., \cite{BaronKrSeSh24,SherzerBaKrRe24,OjedaCvGeBaScSa21,KyritsisDe19}) typically learns a direct mapping from design variables to performance metrics. Consider, for example, the task of optimizing staffing levels. A typical model-free approach might attempt to learn a direct functional mapping from staffing levels to performance outcomes. Such surrogates are useful when the objective is fixed, but they discard the important structural information and struggle to incorporate uncertainty. By instead learning a full generative model, effectively a discrete-event simulator, we preserve the full structure of the underlying dynamics and can flexibly support a wide range of downstream tasks.

Naturally, this flexibility comes at a cost: our method trades modeling effort for data requirements. Training a high-capacity sequence model is viable only when rich, high-resolution event streams are available, a condition that modern information systems increasingly satisfy. In settings where real-world data is limited, one promising alternative is to generate synthetic event data using high-fidelity simulators. This opens the door to pretraining foundation models of service systems on large-scale synthetic datasets, which can then be fine-tuned to specific environments with smaller amounts of domain-specific data.

Finally, we note that training sequence models like Transformers, once thought to require deep expertise and substantial computational resources, can now be achieved with modest infrastructure and minimal setup. Modern software libraries like PyTorch, JAX, and Hugging Face provide streamlined APIs, allowing users with basic coding experience to implement training pipelines. Combined with affordable cloud-based hardware, these tools have made state-of-the-art sequence modelling very accessible and economical. In our experiments, for example, we trained a 1.5 million-parameter Transformer on a single GPU in 10–30 hours. Leveraging cloud-based GPU resources, this can be done at a cost of just \$10–\$30. This accessibility enables researchers and practitioners to build and iterate on data-driven autoregressive simulators within the practical constraints of typical academic and industry projects. 




In this work, we investigate how autoregressive sequence models can be adapted to the event-stream data to automate queueing network modeling and simulation. We also demonstrate how the trained model can be applied to various downstream tasks. The contribution of our work can be summarized as follows:

\noindent{\bf Conceptual framework.} We formalize data-driven stochastic modeling with autoregressive sequence models. We recast a queueing network not by its structural primitives, e.g., arrival process, service mechanism, or routing matrix, but by the conditional distributions of the next event type and next event time. This reframing moves model‐building from expert-specified logic to distribution learning. We show how Transformer-style architectures can parameterize the above distributions.

\noindent{\bf Theoretical guarantees.} We provide a theoretical analysis of the performance of a trained sequence model relative to the true data-generating process. Our results show that the model’s utility, whether for performance evaluation, uncertainty quantification, or other downstream tasks, is fundamentally governed by the training loss it aims to minimize. Put differently, the effectiveness of the sequence model ultimately depends on our ability to optimize it during training. Encouragingly, modern machine learning practices are well-suited to minimizing such losses at scale. 

\noindent{\bf Empirical validation.} We conduct a rigorous empirical evaluation of our approach across a range of downstream tasks. We begin by validating it on canonical Markovian parallel server queues, and then demonstrate its scalability to more complex settings involving general service-time distributions, non-stationary arrivals, customer abandonment, and tandem network structures. We further evaluate the method’s ability to quantify uncertainty in performance metrics when the system’s transition dynamics are uncertain. We also showcase its capacity for counterfactual analysis, a key prerequisite for policy optimization in service operations. Finally, we assess the computational and data requirements of our framework and present preliminary experiments that motivate a foundation-model perspective on queueing systems.

The remainder of the paper is organized as follows. Section \ref{sec:related_work} reviews related literature to contextualize and position our contribution within the existing body of work. Section \ref{sec:probelm_formulation} formalizes the problem. In Section \ref{sec:methodology}, we present our predictive framework and illustrate how it works through an M/M/1 queue.  Section~\ref{sec:generalized_formulation} extends the formulation to incorporate uncertainty quantification and control policy considerations, and showcase the sequence model's ability to assess uncertainty. Section~\ref{sec:Approach Validation} provides a comprehensive empirical evaluation across diverse queueing systems, including multi-class Markovian queues, queues with general interarrival and service time distributions, non-stationary settings, and a call center case study featuring tandem structure and customer abandonment. Section \ref{sec:simulating_counterfactuals} demonstrates the sequence model's capability for counterfactual simulation. Section \ref{sec:computational_requirements} discusses the computational and data requirements for implementing our method, and presents preliminary results motivating a foundation-model approach to queueing systems. Finally, Section \ref{sec:conclusion} concludes the paper.

Throughout the paper, we adopt the following notational conventions. Random variables are written in uppercase letters, e.g., $X, S, T, E$, while the corresponding lowercase letters, e.g., $x, s, t, e$, denote their realizations. Calligraphic symbols are reserved for sets; for example, $\mathcal S$ denotes the state space and $\mathcal E$ denotes the event space. The cardinality of a finite set $\mathcal A$ is written as $|\mathcal A|$. Given a finite or countably infinite sequence $(x_i)_{i\ge 1}$ and integers $s\le t$, we define $x_{s:t}:=(x_s, x_{s+1}, \dots, x_{t})$ as the subsequence from index $s$ to $t$, inclusive. In continuous time we write $x_{s:t}:=\{x(u): s \le u \le t\}$. The true data-generating distribution over complete trajectories is denoted by $\mathbb P$, whereas learned or surrogate distributions carry a caret, e.g., $\what{P}_\phi$ for parameters $\phi$. Given an observed history $x_{0:t}$, the conditional distribution of future quantities is written as $\mathbb P(\,\cdot \mid x_{0:t})$.  Expectations, denoted by $\mathbb{E}[\cdot]$, are taken with respect to $\mathbb{P}$ on the measurable space $(\Omega, \mathcal{F})$ unless specified otherwise. The indicator of an event $A$ is written $1(A)$.

\section{Related Work}
\label{sec:related_work}

Our work relates to a rich literature on developing queueing network models tailored to the analysis of diverse service systems (see, e.g., \cite{shi2016models,armony2015patient,brown2005statistical,koole2002queueing,banerjee2015pricing,cont2010stochastic}). These studies emphasize the importance of using data to guide model development in order to capture the specific operational characteristics of systems ranging from call centers and healthcare settings to ride-sharing platforms and financial services. However, they largely adhere to the traditional modeling pipeline, in which a queueing expert must select and calibrate a network model based on domain knowledge. In contrast, our work introduces a new paradigm for constructing high-fidelity models that greatly reduce the reliance on queueing expertise.

With recent advances in machine learning, there has been growing interest in applying these techniques to model queueing systems and predict queue-related performance metrics. For example, \cite{BaronKrSeSh24, SherzerBaKrRe24} use machine learning to estimate the distribution of the number of customers in the system for queues with general arrival processes and service time distributions. \cite{OjedaCvGeBaScSa21} proposes a deep generative model for service times, while \cite{KyritsisDe19, ang2016accurate} develop machine learning algorithms to predict waiting times. \cite{SharafatBa21} employs deep learning to forecast patient flow rates in emergency departments. A common limitation of these methods is their focus on specific performance targets, such as waiting time or service duration. As a result, each new objective often requires training a separate model. In contrast, a simulation model captures the full system dynamics, allowing flexible adaptation to a wide range of downstream objectives without retraining. Several papers learn stochastic arrival processes with neural tools.  \cite{WangJaHo20} estimate latent intensities for Cox (doubly stochastic) Poisson processes using deep latent models. \cite{ZhengZhZh25}  couple a Wasserstein GAN for nonstationary, multidimensional rate processes with a classical Monte Carlo Poisson simulator, providing consistency and convergence‑rate guarantees under a Wasserstein training objective. These works focus on arrivals; in contrast, we learn the full event stream (arrivals, service starts/completions, routing, abandonment) with a single autoregressive model, enabling uncertainty quantification and counterfactuals without hand‑crafting simulator components. Recent work has also applied deep learning to queueing control problems \cite{AtaHa2025, DaiGl2022}. Our work complements these approaches by learning the underlying model upon which control strategies can be based.

Our work is closely related to \cite{senderovich2014queue, SenderovichWeGaMa15, baron2025queueing}, as all aim to develop data-driven, automated approaches to queueing network modeling with minimal reliance on queueing expertise. However, our methodology differs fundamentally from these prior efforts. The works \cite{senderovich2014queue, SenderovichWeGaMa15} build on process mining techniques, which combine data mining and process analysis to extract process models directly from event logs without prior knowledge of the underlying workflow \cite{Van12, van16, martin2016use}. These approaches typically rely on predefined rules to infer network structure and lack built-in capabilities for uncertainty quantification, counterfactual simulation (unless paired with a separate simulator), or meta-learning across different systems. Meanwhile, \cite{baron2025queueing} introduces structural causal models for queues and applies G-computation to estimate causal effects \cite{van2018targeted}. While powerful for causal analysis, this approach does not support meta-learning across diverse environments and can struggle to generalize to evaluate new interventions without substantial additional modeling effort. Our work is also related to \cite{du2016recurrent}, which models event sequences using a recurrent marked temporal point process. In particular, it uses a recurrent neural network to model the intensity function of the point process as a nonlinear function of the history. In contrast, our framework directly models the conditional distributions, enabling straightforward autoregressive sampling to support various downstream tasks.
 Similar in spirit, \cite{ZhuLiZh23} introduce a neural‑assisted sequential simulator trained by matching simulated and empirical joint distributions in Wasserstein distance. In contrast to this “neural‑in‑the‑loop” design, our approach fits a single autoregressive model to event–time sequences, conditioning on the full history; it parameterizes both the event-type and timing distributions and supports downstream performance analysis, uncertainty quantification, and counterfactual simulation.

Autoregressive sequence models have long been employed to predict structured temporal processes. Classical statistical methods such as ARIMA and GARCH \cite{hamilton2020time} impose parametric assumptions on the conditional dynamics of stationary, regularly sampled time series. More recent machine learning approaches, including Random Forests \cite{Breiman2001} and Gradient Boosting Machines \cite{Friedman01}, offer flexible, nonparametric alternatives for time series regression tasks, though they typically rely on fixed-length input windows and do not explicitly model sequential dependencies beyond the chosen horizon. Modern deep learning architectures overcome these limitations by operating directly on variable-length sequences and capturing complex temporal dependencies. Recurrent Neural Networks (RNNs) \cite{RumelhartHiWi86}, Long Short-Term Memory (LSTM) networks \cite{HochreiterSc97}, and Temporal Convolutional Networks (TCNs) \cite{BaiKoMe18} incorporate mechanisms such as recurrence or dilated convolutions that allow them to model nonlinearity and long-range dependencies effectively. These architectures form the foundation for modern autoregressive models that predict the next value in a sequence conditioned on its history. Among these, Transformers \cite{VaswaniEtAl17} represent a particularly expressive class of sequence models based on self-attention mechanisms. Their ability to capture global dependencies without recurrence has led to breakthroughs in diverse domains, including natural language processing \cite{BrownEtAl20}, computer vision \cite{dosovitskiy2020image}, and multimodal learning \cite{AlayracRuBo22}. In time series forecasting tasks, Transformers provide flexible representations that capture both short- and long-term patterns with high computational parallelism \cite{LimAr20, GarzaChCa24}. Beyond standard prediction tasks, they have been adapted to model exchangeable sequences \cite{MullerHoArGrHu22, HanYoArPf24, HegselmannBuLaAgJiSo23, GardnerPeSc24, YanZhXuZhChSuWuCh24}, structured tabular data \cite{zhaoBiCh23, Hollmannetal25}, and sequential decision processes \cite{ChenLuAr21, JannerLiLe21}. Our approach builds on the autoregressive principle common to all these models: predicting the next output conditioned on the observed sequence of past events. In this sense, any sufficiently expressive autoregressive sequence model could serve as the backbone of our framework. However, our aim is different. We seek to learn a generative model over sequences comprising both discrete event types and their associated continuous timestamps. This would enable flexible queueing network modeling and simulation.

A key capability enabled by our approach is uncertainty quantification, which is essential for reliable decision-making in service system modeling \citep{cranmer2020frontier}. In particular, quantifying model and parameter uncertainty, often referred to as input uncertainty, has been extensively studied in the simulation and stochastic optimization literature (see, e.g., \cite{song2015quickly, lam2022subsampling, iyengar2023hedging}).
Traditional approaches often require explicit assumptions about the input distributions and system structure.
In contrast, our method approaches uncertainty quantification from a generative, data-driven perspective. As we will demonstrate later, sequence models are naturally suited to this task: they offer a flexible way to learn and represent uncertainty directly from observed event-stream data.
The use of sequence modeling for uncertainty quantification has deep theoretical roots, tracing back to De Finetti’s seminal work on exchangeable sequences, which shows that such sequences can be represented as mixtures of i.i.d. processes \cite{DeFinetti33, DeFinetti37, DeFinetti17, Aldous85}. This foundational idea gives rise to a probabilistic interpretation of uncertainty as arising from the latent or unobserved structure, and establishes the equivalence between uncertainty stemming from unobserved latent structures and uncertainty arising from unobserved future data~\cite{BertiPrRi04, BertiDrLePrRi22, FongHoWa23}. Extensions to partial exchangeability, which include mixtures of Markov chains or Markov processes, further enable uncertainty modeling for more general stochastic dynamics \citep{ DiaconisFr80, Freedman96, Aldous85, Banxico14}. For comprehensive overviews, see \citep{Aldous85, Lauritzen16, fortinietal24}.

\section{Problem Formulation}
\label{sec:probelm_formulation}

Operational data are typically available in the form of event tables, which record a sequence of events along with their corresponding timestamps. Let the observed sequence of events and inter-event times be denoted by
\[
\{(E_i,T_{i}) , i =1, ... \},
\]
where \(E_i \in \mathcal{E}\) denotes the type of the $i$th event, and \(T_i\) is the time elapsed between event $i-1$ and event $i$. Note that these inter-event times can be readily computed from the event timestamps. 

Although we only observe event table data, we assume the existence of an {\em underlying stochastic process} that generates these observations. Let this process be denoted by $\{X(t):t\in [0,\infty)\}:= X_{0:\infty}$, defined on a probability space $(\Omega, \mathcal{F}, \mathbb{P})$. For each $t \in [0, \infty)$, the random variable $X(t): \Omega \to \mathcal{S}$ represents the state of the system at time $t$, taking values in a measurable space $(\mathcal{S}, \Sigma)$. The evolution of the process is governed by a probability measure $\mathbb{P}$ on the path space $(\mathcal{S}^{[0,\infty)}, \mathcal{G})$, where $\mathcal{G}$ denotes the appropriate $\sigma$-algebra. We denote this as $X_{0:\infty} \sim \mathbb{P}$. For any $t \ge 0$, and any observed sample path $x_{0:t}:=\{x(s): 0\leq s\leq t\}$, we let $\P(\cdot|x_{0:t})$ denote the conditional probability distribution of future trajectories, i.e., $X_{t:\infty} \sim \P(\cdot|x_{0:t})$.

In many service systems, the underlying stochastic process takes the form of a {\em jump process}, where the system transitions between discrete states at random time points. Specifically, the system starts in an initial state $S_1$, remains there for a random duration $T_1$, then jumps/transitions to a new state $S_2$, remains there for a duration $T_2$, and so on. This process can equivalently be represented as a sequence of state-duration pairs $\{(S_i,T_i), i =1,2, \cdots \}$. The evolution of this process can be characterized by the following two conditional probability distributions:
\begin{equation}
    \P(T_{n}=t |S_1, T_1, S_2, \cdots,T_{n-1},S_{n}) ~ \mbox{ and } ~   \P(S_{n+1} =s|S_1, T_1, S_2, \cdots ,S_{n},T_n).
\end{equation}
The first expression gives the distribution of the inter-event time
$T_{n}$ (i.e., the time the system spends in state $S_{n}$) conditioned on the full history $\{S_1, T_1, S_2, \cdots,T_{n-1},S_{n}\}$. The second gives the distribution of the next state 
$S_{n+1}$, again conditioned on the complete observed history including time $T_n$. Importantly, we do not assume any structural properties such as the Markov property. The process is allowed to exhibit full path dependence, making this framework broad and flexible enough to capture a wide range of real-world service system dynamics.

Moreover, there is a direct correspondence between the jump process representation and the event table. Event tables can be viewed as a simplified encoding of the jump process, where states can be deduced from discrete events.  Specifically, the state of the system $S_{n+1}$ after the $n$th jump can be deterministically inferred from the initial state $S_1$ and the sequence of past events $(E_1,...,E_n)$. 
This connection allows the transition dynamics of the underlying process to be re-expressed directly in terms of the observed event stream:
\begin{equation}
    \P(T_{n} | S_1,T_1,E_1, \cdots,T_{n-1},E_{n-1}) ~ \mbox{ and } ~ \P(E_{n} |S_1,T_1,E_1, \cdots,E_{n-1}, T_{n}).
\end{equation}
Thus, the three representations of the stochastic process governed by $\P$, namely, the continuous-time trajectory $X_{0:\infty}$, the jump process $\{S_1,T_1, S_2,T_2, S_3\cdots\}$, and the event table $\{S_1, T_1, E_1, T_2, E_2, \cdots\}$, are mathematically equivalent. Among these, the event table representation offers two key advantages. First, the event space $\mathcal E$ is typically much smaller than the state space $\mathcal S$,  making the representation more compact and computationally tractable. Second, event tables are easier to record and interpret in practice, which makes them especially well-suited for real-world service systems. The event table representation is also known as discrete-event systems \citep{cassandras2008introduction, glynn2002gsmp, glasserman1991gradient}.

Given this equivalence, we adopt the event table as the primary representation throughout the remainder of this paper. The other two representations will be used interchangeably when appropriate, for clarity or analytical convenience.
With a little abuse of notation, we define the autoregressive conditional distributions as:
\[
    \P_n^{time}(T_n) = \P(T_{n} |S_1, T_{1:n-1}, E_{1:n-1})
~\mbox{ and }~
    \P_n^{event}(E_n) = \P(E_{n}|S_1, T_{1:n-1}, E_{1:n-1}, T_n).
\]
That is $\P_n^{time}$ and $\P_n^{event}$ are the conditional probability distributions of the $n$th inter-event time and event type, respectively, given the realized initial states, inter-event times, and event types before step $n$. While the underlying conditional distributions are random objects, we treat these as deterministic functions evaluated at specific input histories.

Then, the joint distribution over the initial state, inter-event times, and event types up to the $n$th event can be written as a product of the autoregressive conditional distributions:
\begin{align}
\P(S_1,  T_{1:n}, E_{1:n}) &=  \P_0(s_1) \prod_{i=1}^{n} \P_{i}^{time}(T_{i})  \P_{i}^{event}(E_{i}).
\end{align}

In Figure \ref{fig:state-event-inter-event-times}, we illustrate a service system whose dynamics follow those of an M/M/1 queue. 
In such a system, given the initial state, the system state at any time $t$, denoted by $X(t)$ (i.e., the number of customers in the system) can be fully determined by the sequence of arrival and departure events, i.e.,
\[
X(t)=S_{N(t)}=S_1 + \sum_{i=1}^{N(t)} 1\{E_i = \text{arrival}\} - 1\{E_i = \text{departure}\},
\]
where $N(t)=\max\{n\geq 1: \sum_{i=1}^{n} T_i\leq t\}$.
Notably, while the state space of the system is $\mathcal{S} = \mathbb{Z}^+$, the event space is significantly smaller: $\mathcal{E} = \{\text{arrival, departure}\} $. 

\begin{figure}[h!]
    \centering
    \includegraphics[width=\textwidth]{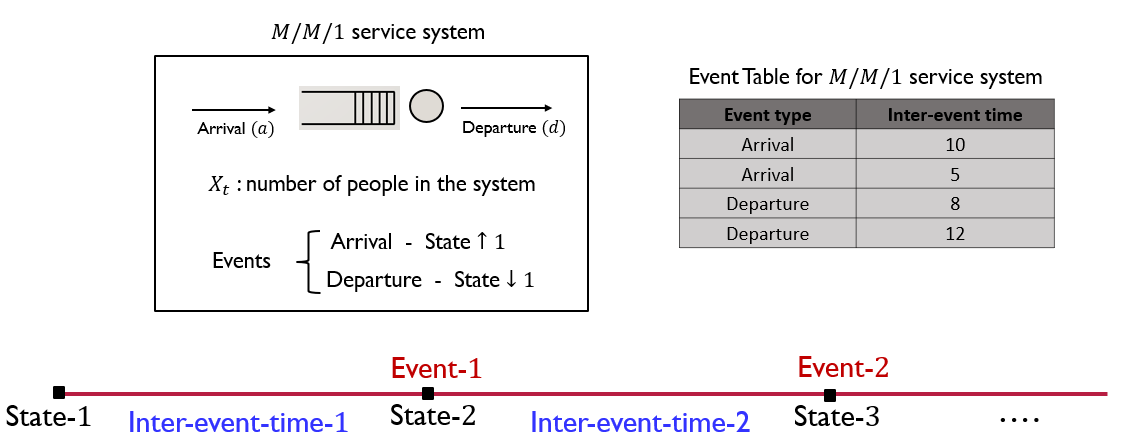}
    \caption{Example of M/M/1 service system with the corresponding event table. The state of the M/M/1 service system can be determined from corresponding events. An M/M/1 service system has only two events: arrival and departure.}
    \label{fig:state-event-inter-event-times}
\end{figure}

In Section \ref{sec:generalized_formulation}, we extend the framework to account for both parameter uncertainty in the system dynamics and the influence of a control policy \(\pi\). Specifically, we assume the parameters \(\theta\), which govern the behavior of the queueing process, are drawn from a prior distribution \(\mu\). In addition, the system evolution is influenced by a policy \(\pi\), which represents the actions taken by the service system operator. The setting discussed thus far can be viewed as a special case of this more general formulation, in which the parameter is fixed at \(\theta = \theta_0\) (i.e., \(\mu = \delta_{\theta_0}\)), and the policy is fixed at \(\pi = \pi_0\).


The primary goal of a stochastic modeler is to simulate plausible trajectories of the service system that are consistent with the underlying probability law. Specifically, we aim to generate samples of the form $\left(S_{1}, E_{1:\infty}, T_{1:\infty}\right) \sim \P(\cdot)$. Once such simulations are possible, potentially under different counterfactual scenarios, a data scientist can use them for prediction or other downstream analyses. These include estimating key performance metrics (e.g., wait times, throughput) or assessing the impact of alternative policies. 


In the conventional modeling approaches, the modeler first specifies a stochastic model $\tilde{\P}$ based on a combination of domain knowledge and available data. Then, a discrete-event simulator is constructed to generate trajectories of the service system under the identified model. While widely used, this approach relies heavily on expert knowledge, particularly in queueing theory, to accurately specify the underlying stochastic model. As noted earlier, this reliance can limit its applicability for practitioners without specialized training in queueing. Our approach addresses this limitation by adopting a predictive framework for stochastic modeling. In particular, instead of manually defining a model based on domain knowledge, we learn to predict future system behavior directly from data by learning the conditional distributions $\P_n^{time}$'s and $\P_n^{event}$'s. 



\section{Methodology}

\label{sec:methodology}


We adopt an autoregressive predictive approach to stochastic modeling. 
We parameterize and learn the conditional predictive distribution of the next event, i.e., $\P_n^{time}(t)$ and $\P_n^{event}(e)$ for all $n\geq 1$. This contrasts with conventional stochastic-network modeling, where the focus is on calibrating structural parameters of the system, such as the number of servers or the distributions of interarrival and service times.

Let the conditional distributions be parameterized by $\phi$. The predictive distributions at step n are given by
\[
    \what{P}_{\phi,n}^{time}(T_n) = \what{P}_\phi({T}_{n} |{S}_1, {T}_{1:n-1}, {E}_{1:n-1})
\]
and
\[
    \what{P}_{\phi, n}^{event}(E_n) = \what{P}_\phi({E}_{n}|{S}_1, {T}_{1:n-1}, {E}_{1:n-1}, {T}_n), \quad n=2,3,\dots.
\]
The resulting joint predictive model over event sequences and their associated inter-event times is given by 
\[\what{P}_\phi ({S}_1 ,  {T}_{1:n}, {E}_{1:n}) =  \what{P}_{\phi,0}(S_1) \prod_{i=1}^{n} \what{P}_{\phi,i}^{time}(T_{i})  \what{P}_{\phi,i}^{event}(E_{i}) .\]
This autoregressive factorization eliminates the need for manual model specification and enables efficient simulation by sampling from the learned conditional distributions. 


The model $\what{P}_\phi$ is trained to approximate the true data-generating distribution $\P$
by minimizing the Kullback–Leibler (KL) divergence between $\what{P}_\phi$ and $\P$, leading to the following optimization problem:
\[\min_\phi \dkl{\P(\cdot)}{\what{P}_\phi(\cdot)}.\] 
%
%
This objective is equivalent, up to an additive constant independent of $\phi$, to minimizing the negative expected log-likelihood of the model under the true distribution:
\begin{align}
     \dkl{\P(\cdot)}{\what{P}_{\phi}(\cdot)} & \propto - \E_{(S_{1}, E_{1:\infty},T_{1:\infty}) \sim \P(\cdot)} \left[ \log \left(\what{P}_\phi \left({S}_1, {E}_{1:\infty}, {T}_{1:\infty}\right)\right)\right] \nonumber\\
     &\propto -\E_{S_1\sim \P_0} \left[\log\what{P}_{\phi,0}({S}_1) \right]  \nonumber \\ & \quad  +\underbrace{\sum_{i=1}^{\infty} -\E_{T_i\sim \P^{time}_{i}}\left[\log\what{P}_{\phi,i}(T_i) \right]}_{\text{Inter-event time losses}} + \underbrace{\sum_{i=1}^{\infty} -\E_{E_i\sim \P_{i}^{event}}\left[\log\what{P}_{\phi,i}(E_i) \right]}_{\text{Event type losses}},
     \label{eq:different_losses}
\end{align}

As noted in Section~\ref{sec:introduction}, recent advances in deep learning have made the implementation and training of autoregressive sequence models widely accessible. With modern hardware and software libraries, these models can now be trained efficiently at scale. We further examine the operational and computational considerations involved in deploying such models in Section~\ref{sec:operationalizing_predictive}.

\subsection{Training and simulation}
As discussed in Section \ref{sec:introduction}, modern service systems increasingly generate rich event-stream data, enabling the direct training of sequence models such as
$\what{P}_\phi$. Once trained, this model can be used to simulate plausible future trajectories of the system (see Figure~\ref{fig:two_stages}). We now describe the training and simulation procedures in more detail.

Suppose we have $K$ event tables generated from the original data-generating process $\P(\cdot)$, where each event table contains $N$ events. We denote the dataset as   \[\left\{\left(S_1^{(j)},T_1^{(j)}, E_1^{(j)}, \cdots, T_{N}^{(j)}, E_{N}^{(j)} \right): 1\leq j\leq K \right\},\] 
where each $j$ corresponds to a distinct event table. For simplicity, we assume the initial state $s_1^{(j)}$ is fixed, although our approach can be readily extended to accommodate varying initial states. 
Define the loss for a single event table as
\begin{align}
&\loss\left(\left(S_1^{(j)},T_1^{(j)}, E_1^{(j)}, \cdots, T_{N}^{(j)}, E_{N}^{(j)} \right),\phi\right)\nonumber\\
&= -\log \what{P}_{\phi,0}(S_1^{j}) - \sum_{i=1}^N \log \left[ \what{P}_{\phi,i}^{time}(T_i^{(j)}) \right] -  \sum_{i=1}^N  \log \left[ \what{P}_{\phi,i}^{event}(E_i^{(j)}) \right]
 \label{eq:loss_quote_everywhere}
\end{align}
The sequence model is trained by minimizing the loss
\begin{equation}\label{eq:mainloss}
\frac{1}{K} \sum_{j=1}^{K} \loss\left(\left(S_1^{(j)},T_1^{(j)}, E_1^{(j)}, \cdots, T_{N}^{(j)}, E_{N}^{(j)} \right),\phi\right).
\end{equation}
%
%
In practice, model parameters $\phi$ are typically optimized using (stochastic) gradient descent.

\begin{figure}
    \centering
    \includegraphics[width=\textwidth]{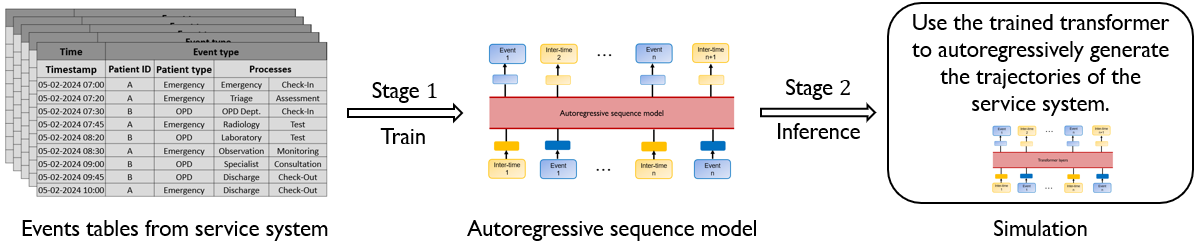}
    \caption{Two stages of implementing our approach. Stage 1: Train the transformer on the events table from the service system. Stage 2: Once trained use it to autoregressively simulate the trajectories of the service system}
    \label{fig:two_stages}
\end{figure}

Once the sequence model is trained, it can be used to generate system trajectories in an autoregressive manner. 
The detailed procedure is outlined in Algorithm \ref{alg:inference-simulation}.
Similarly, if we are provided with a partial trajectory $(S_{1}, T_{1:n}, E_{1:n})$, the model can generate the corresponding future trajectory $({T}_{n+1:N},{E}_{n+1:N})\sim \what{P}_\phi(\cdot|S_1, T_{1:n}, E_{1:n})$.


\begin{algorithm}[ht]
\caption{Simulating service system trajectory using trained sequence model}
\label{alg:inference-simulation}
\begin{flushleft}
\textbf{Input:} Trained sequence model $\what{P}_\phi$, length of the trajectory $N$, initial state $s_1$ 
\begin{enumerate}
    \item \textbf{Initialization:} $\what{\mathcal{H}} = \{s_1\}$
    \item \textbf{for} $i \leftarrow 1$ \textbf{to} $N$ \textbf{do}
    \begin{enumerate}
        \item[3.] Sample ${T}_i \sim \what{P}_{\phi,i}^{time}$
        \item[4.] Update $\what{\mathcal{H}} \leftarrow \what{\mathcal{H}} \cup \{{T}_i\}$
        \item[5.] Sample ${E}_i \sim \what{P}_{\phi,i}^{event}$   
        \item[6.] Update $\what{\mathcal{H}} \leftarrow \what{\mathcal{H}} \cup \{{E}_i\}$
    \end{enumerate}
    \item[7.] \textbf{end for}
    \item[8.] \textbf{Return:} Simulated trajectory  $\what{\mathcal{H}}$
\end{enumerate}
\end{flushleft}
\end{algorithm}




\subsection{Implementation using modern deep learning architectures}
\label{sec:operationalizing_predictive}


Recent advances in deep learning have led to substantial improvements in our ability to model sequential data. Large language models, for example, now achieve state-of-the-art performance on a wide variety of text-based tasks. A variety of autoregressive sequence modeling architectures have been proposed, most notably Transformers~\cite{VaswaniEtAl17}, and more recently, state-space models~\cite{GuD24}. The field continues to evolve rapidly as new designs and training techniques emerge.

Extensive empirical evidence shows that Transformers offer an attractive balance between modeling power and computational efficiency \citep{BrownEtAl20,kaplan2020scalinglawsneurallanguage}.  Their domain-agnostic self-attention mechanism has enabled breakthroughs not only in natural language processing, but also in computer vision \citep{dosovitskiy2020image,dehghani2023} and reinforcement learning \citep{parisotto2019,ChenLuAr21}. Moreover, the availability of mature software libraries, such as PyTorch, JAX, and Hugging Face, has made the implementation and training of Transformers highly accessible, requiring only basic programming proficiency (see Figure \ref{fig:ease-of-implementation}).

The economics of Transformer training and inference have likewise become more favorable. Although both stages typically require GPU acceleration, cloud platforms such as AWS and Vast.ai now offer cost-effective access to GPU instances, eliminating the need for large upfront investments in dedicated hardware. In our experiments, each training run cost approximately \$10-\$30, making the approach financially accessible for a wide range of applications. A detailed discussion of computational requirements is provided in Section~\ref{sec:computational_requirements}.
 
Motivated by its scalability and accessibility, we use the Transformer architecture as an illustrative example of a backbone for our predictive sequence model. While other architectures may also be suitable, we focus on Transformers to keep the exposition and analysis concrete. Crucially, our methodology is model-agnostic and can be adapted to any future deep-learning architecture that offers greater computational efficiency for sequence modeling. 

 \begin{figure}[h!]
  \centering
  \begin{minipage}[b]{0.49\textwidth}   \includegraphics[width=\textwidth, height=6.1cm]{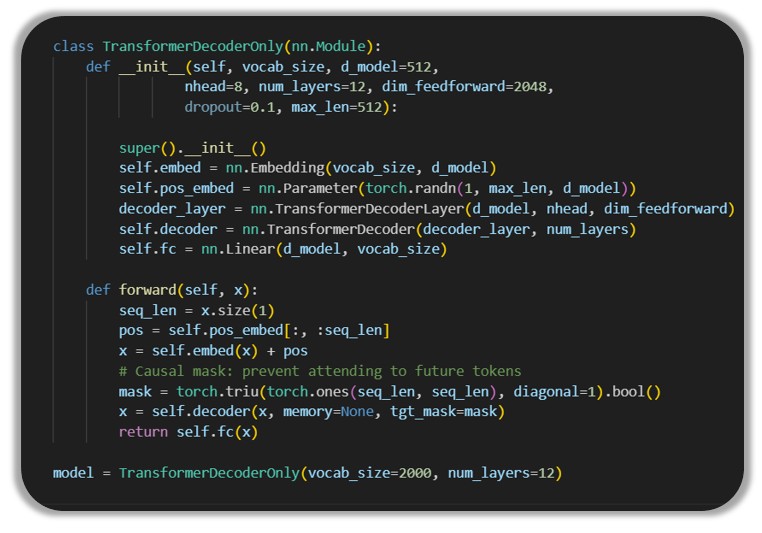}
  \end{minipage}
  \hfill
  \begin{minipage}[b]{0.49\textwidth}   \includegraphics[width=\textwidth, height=6cm]{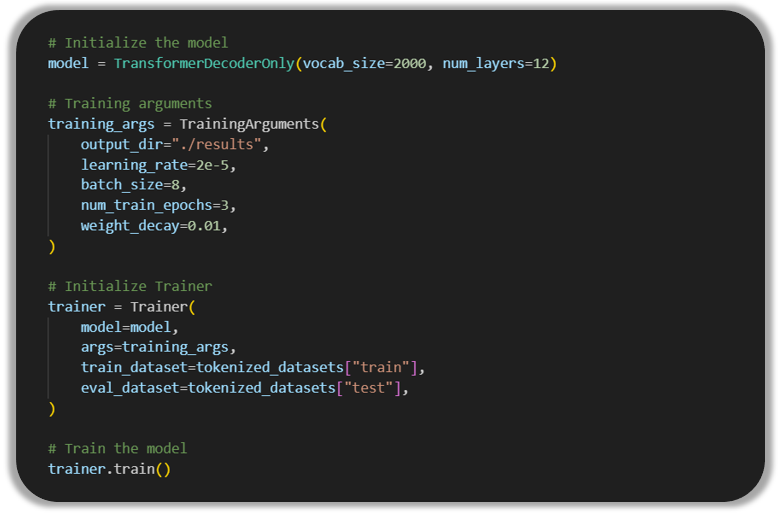}
  \end{minipage}
    \caption{ \textbf{Pseudo-code for Transformer Implementation and Training:}
Implementing and training a transformer takes just a few lines of code, requiring only basic programming proficiency and no specialised knowledge of queueing theory or stochastic modelling.   (left) Pseudo code for implementing a transformer model; (right) Pseudo code for training the transformer model.  }
    \label{fig:ease-of-implementation}
\end{figure}


We next give a brief overview of the Transformer architecture we use, with further implementation details provided in Appendix ~\ref{sec:experimental-details}. We utilize the decoder-only variant of the Transformer architecture \cite{VaswaniEtAl17}, commonly used in large language models. 
Each decoder block consists of masked self-attention and feed-forward layers, enabling the model to generate sequences one token at a time by conditioning on the preceding tokens. The causal masking mechanism ensures that future tokens remain hidden during training, thereby preserving the autoregressive nature of the sequence.
In our setting, the input tokens consist of event types and inter-event times (see Figure~\ref{fig:transformer-architecture}). If the service system involves multiple customer types, this information can be incorporated either by encoding it into the event type token or by introducing a separate token for customer type. To represent the inputs, we use learnable embeddings for event types \citep{mikolov2013distributed}, and Time2Vec embeddings for inter-event times to effectively capture temporal patterns \citep{kazemi2019time2vec}.

\begin{figure}[h]
    \centering
    \includegraphics[width=0.8\textwidth]{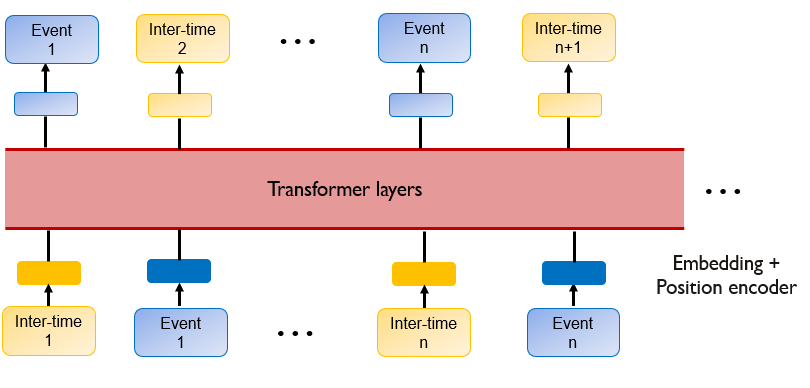}
    \caption{Transformer architecture for stochastic modeling}
    \label{fig:transformer-architecture}
\end{figure}
It is important to note that since the number of event types is finite, their distributions are discrete. This allows us to model them as a classification task, a domain where neural networks excel. However, inter-event times follow continuous distributions, which are inherently more challenging to model using neural networks. To address this challenge and ensure strong performance, we adopt two different implementations based on the information available about the inter-event time distributions:
\begin{enumerate}
    \item \textbf{Known parametric class of inter event-time distributions:} When the inter-event time distribution is known to belong to a specific parametric family, we can train the model to directly predict the parameters of that distribution. For example, if the system is known to be Markovian, inter-event times follow an exponential distribution, and the model can be trained to estimate the rate parameter as a function of the observed history of events and times. 
    This approach can simplify the modeling task by leveraging known structural assumptions. However, it also requires domain knowledge about the system, which may not always be available or reliable in practice. 
    \item \textbf{Unknown parametric class of inter event-time distributions:} When the inter-event time distributions are unknown or not easily parameterizable, we adopt a distributional form that is amenable to neural network training. 
    Inspired by the success of discretization techniques in \citep{bellemare2017,MullerHoArGrHu22}, we model inter-event times using a discretized continuous distribution, often referred to as a Riemann distribution. This approach approximates a continuous probability density over time by partitioning the support into a finite set of bins, allowing the problem to be treated as a classification task over discretized intervals.
\end{enumerate}

For example, in Markovian systems such as the M/M/1 queue, where inter-event times are exponentially distributed, the model is trained to recover the corresponding rate parameter (see Section~\ref{sec:part_1_approach_validation}). In contrast, for more general settings such as G/G/1 queues, we model inter-event times using discretized Riemann distributions to flexibly approximate arbitrary time distributions (see Section~\ref{sec:Approach Validation}).


\subsection{Example}
\label{sec:part_1_approach_validation}



In this section, we illustrate the proposed method using an M/M/1 queue as a canonical example. A more comprehensive empirical validation across a broader range of settings is provided in Section~\ref{sec:Approach Validation}.


For each problem instance (M/M/1 queue with fixed arrival rate and service rate), we train the transformer model using the loss function in \eqref{eq:mainloss} on a training dataset consisting of $K_t$ sample paths, each containing $N$ events:  
\[\mathcal{D}^{train}\equiv \left\{\left(S_1^{(j)},T_1^{(j)},E_1^{(j)} \cdots, T_{N}^{(j)}, E_{N}^{(j)} \right): 1\leq j\leq K_t \right\},\] 
which is generated using the corresponding M/M/1 discrete event simulator. Detailed information regarding data generation, transformer architecture, and training procedures is provided in Appendix \ref{sec:experimental-details}.
We also generate the test dataset
\[
\mathcal{D}^{test}\equiv \left\{\left(S_1^{(j)}, T_1^{(j)}, E_{1}^{(j)} \cdots, T_{N}^{(j)}, E_{N}^{(j)} \right): 1\leq j\leq  K_e \right\},
\]
using the corresponding discrete event simulator.

Table~\ref{tab:validation_losses} compares the losses achieved by the trained Transformer model on the test dataset with the corresponding `optimal' losses. Specifically, for the trained transformer model, we define:
\begin{itemize}
    \item \textit{Averaged event type loss:} 
    $-  \frac{1}{K_eN}\sum_{j=1}^{K_e}\sum_{i=1}^N  \log \what{P}_{\phi,i}^{event}(E_i^{(j)}).$
        \item \textit{Averaged inter-event time loss:} 
        $ \frac{1}{K_eN}\sum_{j=1}^{K_e}
  \sum_{i=1}^N \left( \frac{1}{\what{\lambda}_{\phi,i}^{(j)}} - T_i^{(j)} \right)^2,$
  where $\what{\lambda}_{\phi,i}^{(j)}$ is the rate parameter of the exponential distribution predicted by the transformer for $i$th inter-event time in the $j$th sequence, i.e., $T_{i}^{(j)}$, given the history \((S_1^{(j)}, T_1^{(j)}, \ldots, E_{i-1}^{(j)})\).
\end{itemize}
The `optimal' losses refer to those achieved by an oracle model with complete knowledge of the underlying stochastic process. In particular, they are defined as:
\[
\lim_{N\rightarrow\infty} -  \frac{1}{N}\sum_{i=1}^N  \log {\P}_{i}^{event}(E_i^{(j)}) 
~\mbox{ and }~
\lim_{N\rightarrow\infty} \frac{1}{N} \sum_{i=1}^N \left( \frac{1}{{\lambda}_{i}^{(j)}} - T_i^{(j)} \right)^2.
\]
Closed-form expressions for these optimal losses of the M/M/1 queue are provided in Lemma~\ref{lm:m_m_1_lower_bound} in Appendix~\ref{sec:optimal_losses}.

As shown in Table~\ref{tab:validation_losses}, the validation losses of the trained transformer closely match the optimal loss, suggesting that the transformer effectively learns the true system dynamics. Figure~\ref{fig:M-M-1-distributions} further compares the empirical distributions of inter-arrival times, service times, and waiting times generated by the transformer to their corresponding true distributions. We observe that the distributions of the transformer-simulated data closely match the true underlying distributions.

\begin{table}[h!]
\centering
\begin{tabular}{|c|c|c|c|c|c|c|}
\hline
$\lambda$ & $\nu$ & $\rho$ & \makecell{\textbf{Transformer} \\ \textbf{Event Loss}} & \makecell{\textbf{Optimal} \\ \textbf{Event Loss}} & \makecell{\textbf{Transformer} \\ \textbf{Time Loss}} & \makecell{\textbf{Optimal} \\ \textbf{Time Loss}} \\ \hline
0.2       & 1     & 0.2    & 0.2698                                             & 0.2693                                                 & 10.44                                                & 10.42                                                   \\ \hline
0.5       & 1     & 0.5    & 0.4802                                             & 0.4775                                                 & 1.335                                                & 1.333                                                   \\ \hline
0.8       & 1     & 0.8    & 0.6225                                             & 0.6178                                                 & 0.4378                                               & 0.434                                                  \\ \hline
1         & 1     & 1      & 0.6857                                             & 0.6823                                                 & 0.2632                                               & 0.2617                                                  \\ \hline
2         & 1     & 2      & 0.6363                                             & 0.6364                                                 & 0.1113                                               & 0.1111                                                  \\ \hline
\end{tabular}
\caption{Comparison of test-set losses from the transformer and optimal losses from oracle for different arrival rates ($\lambda$), service rates ($\nu$), and traffic intensities ($\rho$) in an M/M/1 queue.}
\label{tab:validation_losses}
\end{table}

 \begin{figure}[h!]
  \centering
  \begin{minipage}[b]{0.325\textwidth}   \includegraphics[width=\textwidth, height=3.5cm]{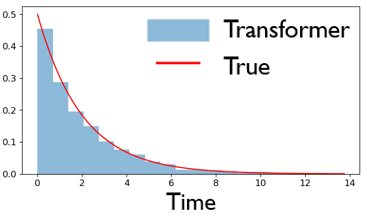}
  \centering{Inter-arrival time}
  \end{minipage}
  \hfill
  \begin{minipage}[b]{0.325\textwidth}   \includegraphics[width=\textwidth, height=3.5cm]{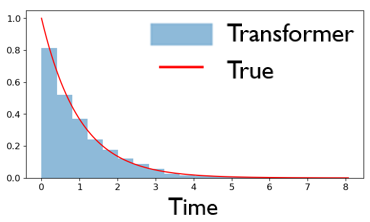}
  \centering{Service time}
  \end{minipage}
  \hfill
  \begin{minipage}[b]{0.325\textwidth}   \includegraphics[width=\textwidth, height=3.5cm]{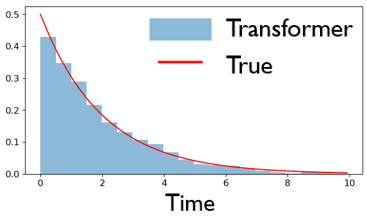}   
  \centering{Positive waiting time}
    \end{minipage}
    \caption{ \textbf{Markovian service system:} Comparison of performance measure distributions predicted by the sequence model (transformer) with the true distributions derived from the underlying M/M/1 queuing network  ($\lambda=0.5$, $\nu=1$). The performance measures include inter-arrival time, service time, and positive waiting time.}
    \label{fig:M-M-1-distributions}
\end{figure}


\section{Generalizing the formulation to include parameter uncertainty and control policy}
\label{sec:generalized_formulation}
We now present a generalized formulation that extends the setup introduced in Section~\ref{sec:probelm_formulation}. With a slight abuse of notation, let $\P$ denote the probability law governing the underlying stochastic process, such that 
\[
\P(X_{0:\infty}|\pi) = \int P_{\pi,\theta}(X_{0:\infty}) d\mu(\theta),
\]
where $\theta \sim \mu$ is an unobserved parameter representing latent characteristics of the environment (e.g., arrival rate and service rate in an $M/M/1$ queue), and $\pi$ is a policy controlled by the system operator. Given a policy $\pi$ and a latent parameter $\theta$, the transition dynamics are governed by $P_{\pi,\theta}$. This extended formulation induces well-defined conditional distributions over inter-event times and event types. In particular, given historical events up to time $n-1$, we define
\begin{align*}
    \P_n^{time}(T_n|\pi) & = \int P_{\pi,\theta}(T_{n} |S_1, T_{1:n-1},  E_{1:n-1}) ~~ d\mu(\theta|S_1, T_{1:n-1}, E_{1:n-1},\pi), \\
    \P_n^{event}(E_n|\pi) & = \int P_{\pi,\theta}(E_{n}|S_1, T_{1:n}, E_{1:n-1}) ~~ d\mu(\theta|S_1, T_{1:n}, E_{1:n-1}, \pi).
 \end{align*}
The full joint distribution of event sequences under a given policy $\pi$ is then
\begin{align*}
    \P(S_1 ,  T_{1:N}, E_{1:N}|\pi) &=  \P_0(S_1|\pi) \prod_{n=1}^{N} \P_{n}^{time}(T_{n}|\pi)  \P_{n}^{event}(E_{n}|\pi).
\end{align*}

This formulation provides a unified framework for addressing several downstream tasks that a decision maker may be interested in. Examples include predicting future system trajectories under a policy $\pi$, estimating the distribution of key performance metrics under a policy $\pi$ accounting for uncertainty in model parameters, i.e., uncertainty quantification, and counterfactual policy evaluation and optimization.
In particular, given a policy $\pi$ and the relevant history $x_{0:t}$ (under the policy $\pi$), future trajectories $X_{t+1:N}$ can be generated iteratively using the conditional distributions $\P_n^{time}(T_n|\pi)$ and $\P_n^{event}(E_n|\pi)$ for $n=t+1, \ldots N$.

Note that in the conventional stochastic modeling framework, the modeler must explicitly specify the underlying model parameter $\theta$, a prior distribution $\mu$ over $\theta$, and the transition probability $P_{\pi,\theta}$. In addition, given observed data $x_{0:t}$, posterior inference requires computing or sampling from
\[
\mu(\theta|x_{0:t}, \pi) = \frac{ P_{ \pi, \theta}(x_{0:t}) \mu(\theta)}{\int P_{\pi, \theta}(x_{0:t}) \mu(\theta) d\theta}.
\]
Identifying the relevant latent parameter $\theta$ typically requires strong modeling assumptions, e.g., arrival rate and service rate assuming a Poisson arrival process and an exponential service time distribution. Constructing the transition kernel
$P_{\pi,\theta}$ often involves substantial domain expertise in queueing theory. Defining an appropriate prior distribution $\mu$ can be highly non-trivial and subjective. Furthermore, posterior inference, e.g., sampling from $\mu(\theta|x_{0:t}, \pi)$, can be computationally intensive, especially in high-dimensional or non-conjugate settings.
In contrast, our approach is fully data-driven. Instead of explicitly modeling $\theta$, $\mu$, and $P_{\pi,\theta}$, we directly learn the conditional distributions $\P_n^{time}(t|\pi)$ and $\P_n^{event}(e|\pi)$ as functions of the observed event history up to step $n-1$. This enables flexible and automated learning of system dynamics without requiring domain-specific modeling assumptions. Algorithms \ref{alg:bayesian} and \ref{algo:autoreg} provide a side-by-side comparison of the conventional Bayesian approach (via Bayesian bootstrap) and our proposed method for uncertainty quantification. An equivalent formulation of the conventional Bayesian approach involves first sampling $\Theta \sim \mu(\cdot|x_{0:t}, \pi)$ and then, conditioned on this sample, generate the remainder of the trajectory $\{T_{n:N}, E_{n:N}\} \sim P_{\pi,\Theta}(\cdot|\mathcal{H}_n)$.  This alternative formulation is presented in Algorithm \ref{alg:alt_bayesian} in Appendix \ref{sec:alt_bayesian}, and is equivalent to Algorithm \ref{alg:bayesian}.

\begin{minipage}{0.525\textwidth}
\begin{algorithm}[H]
\caption{Bayesian Bootstrap} 
\label{alg:bayesian}
\begin{flushleft}
\textbf{Input:} Prior $\mu$, stochastic model $P_{\pi,\theta}$, history $\mathcal{H}_n$, number of trajectories $J$, length of trajectory $N$, performance measure $f(\cdot)$
\begin{enumerate}
    \item \textbf{for} $j \leftarrow 1$ \textbf{to} $J$ \textbf{do}
    \begin{enumerate}
        \item[2.] \textbf{Initialization:} $\mathcal{H}^j =\mathcal{H}_n$
        \item[3.] \textbf{for} $i \leftarrow n$ \textbf{to} $N$ \textbf{do}
        \begin{enumerate}
            \item[4.] Update $\mu(\theta|\mathcal{H}^j, \pi) = \frac{ P_{\theta, \pi}(\mathcal{H}^j) \mu(\theta)}{\int P_{\theta,\pi}(\mathcal{H}^j) \mu(\theta) d\theta}$
            \item[5.] Sample $\Theta\sim \mu(\cdot|\mathcal{H}^j, \pi)$
            \item[6.] Sample $(T_i^{(j)},E_i^{(j)}) \sim P_{\pi,\Theta}(\cdot|\mathcal{H}^j)$   
            \item[7.] Update $\mathcal{H}^j \leftarrow \mathcal{H}^j \cup \{T_i^{(j)}, E_i^{(j)}\}$
        \end{enumerate}
        \item[8.] \textbf{end for}
        \item[9.] Estimate $f(\mathcal{H}^j)$ 
    \end{enumerate}
    \item[10.] \textbf{end for}
    \item[11.] \textbf{Return:} $\{f({\mathcal{H}}^{1}),\dots, f({\mathcal{H}}^{J})\}$
\end{enumerate}
\end{flushleft}
\end{algorithm}
\end{minipage}\hfill
\begin{minipage}{0.44\textwidth}
\begin{algorithm}[H]
\caption{Autoregressive Generation} 
\label{algo:autoreg}
\begin{flushleft}
\textbf{Input:} Trained sequence model $\what{P}_\phi$, history $\mathcal{H}_n$, number of trajectories $J$, length of trajectory $N$, performance measure $f(\cdot)$
\begin{enumerate}
    \item \textbf{for} $j \leftarrow 1$ \textbf{to} $J$ \textbf{do}
    \begin{enumerate}
        \item[2.] \textbf{Initialization:} $\what{\mathcal{H}}^j =\mathcal{H}_n$
        \item[3.] \textbf{for} $i \leftarrow n$ \textbf{to} $N$ \textbf{do}
        \begin{enumerate}
            \item[4.] Sample ${T}_i^{(j)} \sim \what{P}_{\phi,i}^{time}$ 
            \item[5.] Update $\what{\mathcal{H}}^j \leftarrow \what{\mathcal{H}}^j \cup \{{T}_i^{(j)}\}$
            \item[6.] Sample ${E}_i^{(j)} \sim \what{P}_{\phi,i}^{event}$   
            \item[7.] Update $\what{\mathcal{H}}^j \leftarrow \what{\mathcal{H}}^j \cup \{{E}_i^{(j)}\}$
        \end{enumerate}
        \item[8.] \textbf{end for}
        \item[9.] Estimate $f(\what{\mathcal{H}}^j)$ 
    \end{enumerate}
    \item[10.] \textbf{end for}
    \item[11.] \textbf{Return:} $\{f(\what{\mathcal{H}}^{1}),\dots, f(\what{\mathcal{H}}^{J})\}$
\end{enumerate}
\end{flushleft}
\end{algorithm}\end{minipage}

\vspace{3mm}

\subsection{Uncertainty Quantification}
\label{sec:uncertainty_q_experiments}

As previously noted, a core advantage of our approach is the ability to quantify uncertainty, thereby supporting robust prediction and decision-making. In this section, we illustrate this capability using an M/M/1 queue with uncertain parameters $\theta$, representing the arrival and service rates.

In particular, we assume the model parameters $\theta$ are not fixed but instead drawn from a prior distribution $\mu$. This induces a marginal distribution over system trajectories of the form 
\[\P(X_{0:\infty}) = \int P_{\theta}(X_{0:\infty}) d\mu(\theta).\]
This formulation contrasts with the setup considered in Section~\ref{sec:methodology}, where $\theta$ was assumed to be fixed. Under this probabilistic framework, our goal is to assess whether the learned sequence model can accurately capture the resulting uncertainty in performance metrics induced by the prior uncertainty over $\theta$.

Crucially, under this setup, each trajectory of the service system corresponds to a different realization of the parameters, i.e., $\theta \sim \mu$. As a result, each trajectory exhibits distinct average inter-arrival times, service times, and waiting times, reflecting the variability induced by parameter uncertainty. Our analysis therefore focuses on the distribution of these trajectory-level averages across sampled trajectories.


We train the transformer model on a training dataset 
\[
\mathcal{D}^{train}\equiv \left\{\left(S_1^{(j)},T_1^{(j)}, \cdots, E_{N}^{(j)} \right): 1\leq j\leq K_t \right\},
\]
where the $j$th event table is generated from $P_{\theta^{(j)}}$ with $\theta^{(j)} \sim \mu$. For the priors, we assume the arrival rate $\lambda \sim U[1.5,2.5]$ and the service rate $\nu\sim U[3,6]$. Additional experiments with alternative prior distributions are presented in Appendix \ref{sec:experimental-details}.

To evaluate the trained Transformer model, we use Algorithm~\ref{algo:autoreg} to simulate future trajectories conditioned on a given history $\mathcal{H}_n$, and compute the corresponding trajectory-wise averages. The resulting empirical distributions of the averages are then compared to those produced by a conventional Bayesian stochastic modeling approach (as described in Algorithm~\ref{alg:bayesian}), which serves as our oracle baseline.
Figure~\ref{fig:M-M-1-uq} presents the distributions for average inter-arrival times, average service times, and average waiting times under both approaches. 
In each case, predictions are generated using an initial history of 
 \(n = 200\) observed events, followed by simulation over the next $200$ events, i.e., \(N - n = 200\).
 
As shown in Figure \ref{fig:M-M-1-uq}, 
the trained transformer closely approximates the oracle benchmark, demonstrating its capability to capture posterior distributions and uncertainty in queueing systems. Notably, the Transformer achieves this performance without explicit knowledge of the prior distribution or the underlying queueing dynamics, instead learning these relationships implicitly from data.

 \begin{figure}[h!]
  \centering
  \begin{minipage}[b]{0.325\textwidth}   \includegraphics[width=\textwidth, height=3.5cm]{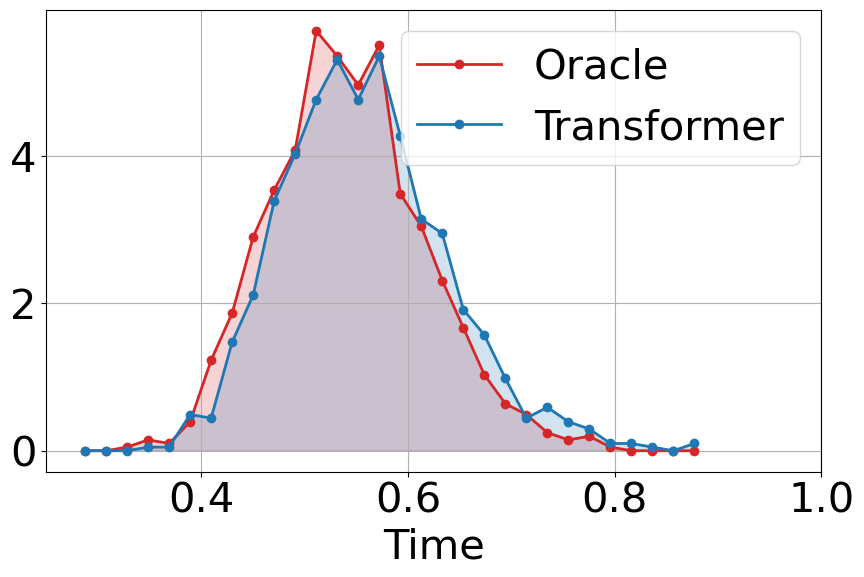}
  \centering{Average inter-arrival time}
  \end{minipage}
  \hfill
  \begin{minipage}[b]{0.325\textwidth}   \includegraphics[width=\textwidth, height=3.5cm]{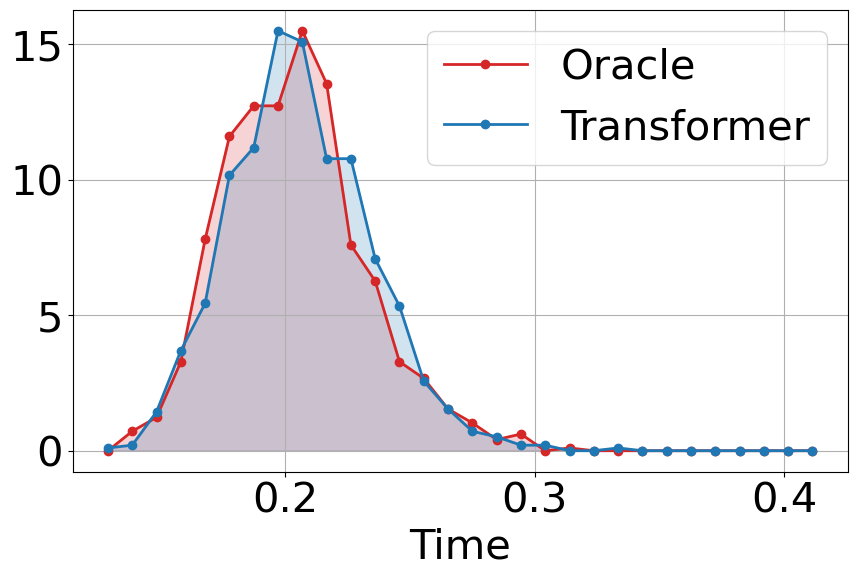}
  \centering{Average service time}
  \end{minipage}
  \hfill
  \begin{minipage}[b]{0.325\textwidth}   \includegraphics[width=\textwidth, height=3.5cm]{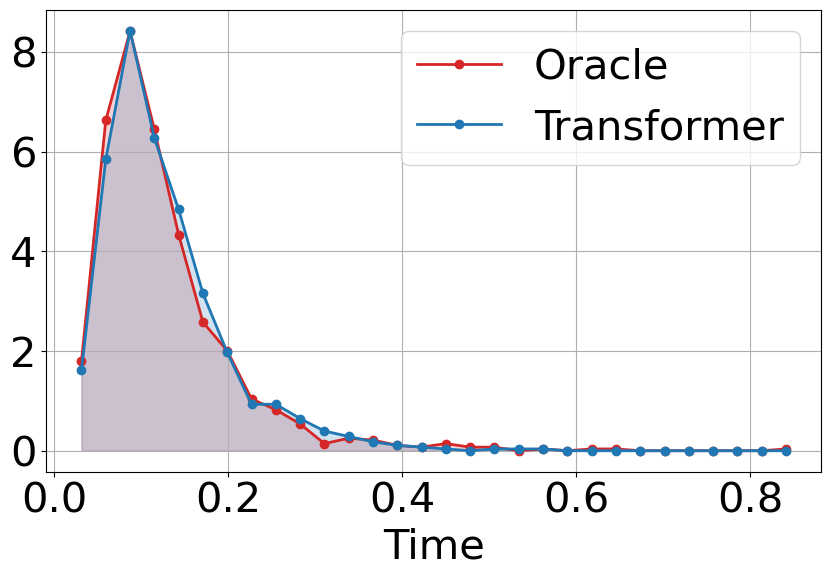}   
  \centering{Average positive waiting time}
    \end{minipage}
     \caption{\textbf{Uncertainty Quantification:} Comparing distribution of performance measures obtained from a trained transformer to a conventional Bayesian approach (oracle) in an M/M/1 queue with underlying parameters coming from a prior. Performance measures: average inter-arrival time, average service time, and average waiting time. History/Initial context ($n$) = 200 events, and Prediction length ($N-n$) = 200 events}
    \label{fig:M-M-1-uq}
\end{figure}

We also compute the KL divergence between the performance metric distributions produced by the Transformer (with a finite prediction horizon $N$) and the oracle benchmark extended to an infinite horizon ($N=\infty$). For comparison, we also calculate the KL divergence between the oracle benchmark at finite horizon $N$ and its infinite-horizon counterpart.  In all cases, predictions are conditioned on a fixed observed history $\mathcal{H}_n$ of length $n=200$.
Figure \ref{fig:M-M-1-kl-div} plots the KL divergence as a function of prediction length  \(N-n \in [25,200]\). We observe that KL divergence decreases monotonically with increasing prediction length. Importantly, the Transformer remains closely aligned with the oracle benchmark across all prediction lengths, highlighting its ability to capture posterior uncertainty, even though it is trained solely for next-token prediction.



 \begin{figure}[h!]
  \centering
  \begin{minipage}[b]{0.325\textwidth}   \includegraphics[width=\textwidth, height=3.5cm]{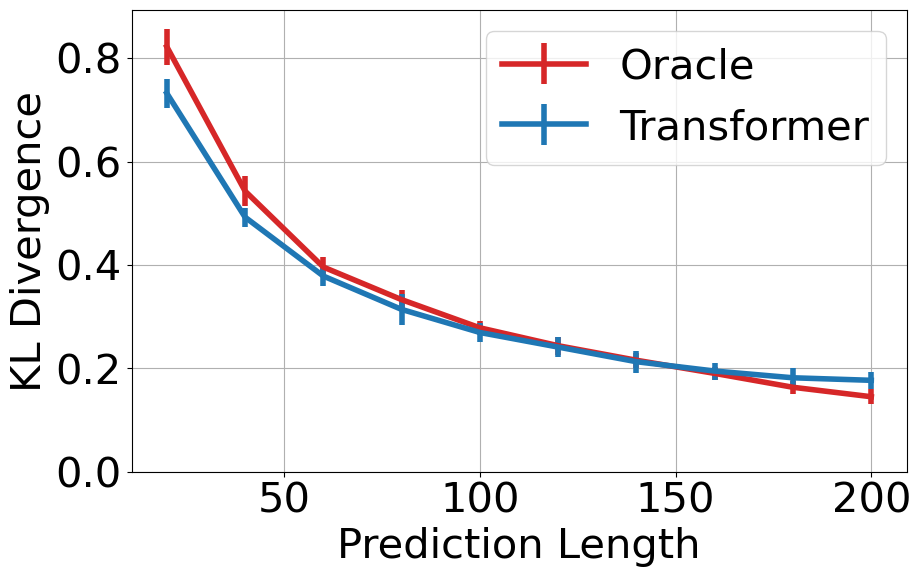}
  \centering{Average inter-arrival time}
  \end{minipage}
  \hfill
  \begin{minipage}[b]{0.325\textwidth}   \includegraphics[width=\textwidth, height=3.5cm]{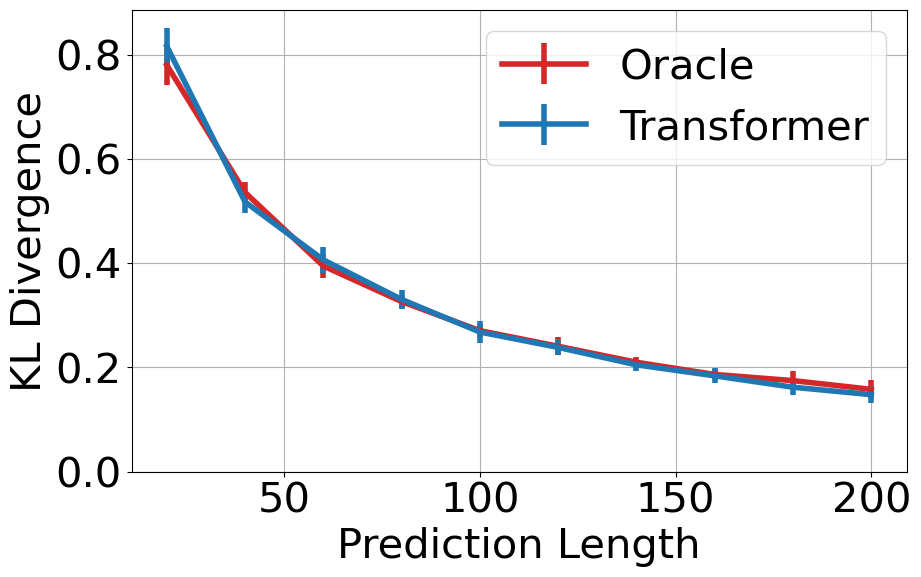}
  \centering{Average service time}
  \end{minipage}
  \hfill
  \begin{minipage}[b]{0.325\textwidth}   \includegraphics[width=\textwidth, height=3.5cm]{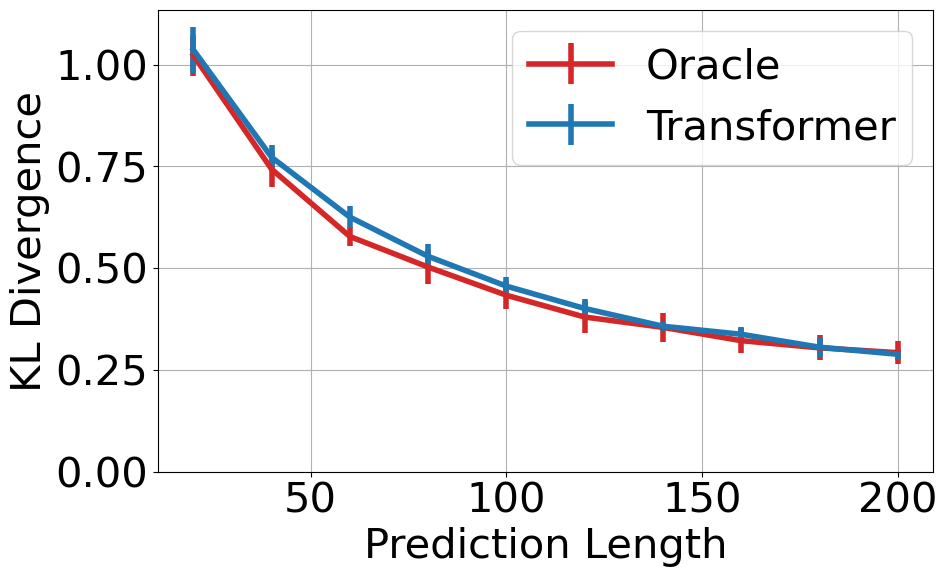}   
  \centering{Average waiting time}
    \end{minipage}
      \caption{\textbf{Uncertainty Quantification:}  Comparison of (1) KL divergence between the Transformer (finite horizon $N$) and oracle benchmark (infinite horizon $N=\infty$) and (2) KL divergence between the oracle benchmark at finite horizon ($N$) and  infinite horizon. (History/Initial context ($n$) = 200 events, and Prediction length ($N-n$) = 200 events).}
    \label{fig:M-M-1-kl-div}
\end{figure}

\vspace{3mm}
\noindent{\textbf{Connections to De Finetti's Partial Exchangeability}} The ability of sequence models to quantify uncertainty has deep theoretical foundations rooted in De Finetti’s seminal work on exchangeable sequences. De Finetti’s theorem establishes that any infinite exchangeable sequence of random variables can be represented as a mixture of i.i.d.\ sequences \citep{DeFinetti33, DeFinetti37, Aldous85}, thereby framing predictive uncertainty 
as arising from latent random parameters. This representation links uncertainty in future observations to uncertainty over an unobserved generative mechanism. Subsequent extensions of this theory have generalized the concept of exchangeability to partial exchangeability, which allows for structured dependencies such as those arising in Markov chains. In particular, partially exchangeable sequences can be represented as mixtures of Markov chains or more general Markov processes \citep{ DiaconisFr80, Freedman96, Aldous85, Banxico14}.

A formal review of De Finetti's notion of partial exchangeability (Definition \ref{def:definneti_assumption}) along with the corresponding result (Theorem \ref{thm:definneti_thm}) is provided in Appendix~\ref{sec:Definneti_partial_exchangeability}. Informally, the theorem asserts that a stochastic process $\P$ satisfies partial exchangeability if and only if it admits a representation as a mixture of recurrent Markov processes:
\[
\P=\int P_\theta d\mu(\theta).
\]
where $P_\theta$ parameterizes the transition dynamics and $\mu$ is a prior distribution over $\theta$.

This result offers a powerful interpretation: uncertainty over the latent parameter $\theta$, which governs the dynamics of the process, is mathematically equivalent to predictive uncertainty over unobserved future trajectories. That is, the uncertainty in future observations reflects incomplete knowledge of the underlying generative mechanism. Sequence models leverage this principle without explicitly modeling $\theta$ or specifying a prior $\mu$. Instead, they learn the conditional distribution of future events given the observed history and use it to generate plausible continuations of the process.

\section{Theoretical Insights}
\label{sec:theory}

In this section, we present a theoretical analysis of how a trained sequence model $\what{P}_\phi$ performs compared to the true data-generating process $\P$. By performance, we refer to the model's ability to accurately estimate specific system performance measures, e.g., average waiting times, when computed under $\widehat{P}_\phi$ as opposed to the ground-truth distribution $\mathbb{P}$. For clarity, we conduct our analysis using the sequence representation $X_{0:t}$ rather than the event table format. In what follows, trajectories drawn from $\P$ are denoted by $X_{0:t}$, while those drawn from $\what{P}_\phi$ are written as $\what{X}_{0:t}$.  

In the ideal case where the sequence model is perfectly trained, i.e., $\what{P}_\phi = \P$, it follows that for any real-valued measurable function $f(\cdot)$, the distributions of $f(\widehat{X}_{0:t})$ and $f(X_{0:t})$ coincide:
$f(\what{X}_{0:t}) \overset{D}{=} f(X_{0:t})$. 
However, in practice, the learned distribution $\widehat{P}_\phi$ will generally differ from $\mathbb{P}$. Consequently, it is essential to understand how accurately we can estimate the functional $f$ using samples $\what{X}_{0:t} \sim \what{P}_\phi(\cdot)$. 
Let $\widehat{P}_{0:t}$ and $\mathbb{P}_{0:t}$ denote the marginal distributions of $\widehat{X}_{0:t}$ and $X_{0:t}$, respectively. The following lemma, based on the Kantorovich–Rubinstein duality and Pinsker's inequality, provides a bound on the discrepancy between the evaluations of $f$ under the two distributions:
\begin{lemma}
    Assume $f$ is bounded, that is $\|f\|_\infty<\infty$, then
 \[W\left({f(\what{X}_{0:t}), f(X_{0:t}})\right) \leq {\sqrt{2} \,\|f\|_\infty} \,\sqrt{\dkl {{{\P}}_{0:t}}{\what{{P}}_{0:t}}},\]
where $W$ denotes the Wasserstein-1 distance.
 \label{lemma:preliminary_lemma}
\end{lemma} 

This result shows that the Wasserstein distance between the performance evaluations under the learned and true distributions is controlled by the square root of the KL divergence between the sequence distributions  $\mathbb{P}_{0:t}$ and $\widehat{P}_{0:t}$. Since the KL divergence is closely related to the loss function used to train the sequence model, see \eqref{eq:different_losses}, this bound provides theoretical justification that better model training, i.e., a smaller validation loss, leads to more accurate performance estimation. Importantly, Lemma~\ref{lemma:preliminary_lemma} holds without any structural assumptions on the true data-generating process $\mathbb{P}$ and applies to any bounded performance function $f$.

We next present a more detailed analysis by benchmarking our methodology against the conventional Bayesian approach as outlined in Section \ref{sec:generalized_formulation}. In the Bayesian framework, inference is performed over the latent parameter $\theta$, and uncertainty in a performance measure $f(\theta)$ is quantified via the posterior distribution over $\theta$. To facilitate a meaningful comparison, we impose additional structural assumptions on the data-generating process $\P(\cdot) $ and the learned autoregressive sequence model $\what{P}_\phi(\cdot)$.  

Specifically, we consider a family of stochastic processes $\{P_\theta:\theta \in \Theta\}$ that satisfies the regularity conditions in Assumptions~\ref{ass:PHR} -- \ref{ass:separability}.
Let $\mathcal S$ be a Polish space equipped with its Borel $\sigma$-algebra $\Sigma = \mathcal{B}(\mathcal S)$.
For each $\theta \in \Theta$, let $P_\theta = \{P_t^\theta : t \ge 0\}$ denote the transition semigroup of a time-homogeneous strong Markov process $X=(X(t))_{t \ge 0}$ on $(\mathcal S,\Sigma)$ with right-continuous paths having left limits (RCLL). 
For $A\in\Sigma$, define the hitting time
$
\kappa_A := \inf\{t\ge 0:\; X_t\in A\}.
$ For $x\in\mathcal S$, write $ P_{\theta,x}$ for the law of $X$ under parameter $\theta$ with $X_0=x$. 
Equivalently, the transition probabilities are
$
P_{t}^\theta(x,B) := P_{\theta,x}\!\left(X_t\in B\right), $ $ t\ge 0,\ B\in\Sigma.
$

\begin{assumption} \label{ass:PHR}
For each $\theta \in \Theta$, $P_\theta$ is an enhanced positive Harris recurrent strong Markov process with RCLL paths on $(\mathcal S,\Sigma)$ such that there exists a subset $C_\theta \in \Sigma$, a constant $\epsilon_\theta>0$, a probability $\varphi_\theta$ on $S$, and $\zeta_\theta>0$ such that
\begin{enumerate}
    \item $\kappa_{C_\theta}<\infty$ $P_{\theta,x}$ a.s. for each $x \in \mathcal S$.
    \item $P_x(X(\zeta_\theta)\in B) \geq \epsilon_\theta \varphi_\theta(B)$ for all $x\in C_\theta$, $B\in \Sigma$.
\end{enumerate}
\end{assumption}

Assumption~\ref{ass:PHR} is standard in the study of Harris recurrent Markov processes (see, e.g., \cite{asmussen2003, glynn11, Vlasiou2014}), and many queueing systems satisfy it \citep{SigmanWo1993, Sigman90, Vlasiou2014}. Whether a general Harris recurrent Markov process automatically satisfies properties~(1)–(2) in Assumption~\ref{ass:PHR} remains an open problem; see \cite{glynn11}. The following proposition is a direct consequence of Assumption~\ref{ass:PHR}.

\begin{proposition}
\label{prop:regeneration}
Under Assumption~\ref{ass:PHR}, there exist a sequence of a.s.\ finite stopping times $T_\theta(0)<T_\theta(1)<\cdots$, called \emph{regeneration times}, with the following properties: Let $
\tau_\theta(j):=T_\theta(j)-T_\theta(j-1)$, for $ j\ge 1$, denote the inter-regeneration lengths. Then,
\[
\{\tau_\theta(j), X(t):\, T_\theta(j-1)\le t<T_\theta(j)\},\quad j\ge 1,
\]
form one-dependent, identically distributed cycles. In addition,
$(\tau_\theta(j), j\ge 1)$ are
i.i.d.; and for each $j\ge 0$, the post-regeneration process $(X(T_\theta(j)+s):s\geq 0)$ is independent of $T_\theta(j)$. \end{proposition}


We further impose the following assumptions based on the construction in Proposition~\ref{prop:regeneration}.

\begin{assumption}
\label{ass:cycle-moments}
Fix the objects $T_\theta(j)$'s and $\tau_{\theta}(j)$'s from Proposition~\ref{prop:regeneration} and let $T_\theta(0)$ be the first regeneration time. Then,
\[
0<\inf_{\theta\in\Theta}\,\mathbb E_{P_\theta}[\tau_{\theta}(1)]
\ \le\ \sup_{\theta\in\Theta}\,\mathbb E_{P_\theta}[\tau_{\theta}(1)]<\infty,
\quad
\sup_{\theta\in\Theta}\,\mathbb E_{P_\theta}[\tau_{\theta}(1)^2]<\infty,
\]
and
\[
\sup_{\theta\in\Theta}\ \sup_{x\in\mathcal S}\ \mathbb E_{x,P_\theta}\!\big[T_{\theta}(0)\big]\;<\;\infty.
\]
\end{assumption}

\begin{assumption}\label{ass:separability}
Let $(C_\theta,\varepsilon_\theta,\varphi_\theta,\zeta_\theta)$ denote the small-set tuple from Assumption ~\ref{ass:PHR}. Assume $\inf_{\theta\in\Theta}\zeta_\theta>0$, and fix once and for all $\tilde{\zeta}\in\bigl(0,\inf_{\theta\in\Theta}\zeta_\theta\bigr)$. Define a probability measure $Q_\theta$ on $(\mathcal S\times\mathcal S,\Sigma\otimes\Sigma)$ by
\[
Q_\theta(A\times B)\;:=\;\int_A P^{\theta}_{\tilde{\zeta}}(x,B)\,\varphi_\theta(dx),\qquad A,B\in\Sigma.
\]
We assume separability of the family $\{Q_{\theta},\theta\in \Theta\}$ in total variation, i.e., for all $\theta\neq\theta'$,
\[
\|Q_\theta-Q_{\theta'}\|_{\mathrm{TV}}>0,
\]
where $\|\cdot\|_{\mathrm{TV}}$ denotes total variation distance. Equivalently, for each $\theta\neq\theta'$ there exist $A,B\in\Sigma$ such that
$
\int_A P^{\theta}_{\tilde{\zeta}}(x,B)\,\varphi_\theta(dx)\;\neq\;\int_A P^{\theta'}_{\tilde{\zeta}}(x,B)\,\varphi_{\theta'}(dx).
$
\end{assumption}

We further assume that both the true data-generating process $\mathbb{P}$ and the learned sequence model $\widehat{P}_\phi$ can be represented as mixtures over the processes $P_\theta$, differing only in their respective (implicit) mixing measures $\mu$ and $\widehat{\mu}$. This assumption is formalized below:

\begin{assumption} $\P$ and $\what{P}_\phi$ admit the following mixture representations:
        \[\P(X_{0:\infty}) = \int P_{\theta}(X_{0:\infty}) d\mu(\theta) \quad \text{and} \quad \what{P}_\phi(X_{0:\infty}) = \int P_{\theta}(X_{0:\infty}) d\what{\mu}(\theta). \]
\label{ass:mixture_assumption}
\end{assumption}

As discussed in Section~\ref{sec:uncertainty_q_experiments}, Assumption
\ref{ass:mixture_assumption} for $\P$ captures parametric uncertainty by positing that $\P$ is a mixture of Markov laws. Over a countable state space, this assumption admits a De Finetti–type characterization via \emph{Markov (partial) exchangeability}: a process is partially exchangeable (i.e., exhibits a symmetry of $X_{0:\infty}$) if and only if it is a mixture of recurrent Markov processes (on countable spaces, Harris recurrence reduces to the usual notion of recurrence). See Appendix~\ref{sec:Definneti_partial_exchangeability} for a formal review. Thus, in the countable–state case, Assumption~\ref{ass:mixture_assumption} is equivalent to imposing the symmetry of Markov (partial) exchangeability on $\P$.

For the learned model, Assumption~\ref{ass:mixture_assumption} (again, in the
countable–state case) likewise requires \emph{partial exchangeability}. Enforcing this structure in practice may involve architectural choices or inductive biases that encourage partial exchangeability. When the training data themselves are 
partially exchangeable, it is natural—and often empirically justified—to expect a
well–trained sequence model to inherit this structure to a meaningful degree. Even
if the assumption does not hold exactly, it provides a principled framework for
analyzing the effect of model misspecification on performance evaluation.

\vspace{3mm}

Next, we analyze a class of performance metrics that take the form of time-averaged functionals of the system trajectory. Specifically, we consider performance functions of the form 
\[f(X_{0:t}) = \frac{1}{t} \int_{s=0}^t h(X_s) ds,\] 
where $h: \mathcal{S} \to \mathbb{R}$ is a measurable function.
A typical example is when $h$ denotes the number of customers in the queue, in which case $f(X_{0:t})$ represents the average queue length over the time interval $[0,t]$. 
For each $\theta\in\Theta$, let $X_{0:\infty}(\theta)$ denote the trajectory of the Markov process under $P_{\theta}$.
Under suitable regularity conditions, the law of large numbers for Harris positive recurrent Markov processes ensures that the time-average performance converges almost surely (Theorem \ref{thm:L1_bound_harris_recurrent}), i.e.,
\[
f(\theta) := \lim_{t \to \infty} \frac{1}{t} \int_{s=0}^t h(X_s(\theta)) ds \mbox{ almost surely}.\]
Here, we slightly abuse notation by writing $f(\theta)$ as the limiting value of the time-average under dynamics governed by parameter $\theta$. Moreover, under the mixture distribution $\mathbb{P}$, the parameter $\theta$ is random, and hence $f(\theta)$ should be interpreted as a random variable induced by the prior $\mu$.
\begin{assumption} 
    Consider the family of stochastic processes $\{P_\theta : \theta \in \Theta\}$ and the measurable function $h: \mathcal{S} \to \mathbb{R}$.  
    \begin{enumerate} 
        \item       
         $ \sup_{\theta\in\Theta}\mathbb{E}\int_{s=T(0)}^{T(1)} |h(X_s(\theta))| ds < \infty$,  $\sup_{\theta\in\Theta}\mathbb{E}\left(\int_{s=T(0)}^{T(1)} |h(X_s(\theta)| ds \right)^2 < \infty$ and, uniformly over initial states, $\sup_{\theta\in\Theta} \sup_{x \in \mathcal S} \mathbb{E}\int_{s=0}^{T(0)} |h(X_s(\theta))| ds < \infty$. 
        \item 
        $\sup_{\theta,\theta' \in \Theta}|f(\theta)-f(\theta')| < \infty$.
    \end{enumerate}
    \label{ass:function_assumption}
\end{assumption}

Before presenting our main theoretical result, we state a lemma that links the divergences $\dkl{\P_{0:t}}{\what{P}_{0:t}}$  and  $\dkl {\mu}{\what{\mu}}$. 

\begin{lemma} Under Assumptions $\ref{ass:PHR}$ -- $\ref{ass:mixture_assumption}$,  
    \[
\dkl{\P_{0:t}}{\what{P}_{0:t}} \rightarrow \dkl {\mu}{\what{\mu}} \mbox{ as $t\rightarrow\infty$.}
\]
\label{lemma:convergence_kl_div_lemma}
\end{lemma}

In particular, for sufficiently large $t$ we have $\dkl{\mu}{\what{\mu}} \approx \dkl{\P_{0:t}}{\what{P}_{0:t}}$.
We are now ready to state our main theoretical result, which quantifies how well the learned sequence model approximates the true long-run performance:
\begin{theorem} Let $\widehat{P}_{\phi}$ be a sequence model trained on data generated from the true process $\mathbb{P}$.  
Suppose Assumptions~\ref{ass:PHR}  -- ~\ref{ass:function_assumption} hold. 
Then for all $t>0$,
 \[W({f(\what{X}_{0:t}), f(\theta)}) \leq \underbrace{c_1 \sqrt{\dkl {\mu}{\what{\mu}}}}_{\text{Error due to imperfect learning}}+ \underbrace{ \frac{c_2}{\sqrt{t}}}_{\text{Error due to limited prediction length}},\]
where $W$ is the Wasserstein-1 distance, $c_1$ and $c_2$ are finite constants that do not depend on $t$.
 \label{thm:final_thm}
\end{theorem}


Theorem \ref{thm:final_thm} provides a finite-sample performance guarantee on the discrepancy between the distribution of the performance metric $f(\what{X}_{0:t})$ estimated from trajectories generated by the learned sequence model $\what P_{\phi}$ and the true long-run performance measure $f(\theta)$ under the underlying generative model indexed by the latent parameter $\theta\sim\mu$.
The bound comprises two additive components:
The first term
quantifies the error arising from imperfect learning, reflecting how well the sequence model approximates the true data-generating process, and is directly determined by the quality of training.
The second term, decaying at the rate $\mathcal{O}(1/\sqrt{t})$, captures approximation error induced by the finite prediction horizon. Even with a perfect model, a finite $t$ introduces error in approximating the long-run average $f(\theta)$. 

This theoretical bound is consistent with the empirical trends observed in Figure \ref{fig:M-M-1-kl-div} for the M/M/1 queue. As the prediction length $N-n$ increases, the KL divergence between performance distributions, computed either from the transformer model or the finite-horizon oracle benchmark, relative to the infinite-horizon oracle benchmark, decreases.

All previous results extend straightforwardly to settings where the model is conditioned on an observed history $X_{0:t}$ and a given policy $\pi$.  

\section{Empirical validation}
\label{sec:Approach Validation}


To demonstrate the capabilities of our modeling approach, we conduct a series of experiments. We start by evaluating whether Transformer models can learn the dynamics of queueing systems. Unlike standard applications in text and vision, queueing models impose domain-specific structural constraints that are not straightforward to learn, e.g., departures cannot occur when the system is empty. In Section~\ref{sec:Markovian_systems}, we validate our approach on a multi-class M/M/n queue under different scheduling policies. We then extend the analysis to more general non-Markovian, non-stationary systems in Section~\ref{sec:complex_service_system}. Finally, we conduct a case study using real-world call center data, involving a tandem queueing network with customer abandonment in Section \ref{sec:call-center}. These experiments demonstrate the flexibility of our Transformer-based approach in modeling a broad range of queueing dynamics.

For each setting, we train the Transformer model on a dataset 
\[\mathcal{D}^{train}\equiv \left\{\left(S_1^{(j)},T_1^{(j)},E_1^{(j)} \cdots, T_{N}^{(j)}, E_{N}^{(j)} \right): 1\leq j\leq K_t \right\},\]   
where each trajectory in $\mathcal{D}^{\text{train}}$ is generated using a discrete-event simulator of the corresponding queueing system. Additional details on data generation, model architecture, and training procedures are provided in Appendix~\ref{sec:experimental-details}.

To evaluate the performance of the trained Transformer model, we use evaluation metrics tailored to each system configuration, designed to assess the model’s ability to accurately capture the underlying system dynamics. For Markovian systems, i.e., M/M/n queues, we report both the prediction loss relative to the optimal loss and distributional comparisons between empirical and true distributions for key performance measures, including inter-arrival times, service times, and waiting times. For non-Markovian systems, i.e., G/G/1 queues, we focus on distributional validation, comparing empirical and theoretical distributions for the same performance metrics. For non-stationary systems, i.e., $M_t/M/n$ queues, we employ time-dependent mean-based metrics to evaluate the model’s ability to track non-stationary temporal dynamics.

Given the generated event table data, additional system statistics, such as inter-arrival times, service times, and waiting times—can be readily extracted. See Appendix~\ref{sec:extracting_events_table} for further details.

\subsection{M/M/5 queue}
\label{sec:Markovian_systems}

In this section, we study M/M/5 queueing systems with five customer classes under various scheduling policies.

The presence of multiple customer types adds complexity to the system representation. Each event now includes both an event type (e.g., arrival or departure) and an associated customer type, resulting in composite events such as “arrival of a type 1 customer.” To model this structure, two approaches are possible: (1) assigning a unique index to each event–customer type combination, which leads to a total event space equal to the product of the number of event types and customer types; or (2) using a conditional modeling approach, where the event type is predicted first, followed by the customer type conditioned on the event.
We adopt the second approach to reduce the size of the prediction space (see Appendix~\ref{sec:experimental-details} for details). Accordingly, we report both event type prediction loss and customer type prediction loss to evaluate the model’s accuracy in capturing composite event dynamics. Our transformer is trained on 40,000 sequences.

Tables~\ref{table:M-M-5-FIFO} and~\ref{table:M-M-5-Strict} compare the average losses achieved by the trained transformed on a test set and the corresponding ‘optimal’ losses for the multi-class M/M/5 queue under First-In-First-Out (FIFO) and strict priority scheduling policies, respectively. In both cases, the Transformer model achieves prediction losses that are close to the optimal loss, indicating that it effectively learns the underlying system dynamics. For brevity, distributional comparisons of key performance metrics are provided in Appendix~\ref{sec:experimental-details}.

\begin{table}[h!]
\caption{Comparison of test-set average losses from the transformer and optimal losses from the oracle for 5-class M/M/5 queues with different arrival rates and service rates under FIFO}
\label{table:M-M-5-FIFO}
\centering
\resizebox{\textwidth}{!}{%
\begin{tabular}{|c|c|ccc|ccc|}
\hline
&    & \multicolumn{3}{c|}{Optimal loss} & \multicolumn{3}{c|}{Transformer loss}  \\ \cline{3-8}
 $\lambda$ & $\nu$    & event   &time   &customer    & event   &time   &customer          
 \\ \hline
 0.2,0.4,0.6,0.8,1.0  & 0.8,0.8,0.8,1.0,1.0        
&1.3075   &0.0299 &0.7481   
&1.3092    &0.0298   &0.7488
\\ \hline
 0.4,0.8,1.2,1.6,2.0  & 0.8,0.8,0.8,1.0,1.0        
&1.3650   &0.0091 &0.8519   
&1.3688    &0.0091   &0.8532
\\ \hline
\end{tabular}%
}
\end{table}

\begin{table}[h!]
\caption{Comparison of test-set average losses from the transformer and optimal losses from the oracle for 5-class M/M/5 queues with different arrival rates and service rates under strict priority ($1>2>\dots>5$)}
\label{table:M-M-5-Strict}
\centering
\resizebox{\textwidth}{!}{%
\begin{tabular}{|c|c|ccc|ccc|}
\hline
                  & &\multicolumn{3}{c|}{Optimal loss} &\multicolumn{3}{c|}{Transformer loss}  \\ \cline{3-8}
 $\lambda$ & $\nu$  
& event   &time   &customer    & event   &time   &customer     \\ \hline
 0.2,0.4,0.6,0.8,1.0  & 0.8,0.8,0.8,1.0,1.0     
&1.3012   &0.0302 &0.7518   
&1.3047   &0.0301   &0.7564
\\ \hline
0.4,0.8,1.2,1.6,2.0  & 0.8,0.8,0.8,1.0,1.0      
&1.3423   &0.0094 &0.8656
&1.3450    &0.0095   &0.8728
\\ \hline
\end{tabular}%
}
\end{table}

\subsection{Non-Markovian and non-stationary systems}
\label{sec:complex_service_system}


\paragraph{$G/G/1$ queue:}  We consider an $G/G/1$ queue with uniform inter-arrival and service time distributions. In this case, the distributions of the next inter-event time and event type depend on the entire history. 
Due to a lack of parametric characterization of the inter-event time distribution, we discretize it using Riemann distributions described in Section \ref{sec:operationalizing_predictive}. The transformer model is trained on 10,000 sequences.

Figure \ref{fig:G-G-1-distributions} compares the distributions of the inter-arrival times, service times, and waiting times generated by the trained transformer against the true distributions, respectively. The results show a near-exact match, indicating that the model effectively captures the underlying stochastic behavior.


 \begin{figure}[h!]
  \centering
  \begin{minipage}[b]{0.325\textwidth}   \includegraphics[width=\textwidth, height=3.5cm]{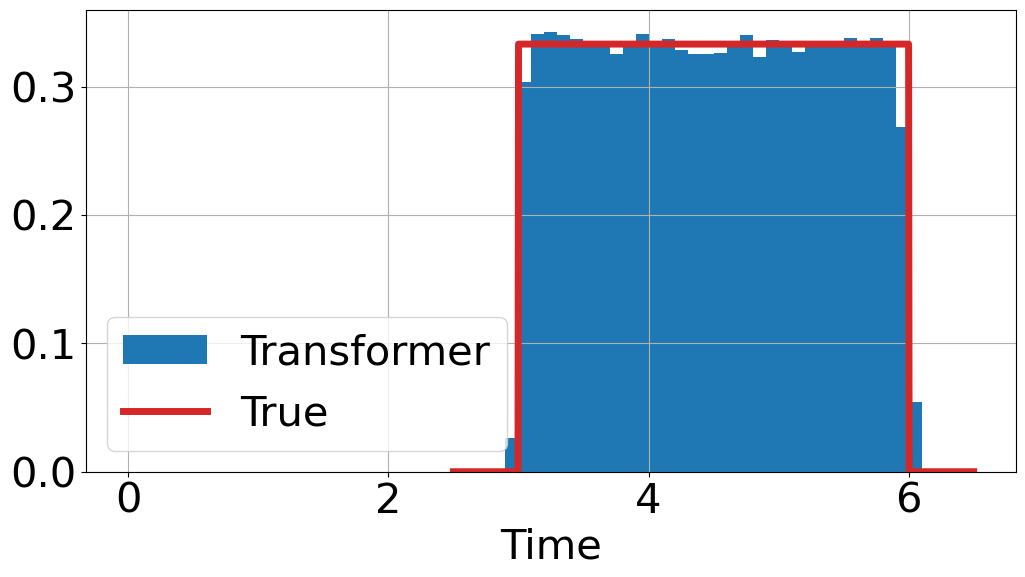}
  \centering{Inter-arrival time}
  \end{minipage}
  \hfill
  \begin{minipage}[b]{0.325\textwidth}   \includegraphics[width=\textwidth, height=3.5cm]{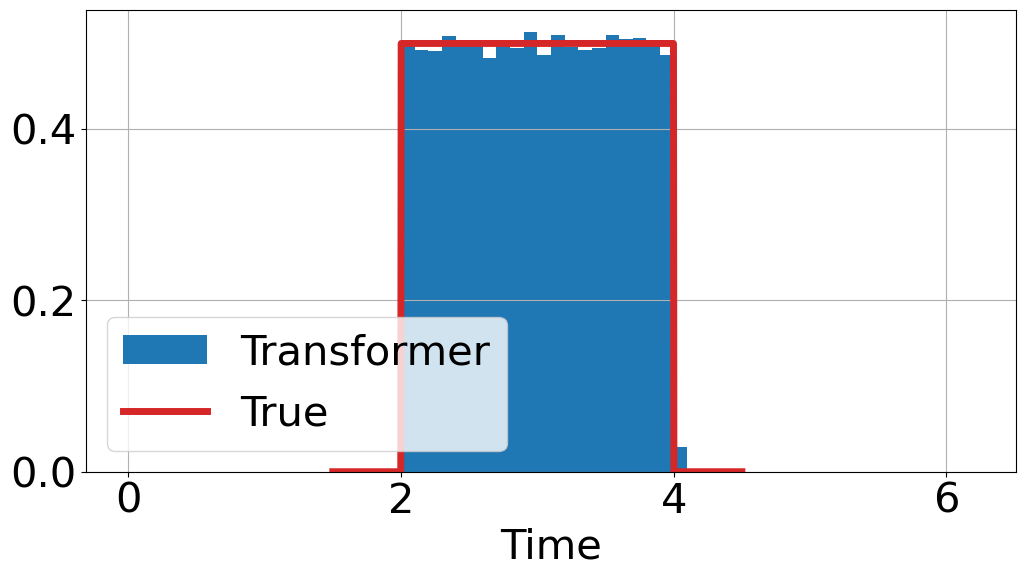}
  \centering{Service time}
  \end{minipage}
  \hfill
  \begin{minipage}[b]{0.325\textwidth}   \includegraphics[width=\textwidth, height=3.5cm]{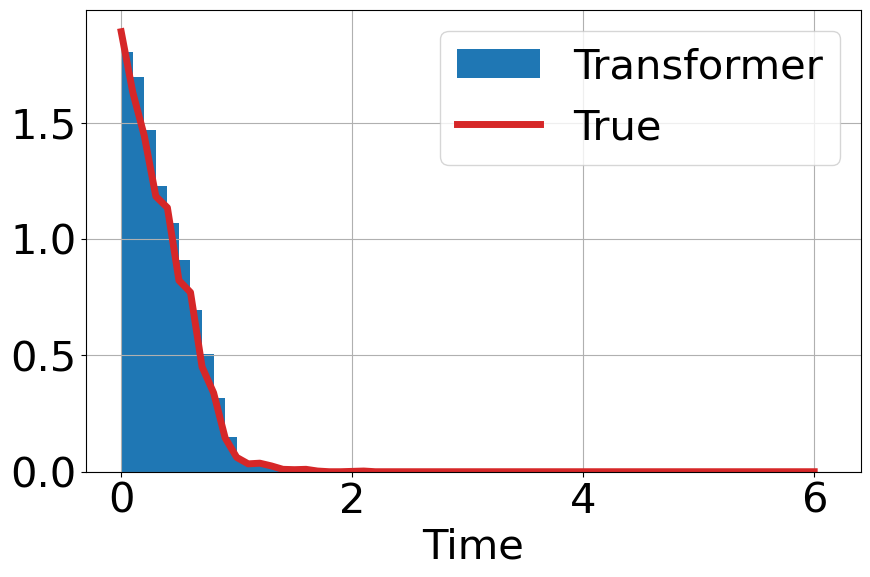}   
  \centering{Positive waiting time}
    \end{minipage}
    \caption{\textbf{Non-Markovian service system:} Comparison of performance measure distributions predicted by sequence model (transformer) and true distributions (derived using the knowledge of underlying queuing network) for a G/G/1 queue. Performance measures - Inter arrival time, Service time, Positive waiting time.}
    \label{fig:G-G-1-distributions}
\end{figure}

\paragraph{$M_t/M/n$ queue} 
We consider an $M_t/M/n$ queueing system in which arrivals follow a time-varying Poisson process and service times are exponentially distributed. The system is staffed with 11 servers, each operating at a service rate of 1.6 customers per hour. The arrival rate $\lambda(t)$ varies over a 17-hour period according to the hourly profile: $[ 8,  8,  8,  8,  8, 14, 15, 16, 17, 18, 19, 18, 17, 16, 15, 11, 11]$ customers per hour. Our transformer model is trained on 40,000 sequences.

Figure \ref{fig:MT-M-N-distributions} compares the average inter-arrival time, service time, and waiting time across different hours over the 17-hour period, as generated by the trained transformer versus the true discrete-event simulator. The averages are estimated using 200 trajectories of length 400. We again observe a strong alignment between the two.

 \begin{figure}[h!]
  \centering
  \begin{minipage}[b]{0.325\textwidth}   \includegraphics[width=\textwidth, height=3.5cm]{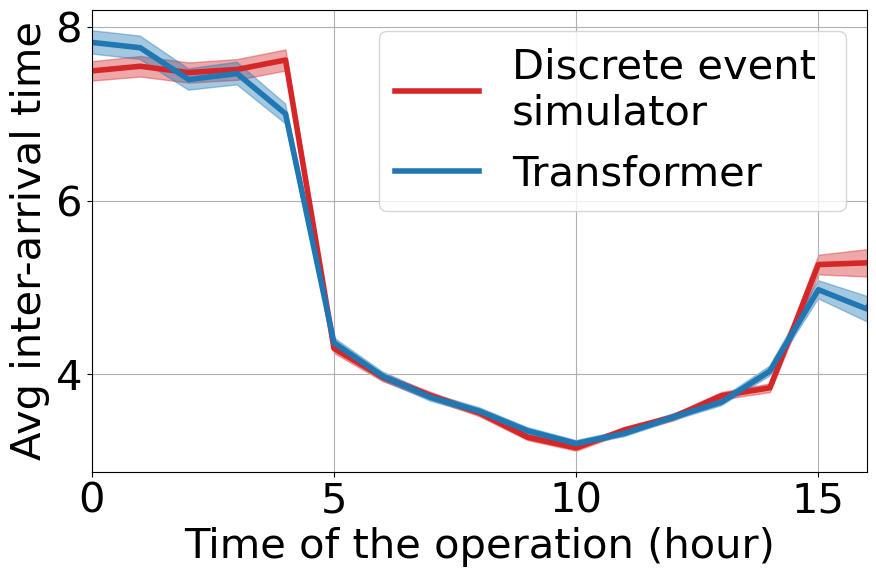}
  \centering{Average inter-arrival time}
  \end{minipage}
  \hfill
  \begin{minipage}[b]{0.325\textwidth}   \includegraphics[width=\textwidth, height=3.5cm]{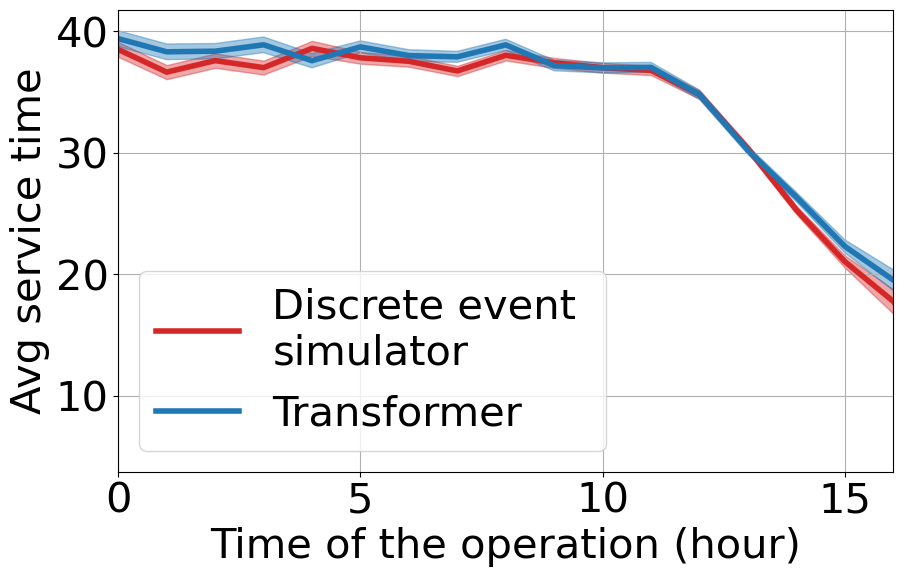}
  \centering{Average service time}
  \end{minipage}
  \hfill
  \begin{minipage}[b]{0.325\textwidth}   \includegraphics[width=\textwidth, height=3.5cm]{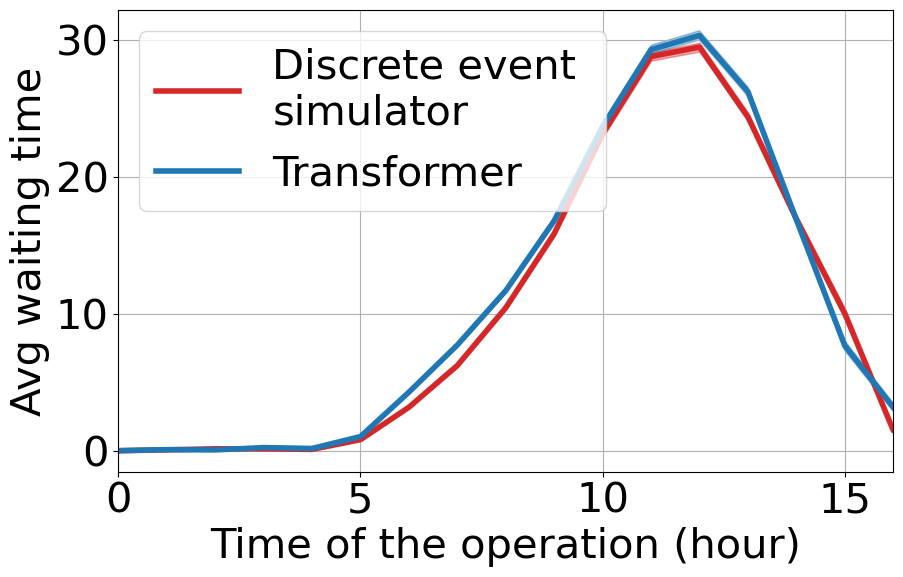}
  \centering{Average waiting time}
    \end{minipage}
    \caption{\textbf{Service system with non-stationary arrival rates:} Comparing trained transformer and a discrete event simulator with knowledge of the queueing network.  Graphs plot average performance measures v/s hour of the service system operation  (for e.g. average inter-arrival time, average service time, average waiting time)  }
    \label{fig:MT-M-N-distributions}
\end{figure}


\subsection{Call Center}
\label{sec:call-center}
We build a call center simulator using data from the call center of an Israeli Bank \citep{seelab-report}.
The call center follows a multi-stage service process with the following characteristics (see Figure \ref{fig:call-center-service-system} for a pictorial illustration):

\begin{itemize}
\item \textbf{Customer Types:} There are six distinct customer types, each characterized by different arrival rates, service time distributions, patience time distributions, and priority levels. 

\item  \textbf{Customer Arrival:} Customers arrive at the system according to Poisson processes. 

\item \textbf{VRU Servers:} Upon arrival, each customer is initially routed to a pool of Voice Response Unit (VRU) servers. This stage typically involves automated interactions such as navigating menus or providing account verification. 
The service times at VRU servers are sampled from the empirical distribution, without any parametric assumption. 

\item \textbf{Queueing and Abandonment:} After completing the VRU stage, customers proceed to the second stage, where they are served by human agents. If no agent is immediately available, customers wait in a queue.
While waiting, customers may abandon the queue due to excessive wait times.
Abandonment behavior is modeled via exponential patience-time distributions, where the rates are calibrated from the data.



\item \textbf{Human Servers:} Customers who do not abandon the queue are eventually routed to human agents for service. This stage involves more complex and personalized interactions. 
Service times at this stage are sampled from the empirical distribution, without any parametric assumption. 

\item \textbf{Customer Departure:} After being served by a human agent, customers exit the system.  

\end{itemize}

 \begin{figure}[h!]
 \centering
  \includegraphics[width=0.85\textwidth, height=2.75cm]{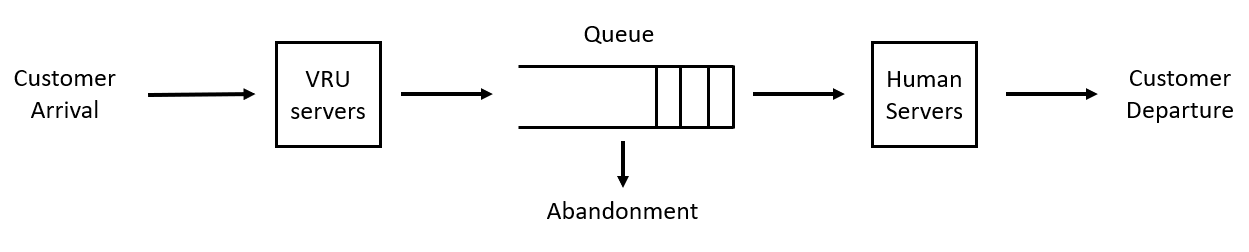}   
    \caption{Call-center service system.}
    \label{fig:call-center-service-system}
\end{figure}

 A simulator is necessary in this setting due to the limited size of the call center dataset, which is insufficient for training purposes. Note that building this simulator requires substantial queueing expertise. See Appendix \ref{sec:experimental-details} for details about the model building and calibration.


Our simulation framework incorporates nine different event types:
1) customer arrival, (2) VRU service completion, (3) queue abandonment, and (4-9) customer departure from human servers. Event (4-9) comprises six events since our model incorporates six distinct human servers. As before, we use the conditional modeling approach, where we first predict the event type and then, conditioned on that event, predict the customer type. Our transformer is trained on 40,000 sequences.

\begin{figure}
\centering
   \begin{minipage}[b]{0.18\textwidth}   
  \end{minipage}
  \hfill
  \centering
  \begin{minipage}[b]{0.33\textwidth}   \includegraphics[width=\textwidth, height=3.5cm]{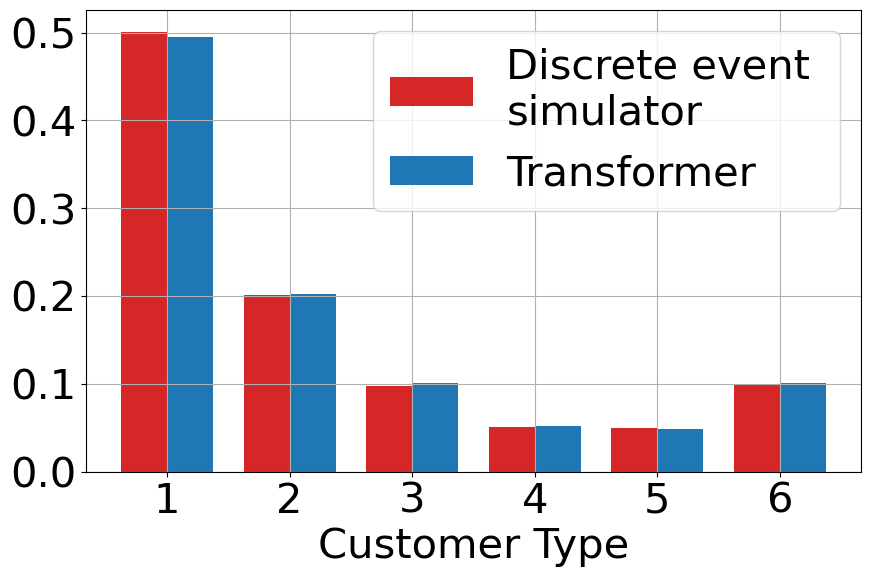}
  \centering
  {Customer type}
  \end{minipage}
  \hfill
  \begin{minipage}[b]{0.33\textwidth}   \includegraphics[width=\textwidth, height=3.5cm]{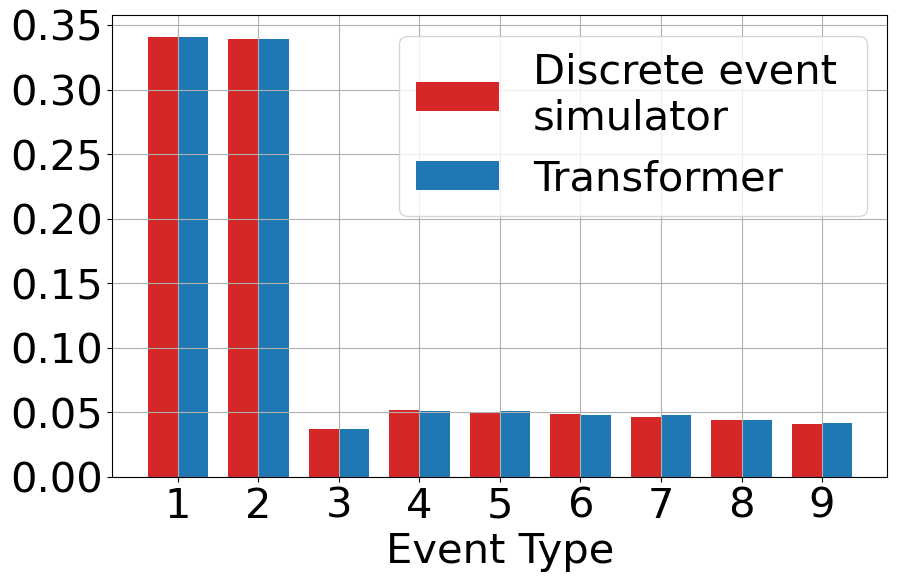}
  \centering
  {Event type}
  \end{minipage}
  \hfill
  \begin{minipage}[b]{0.18\textwidth}   
  \end{minipage}
  \caption{\textbf{Call-center (summary statistics):} Comparing ratios of different customers and events across trajectories from the sequence model (transformer) and discrete event simulator.
  }
    \label{fig:call-center-summary-statistics}
\end{figure}

 \begin{figure}
  \centering
    \begin{minipage}[b]{0.18\textwidth}
    \end{minipage}
    \hfill
  \begin{minipage}[b]{0.325\textwidth}   \includegraphics[width=\textwidth, height=3.5cm]{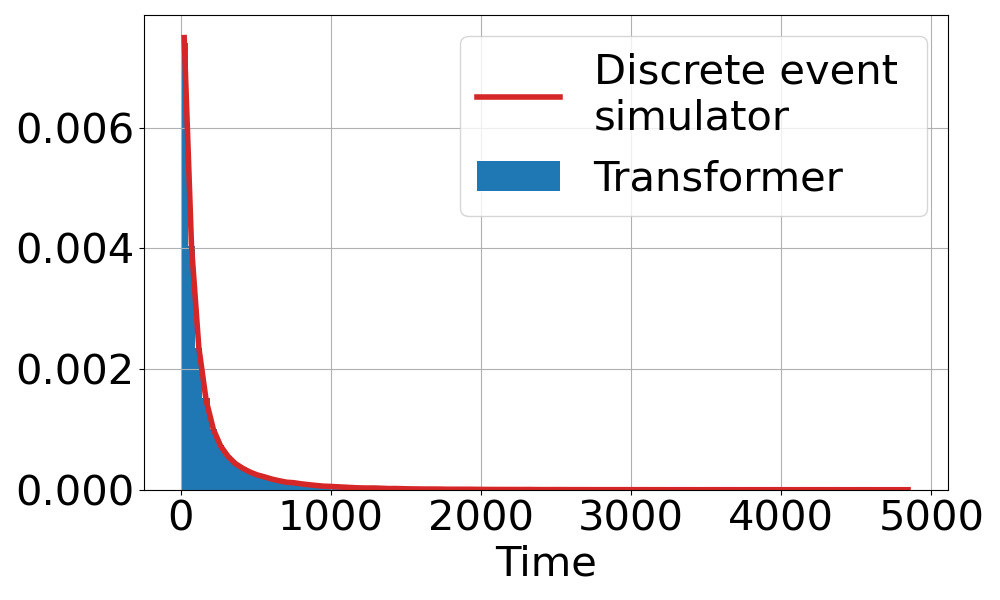}
  \centering{Inter-arrival time}
  \end{minipage}
  \hfill
  \begin{minipage}[b]{0.325\textwidth}   \includegraphics[width=\textwidth, height=3.5cm]{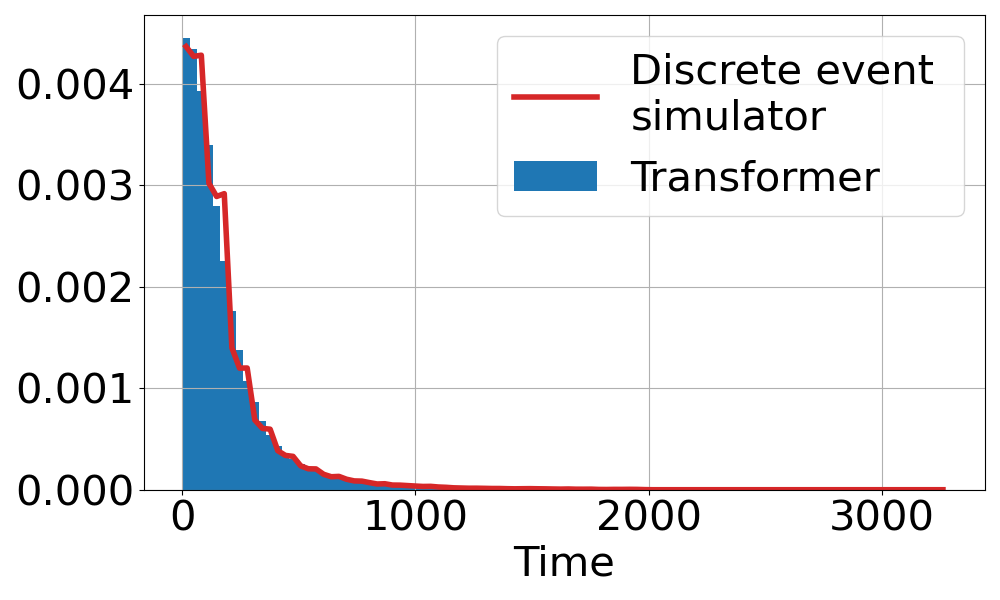}
  \centering{Service time}
  \end{minipage}
  \hfill
  \begin{minipage}[b]{0.18\textwidth}
    \end{minipage}
    \caption{\textbf{Call-center:} Comparison of performance measure distributions predicted by sequence model (transformer) and discrete event simulator (derived using the knowledge of underlying queuing network) for a call-center. Performance measures show inter-arrival times and service times pooled across all customers.}
    \label{fig:See-lab-average-distributions}
\end{figure}

Figure \ref{fig:call-center-summary-statistics} compares the distributions of customer types and event types generated by the trained transformer and the call center simulator. 
We observe a strong match between the two distributions. In addition, Figure \ref{fig:See-lab-average-distributions} compares the distributions of inter-arrival times and service times over all customers generated by the transformer and the call center simulator. Figure \ref{fig:representative-waiting-time-distributions-see-lab} compares the class-dependent average waiting time distributions for three representative customer classes (1, 4, and 6) generated by the transformer versus the call center simulator (Additional results are presented in Appendix \ref{sec:experimental-details}).
We again observe a strong match between the two.
These results demonstrate the trained Transformer model’s ability to accurately capture and reproduce the dynamics of the underlying system, highlighting its effectiveness as a high-fidelity queueing simulator.



 \begin{figure}
  \centering
  \begin{minipage}[b]{0.325\textwidth}   \includegraphics[width=\textwidth, height=3.5cm]{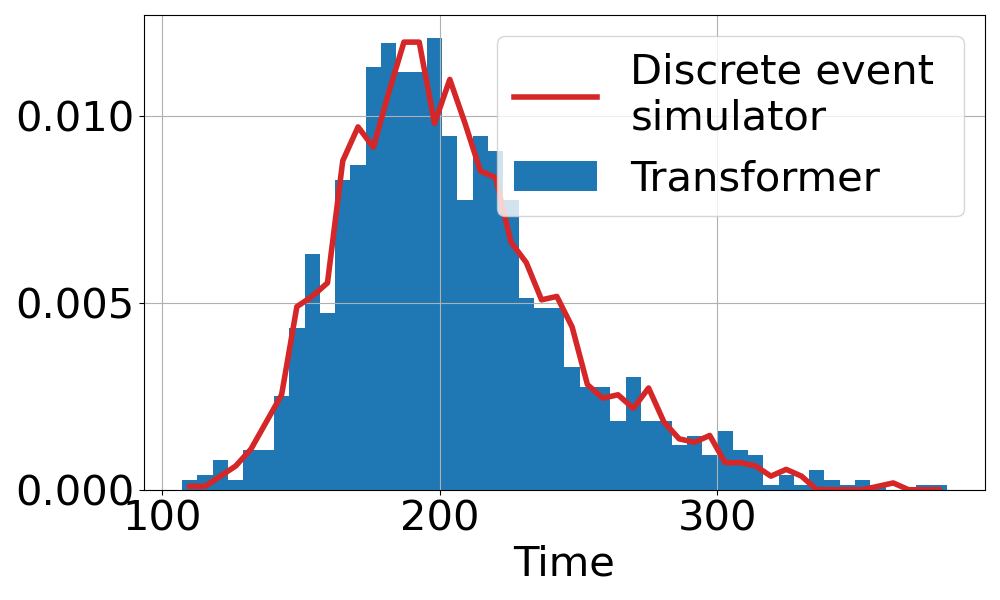}
  \centering
  {Customer 1}
  \end{minipage}
  \hfill
  \begin{minipage}[b]{0.325\textwidth}   \includegraphics[width=\textwidth, height=3.5cm]{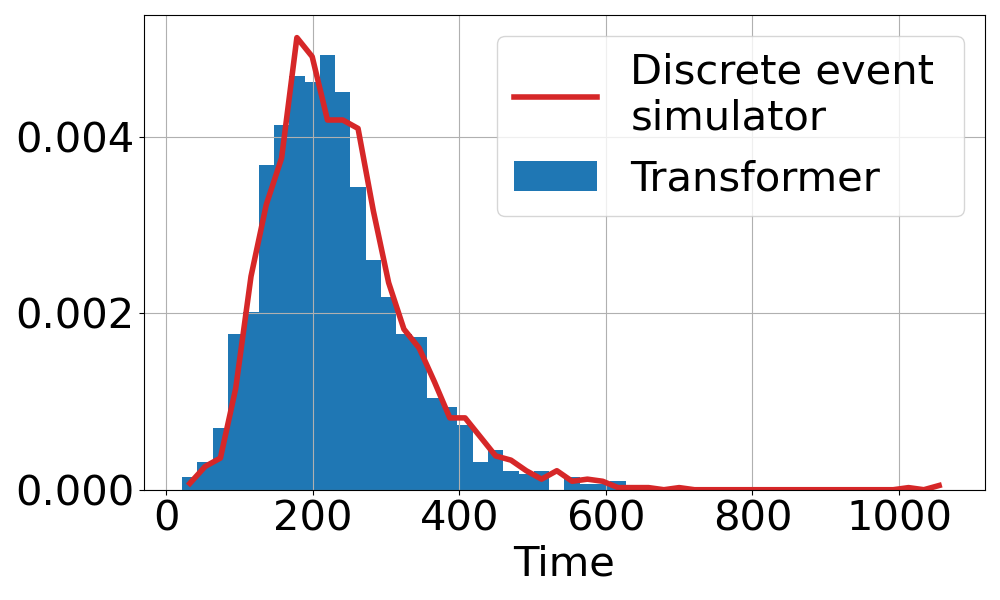}
  \centering
  {Customer 4}
  \end{minipage}
  \hfill
  \begin{minipage}[b]{0.325\textwidth}   \includegraphics[width=\textwidth, height=3.5cm]{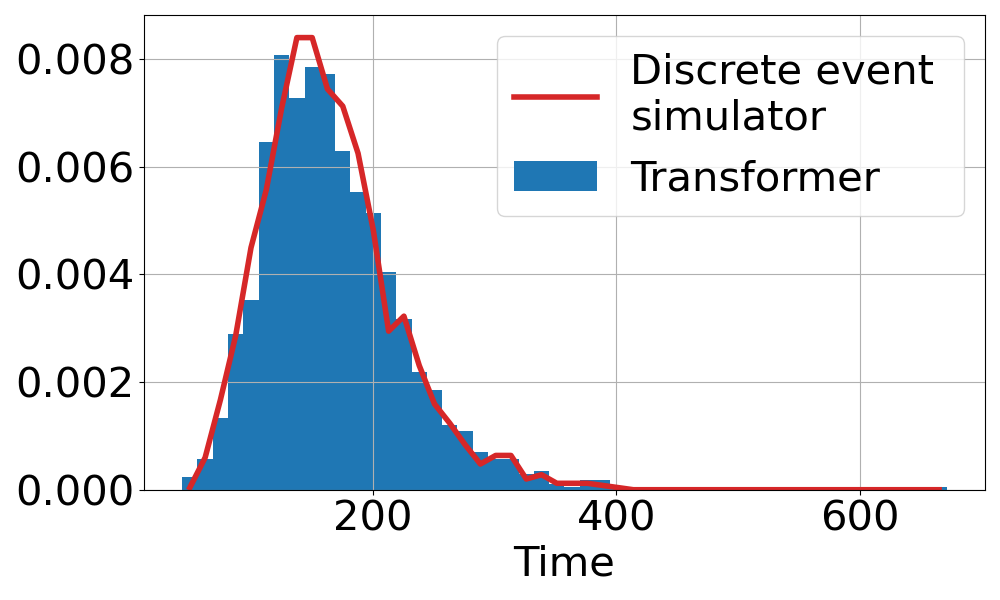}
  \centering
  {Customer 6}
    \end{minipage}
    \hfill
    \caption{\textbf{Call-center (Average waiting time distributions for different customers):} Comparison of average waiting time distributions predicted by sequence model (transformer) and discrete event simulator (derived using the knowledge of underlying queuing network) for a call-center.}
    \label{fig:representative-waiting-time-distributions-see-lab}
\end{figure}




\section{Simulating counterfactuals}
\label{sec:simulating_counterfactuals}

The ability to simulate counterfactual scenarios is a central strength of stochastic simulation and plays a critical role in operational decision-making. By enabling exploration of `what-if' scenarios, these simulations support data-driven policy optimization. In this section, we assess whether sequence models can effectively generate meaningful counterfactual simulations.

We examine a service system analogous to a hospital emergency department, where the service operator must determine the optimal number of servers $N$ (e.g., nurses) to deploy for a 12-hour shift. The arrival rate exhibits daily variation and is modeled as $c+\lambda(t)$, where $c$ represents day-to-day variability and $\lambda(t)$ captures time-varying intraday arrival patterns.  The system is modeled as an $M_t/M/N$ queue with $N\in \{2,3, \ldots, 20\}$, and the service rate is 3.5 patients per hour.
The intraday baseline $\lambda(t)$ follows the hourly pattern   $[8,8,8,8,8,14,15,16,17,18,19,18]$ patients per hour. Staffing levels, once chosen, remain fixed for a 12-hour shift.

We assume that the operator has access to the daily value of $c$ a priori through a prediction model. Given this estimate of $c$, the trained Transformer (sequence model) must simulate the system across multiple values of $N$ to determine the optimal $N$ to deploy for that shift. Consequently, this application requires the Transformer model to generate system trajectories under diverse operational conditions characterized by varying values of $c$ and $N$.

 \begin{figure}[t]
  \centering
  \begin{minipage}[b]{0.325\textwidth}   \includegraphics[width=\textwidth, height=3.5cm]{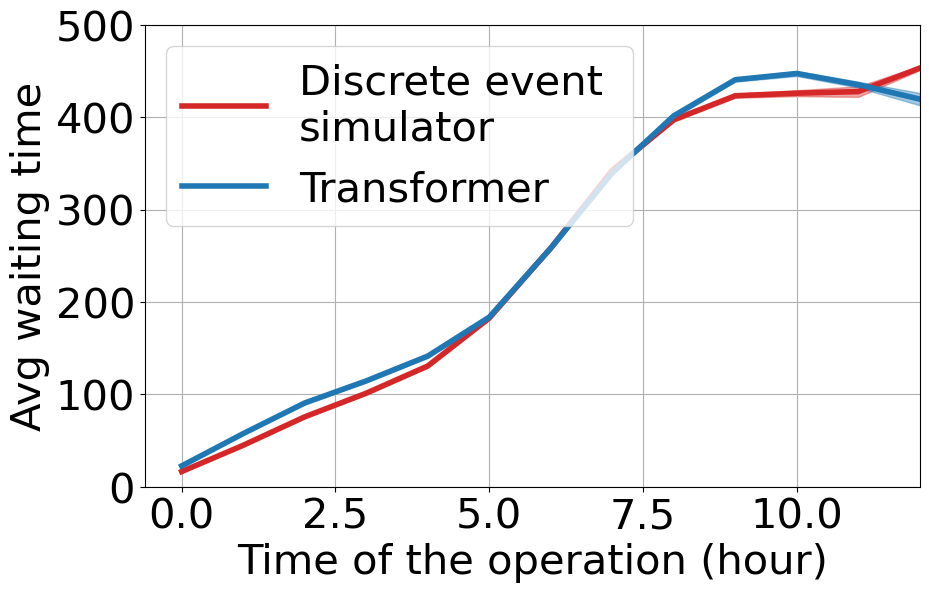}
  \centering{{{$N=2$, $c=2$ \\  Average waiting time}}}
  \end{minipage}
  \hfill
  \begin{minipage}[b]{0.325\textwidth}   \includegraphics[width=\textwidth, height=3.5cm]{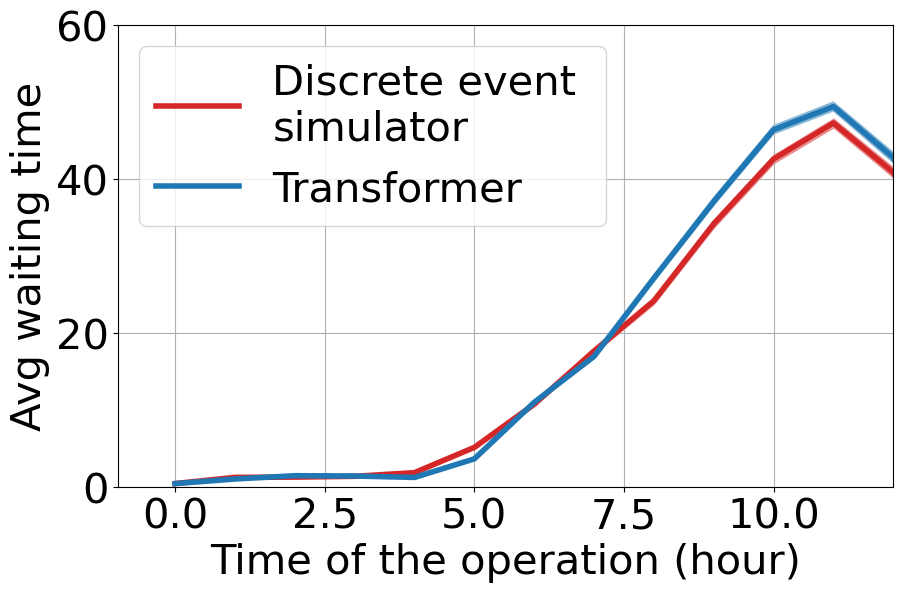}
  \centering{{{$N=5$, $c=2$ \\ Average waiting time }}}
  \end{minipage}
  \hfill
  \begin{minipage}[b]{0.325\textwidth}   \includegraphics[width=\textwidth, height=3.5cm]{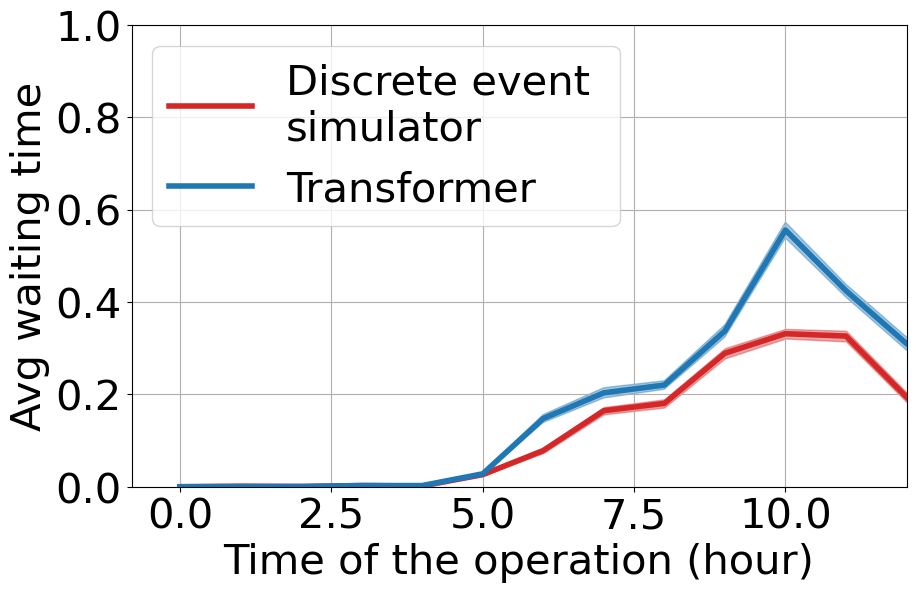}
  \centering{{{$N=10$, $c=2$ \\ Average waiting time}}}
  \end{minipage}
    \caption{\textbf{Counterfactual simulations:} Comparing trained transformer and a discrete event simulator with knowledge of the queueing network.  Graphs plot average waiting time v/s hour of the service system operation.  }
    \label{fig:n_c_waiting-time-distributions}
\end{figure}

 \begin{figure}[t]
  \centering
  \begin{minipage}[b]{0.18\textwidth}  
  \end{minipage}
  \hfill
  \begin{minipage}[b]{0.325\textwidth}   \includegraphics[width=\textwidth, height=3.5cm]{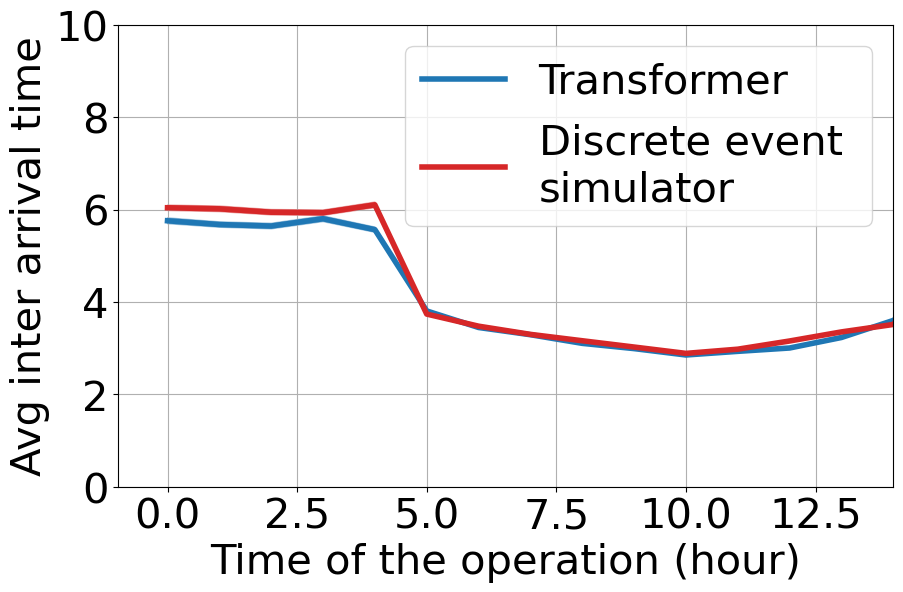}
   \centering{{{$N=2$, $c=2$ \\ Average inter-arrival time}}}
  \end{minipage}
  \hfill
  \begin{minipage}[b]{0.325\textwidth}   \includegraphics[width=\textwidth, height=3.5cm]{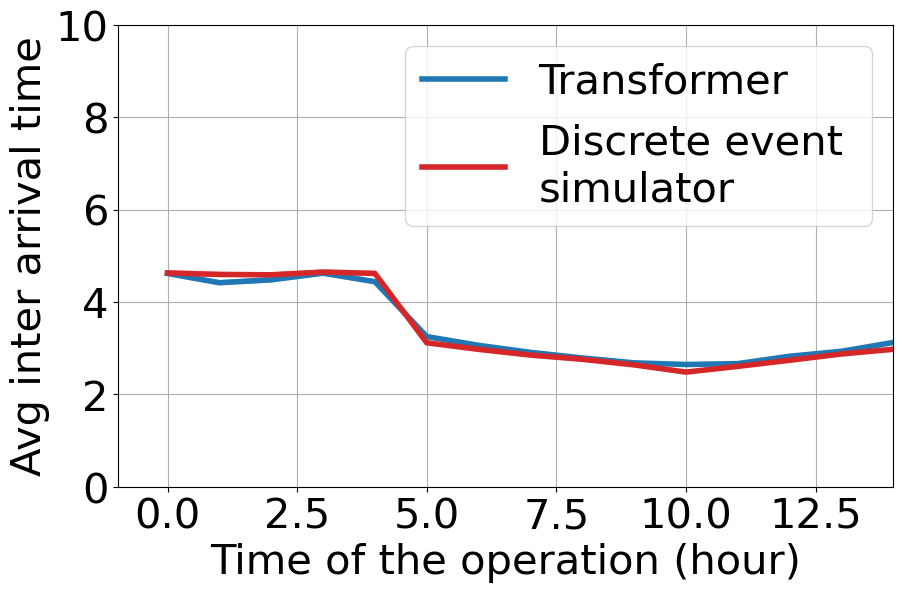}
   \centering{{{$N=2$, $c=5$ \\ Average inter-arrival time}}}
   \end{minipage}
   \hfill
   \begin{minipage}[b]{0.18\textwidth}   
  \end{minipage}
    \caption{Comparing trained transformer and a discrete event simulator with knowledge of the queueing network.  Graphs plot average inter-arrival time v/s hour of the service system operation. }
    \label{fig:n_c_arrival-time-distributions}
\end{figure}

 \begin{figure}[t]
  \centering
  \begin{minipage}[b]{0.18\textwidth}  
  \end{minipage}
  \hfill
  \begin{minipage}[b]{0.325\textwidth}   \includegraphics[width=\textwidth, height=3.5cm]{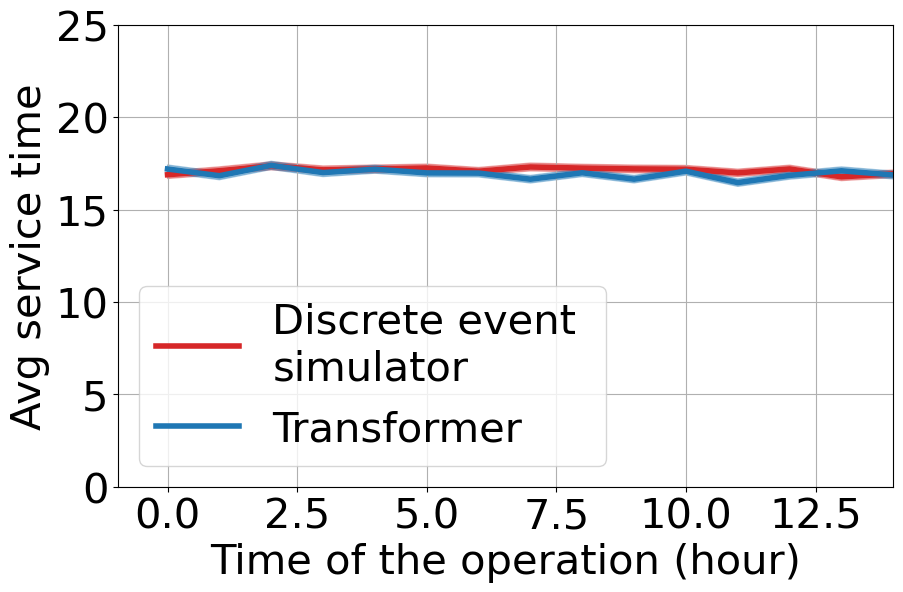}
   \centering{\small{{$N=2$, $c=2$ \\ Average service time}}}
  \end{minipage}
  \hfill
  \begin{minipage}[b]{0.325\textwidth}   \includegraphics[width=\textwidth, height=3.5cm]{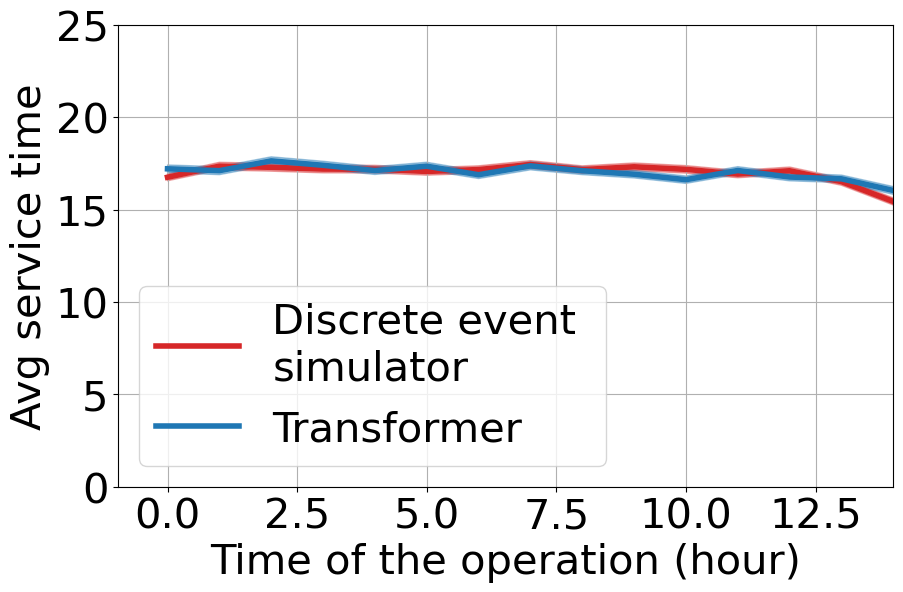}
   \centering{\small{{$N=2$, $c=5$ \\ Average service time }}}
   \end{minipage}
   \hfill
   \begin{minipage}[b]{0.18\textwidth}   
  \end{minipage}
    \caption{Comparing trained transformer and a discrete event simulator with knowledge of the queueing network.  Graphs plot average service time v/s hour of the service system operation. }
    \label{fig:n_c_service-time-distributions}
\end{figure}

The Transformer is trained on the dataset:
\[
\mathcal{D}^{train}\equiv \left\{\left(c_j, N_j,\left(S_1^{(j)},T_1^{(j)}, \cdots, E_{N}^{(j)} \right)\right): 1\leq j\leq K_t \right\},
\]
where each event sequence $\left(S_1^{(j)}, T_1^{(j)}, \dots, E_N^{(j)}\right)$  is generated via discrete event simulation using the corresponding parameter values  $c_j$ and $N_j$. The parameters $c_j$ and $N_j$ are included as inputs to the Transformer. Details on how they are incorporated into the model architecture are provided in Appendix~\ref{sec:experimental-details}. We train the Transformer on 40,000 sequences. 

In our training data, we draw $c$ uniformly from $ [-5,-1]\cup[5,8]$, while $N$ is drawn unifromly from $\{2,3,\ldots,20\}$. 
For testing, we consider the range $c\in[-1,4]$, which lies outside the training range for $c$, thereby testing the model's extrapolation capabilities.

We benchmark the transformer's output against a discrete event simulator. Figure \ref{fig:n_c_waiting-time-distributions} illustrates the time-varying average waiting times for $c=2$ and $N\in\{2,5,10\}$. The transformer's predictions demonstrate strong concordance with simulator outputs. 
Similar predictions can be generated for alternative values of $N$, enabling operators to select the most appropriate staffing level for the shift.
Figures \ref{fig:n_c_arrival-time-distributions} and \ref{fig:n_c_service-time-distributions} present analogous results for time-varying average inter-arrival and (time homogeneous) average service times, respectively, further confirming robust agreement between the transformer and discrete event simulator outputs.

Overall, our experimental findings suggest that Transformer-based sequence models can perform counterfactual simulations, supporting their use for operational decision-making.

\section{Computational and Data requirements}
\label{sec:computational_requirements}


This section provides an initial analysis of the computational and data requirements, aiming to provide practical insights for researchers and practitioners interested in applying Transformer models to similar problems. In addition, we present preliminary experiments that explore a promising research direction: the development of a foundation model for queueing networks. Such a model may be especially useful for organizations with limited datasets, as it could potentially adapt quickly through in-context learning.

\subsection{Computational Requirements} 
A common concern when training Transformers and other sequence models is their substantial computational resource requirements. However, our experiments suggest that, in this specific application domain, Transformer models can be trained with relatively modest resources, making them more practical for real-world use than often assumed.
Although both training and inference generally require access to GPUs, the need for dedicated infrastructure can be mitigated through cloud-based solutions. This not only reduces capital and maintenance costs but also improves the overall accessibility and cost-effectiveness of the proposed approach.

All training experiments were conducted on a single NVIDIA A100-SXM4 GPU with 80GB of memory. The training datasets ranged from 10,000 to 40,000 trajectories (event sequences), each consisting of 400 events. Each model was trained for 500 epochs, with total training time ranging from approximately 7 to 30 hours, depending on the dataset size. We observe that the model performance, measured by validation loss, typically stabilizes within the first 300 epochs. During training, memory usage peaked at approximately 20GB, primarily for storing intermediate activations and gradients. This represents only 25\% of the available GPU memory, indicating substantial headroom for parallelization. 

Inference was performed on a single NVIDIA A40 GPU with 48GB of memory. Computational time scaled linearly with both the input sequence length and prediction horizon. Generating 1,000 trajectories, each with 400 events, required approximately 120 minutes, with peak memory usage of 3.5GB, just 7\% of the GPU’s capacity. This low memory footprint suggests that inference tasks can be executed efficiently even on more modest hardware configurations.

Based on current cloud platform pricing (vast.ai as of May 2025), with conservative estimates of \$1.00/hour for A100 SXM4 and \$0.60/hour for A40 instances, our cost structure is highly competitive:
\begin{itemize}
    \item Single Training Run: \$10-\$30 per complete training cycle
\item Inference Generation: $<$\$2.00 per 1,000 trajectories
\end{itemize}

\noindent When leveraging parallel processing capabilities enabled by low memory utilization:
\begin{itemize}
    \item Parallel Training: \$2.50-\$7.50 per training run (4 concurrent jobs)
    \item Parallel Inference: \$0.50 per 1,000 trajectories (4 concurrent jobs)
\end{itemize}

While our experiments are limited to relatively simple modeling tasks, the observed resource efficiency suggests potential for scaling. Specifically, the modest memory and compute requirements open up the possibility of scaling model size or sequence length without requiring additional hardware, enabling more complex modeling under similar cost constraints. These preliminary findings indicate that Transformer-based stochastic modeling may offer a cost-effective and accessible alternative to traditional simulation approaches.

\subsection{Data Requirements} 
We now examine the data requirements for training our model. As operational data becomes increasingly available through advanced information systems, there is growing potential to train data-driven models at scale. Understanding how inference quality improves with training data volume is therefore key to realizing the full value of such data and guiding practical implementations.

To investigate this relationship, we conducted a series of experiments using $M/M/n$ queues, $n=1,5$, with homogeneous customer populations. The arrival rate $\lambda$ and service rate $\nu$ are set such that the system utilization satisfies $\rho = \lambda/n\nu = 0.6$. 

For each system and training dataset size, we evaluate three key performance metrics:
\textbf{(1)} Prediction loss - it is the sum of both the event prediction losses and time prediction losses as defined in Section \ref{sec:part_1_approach_validation}.  
 \textbf{(2)} Fraction of valid trajectories -- the proportion of generated trajectories that satisfy the routing and operational constraints of the queueing system (e.g., no departures from an empty system or idle servers);
 \textbf{(3)} KL divergence -- we assess the quality of our sequence model by estimating the distribution of average inter-arrival times across multiple simulated trajectories and measuring its KL divergence relative to the ground truth distribution obtained from the oracle discrete event simulator. 


Figures~\ref{fig:data_requirements_M_M_1} and~\ref{fig:data_requirements_M_M_5} present the three performance metrics for the M/M/1 and M/M/5 queues with varying training dataset size, respectively. In both cases, the Transformer's prediction loss decreases steadily as the training dataset size increases. However, the convergence behavior differs across systems:  for the M/M/1 queue, prediction loss stabilizes with approximately 1,000–2,000 training trajectories, whereas the M/M/5 queue requires 3,000–4,000 trajectories to achieve comparable performance. This gap reflects the greater structural complexity of multi-server systems, which demand more data to learn the dynamics of concurrent service and queue coordination accurately. 

Overall, we find that simple queueing systems demand only modest amounts of data, with little variation across different traffic intensities.  By contrast, our experiments reveal that data requirements increase with the number of event types. As the network’s complexity grows, so does the volume of data required.
A key direction for future work is to rigorously characterize how data requirements scale with the structural complexity of the underlying queueing network.

\begin{figure}[h!]
  \centering
  \begin{minipage}[b]{0.325\textwidth}
    \includegraphics[width=\textwidth, height=3.5cm]{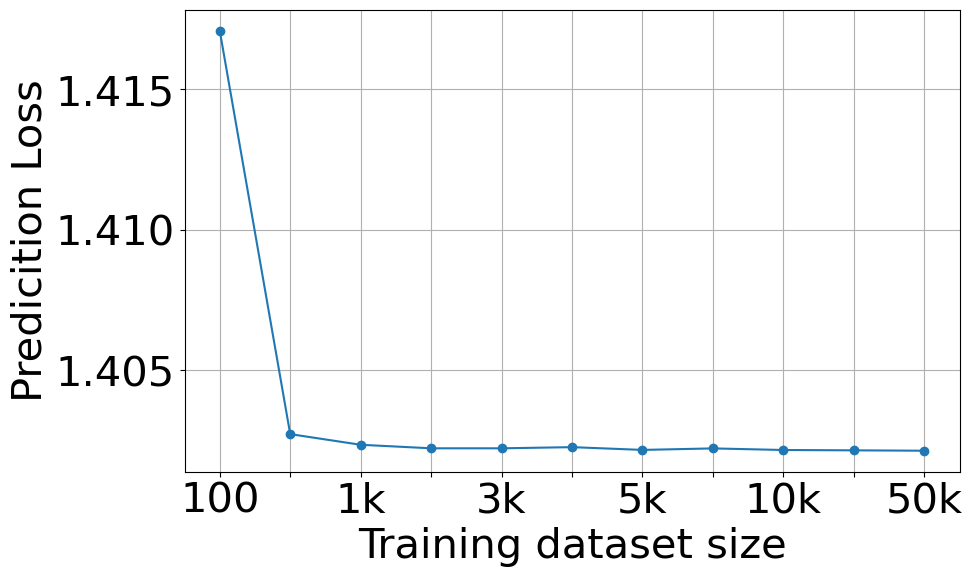}
    \centering{(a)}
  \end{minipage}
  \hfill
  \begin{minipage}[b]{0.325\textwidth}
    \includegraphics[width=\textwidth, height=3.5cm]{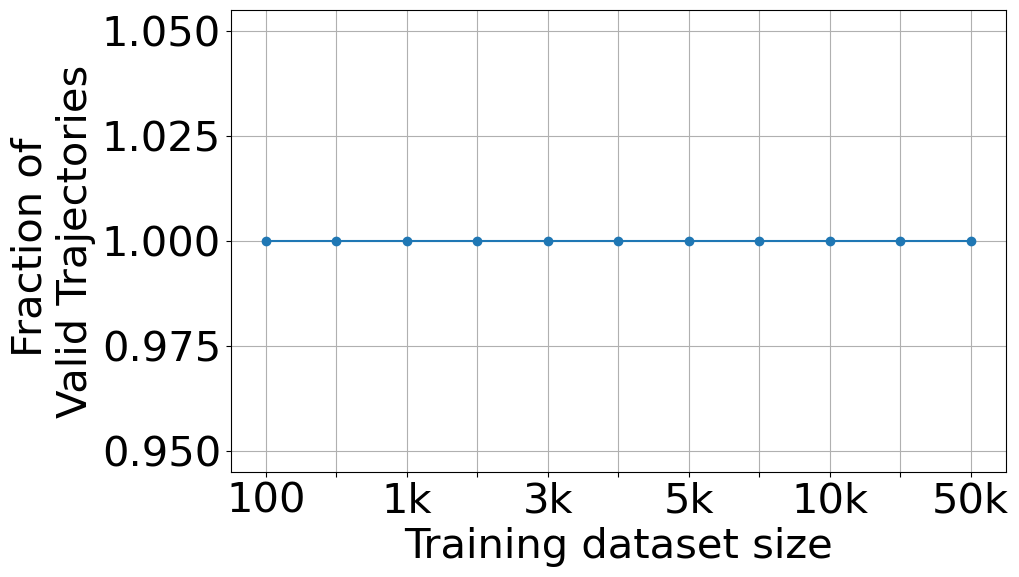}
    \centering
    {(b)}
  \end{minipage}
  \hfill
  \begin{minipage}[b]{0.325\textwidth}
    \includegraphics[width=\textwidth, height=3.5cm]{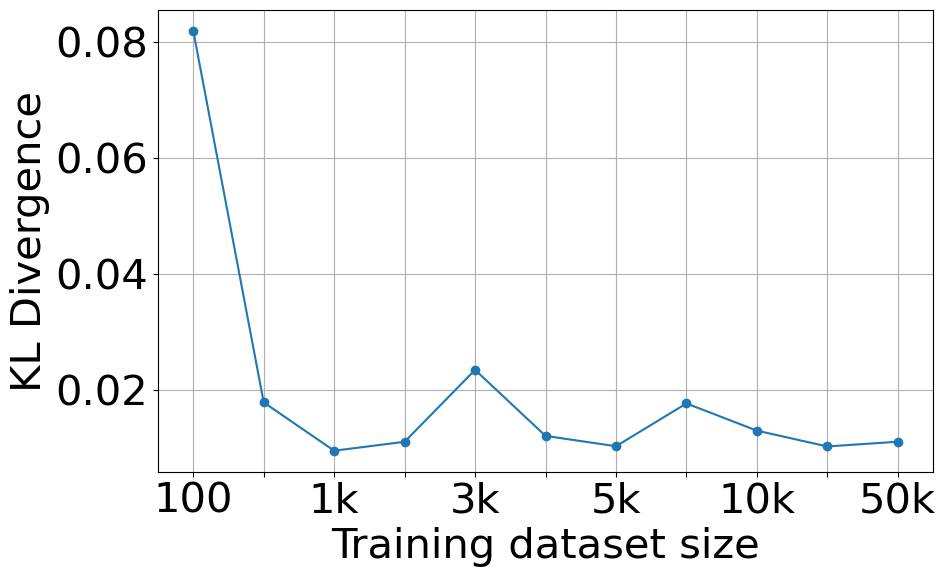}
    \centering
    {(c)}
  \end{minipage}
  \caption{\textbf{Data requirement analysis (M/M/1):}  Impact of training dataset size on sequence model performance for the M/M/1 queue: (a) prediction loss, (b) fraction of valid trajectories, and (c) KL divergence between average arrival time distributions predicted by the transformer model and those generated by the reference discrete event simulator.}
  \label{fig:data_requirements_M_M_1}
\end{figure}

\begin{figure}[h!]
 \centering
  \begin{minipage}[b]{0.325\textwidth}
    \includegraphics[width=\textwidth, height=3.5cm]{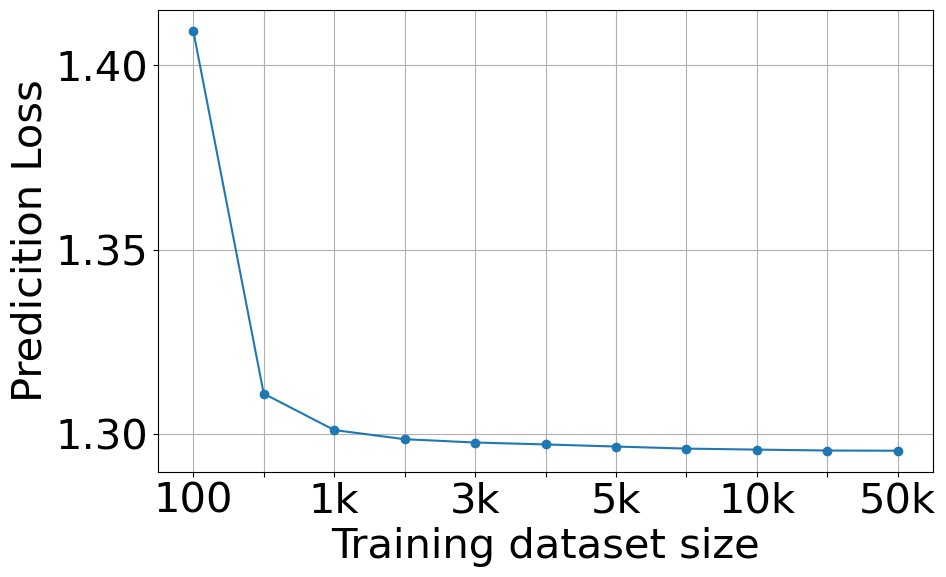}
    \centering{(a)}
  \end{minipage}
  \hfill
  \begin{minipage}[b]{0.325\textwidth}
    \includegraphics[width=\textwidth, height=3.5cm]{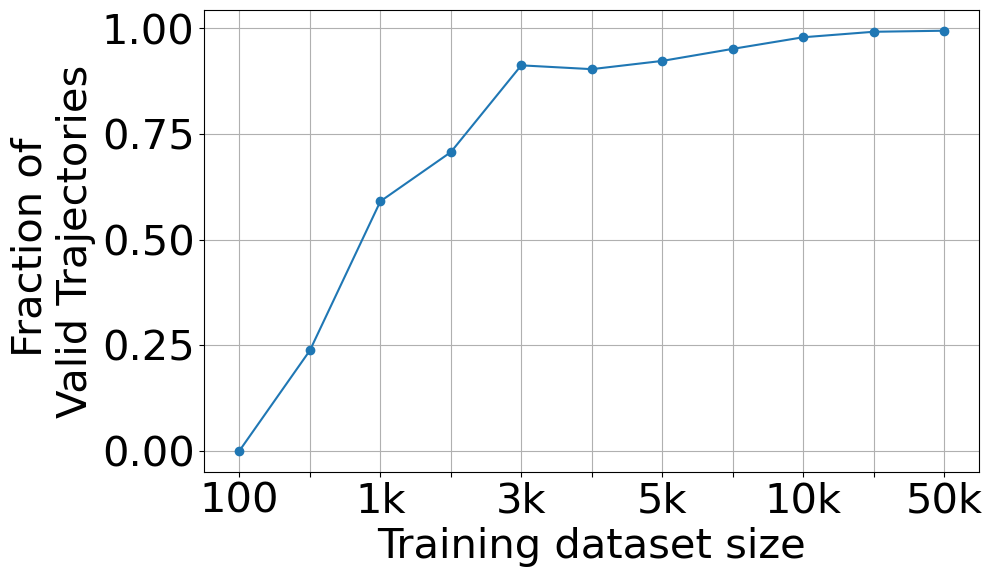}
    \centering
    {(b)}
  \end{minipage}
  \hfill
  \begin{minipage}[b]{0.325\textwidth}
    \includegraphics[width=\textwidth, height=3.5cm]{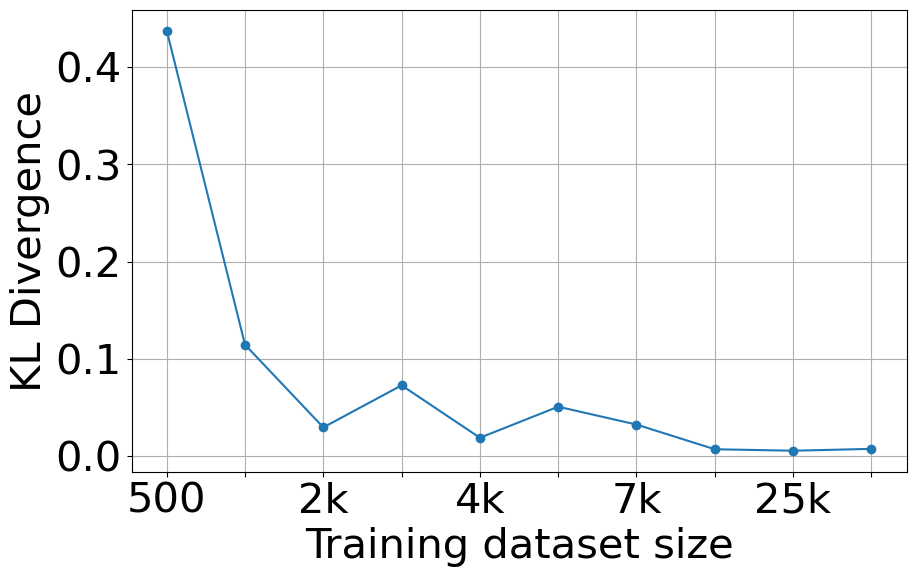}
    \centering
    {(c)}
  \end{minipage}
  \caption{\textbf{Data requirement analysis (M/M/5):}  Impact of training dataset size on sequence model performance for the M/M/5 queue: (a) prediction loss, (b) fraction of valid trajectories, and (c) KL divergence between average arrival time distributions predicted by the transformer model and those generated by the reference discrete event simulator.}
  \label{fig:data_requirements_M_M_5}
\end{figure}


\subsection{Towards Foundational Models} 
In data‑scarce settings, a promising strategy is to pre‑train a foundation sequence model on synthetic queueing data and then adapt it to new environments through in‑context learning. We explore this possibility in this section.

Recognizing the broader trend in AI toward foundation models, i.e., pretrained systems that generalize across multiple tasks, we investigate whether a similar paradigm could apply to queueing models. We present preliminary yet encouraging results. In particular, a Transformer model trained on a set of queueing networks exhibits the ability to adapt to input trajectories from any of these networks, indicating that in-context learning is a viable strategy in this setting.

We consider four distinct queueing models, each consisting of three nodes (or servers) connected through different network topologies (see Figure \ref{fig:different_queueing_models}).Training data is generated by simulating these models, and the Transformer is trained on the dataset
\[\mathcal{D}^{train}\equiv \left\{\left(S_1^{(j)},T_1^{(j)}, \cdots, E_{N}^{(j)} \right): 1\leq j\leq K_t \right\},\] 
where each event sequence $\left(S_1^{(j)}, T_1^{(j)}, \dots, E_N^{(j)}\right)$ is sampled from one of the four queueing networks, chosen uniformly at random. Across these models, there are nine distinct event types: arrivals to servers 1, 2, and 3; routing between nodes 1-2, 1-3, and 2-3; and departures from servers 1, 2, and 3. Each event type is encoded with a unique index.

To evaluate the model's in-context learning ability, we generate a historical trajectory $\mathcal{H}_n$
from one of the queueing networks and test whether the transformer can generate a valid continuation $\mathcal{H}_{n:N}$
that respects the routing and operational constraints of the originating network. Figure \ref{fig:fraction_correct} reports the fraction of correctly inferred network types as a function of the input trajectory length $n\in[1,100]$,
with each generated trajectory extended to a total length of $N=400$.

 \begin{figure}[h!]
    \centering
    \includegraphics[width=\textwidth]{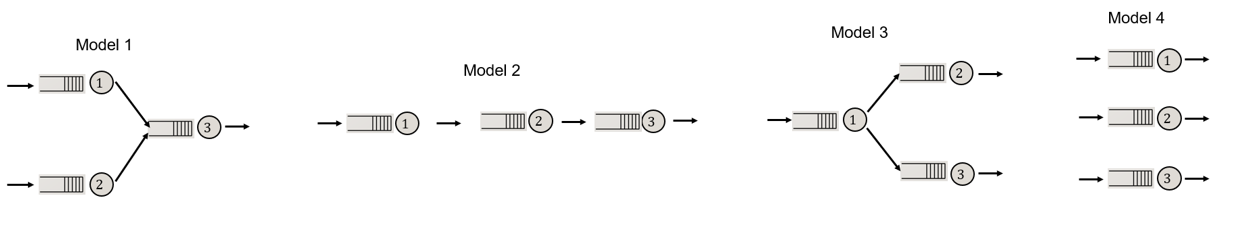}
    \caption{Different queueing models considered}
    \label{fig:different_queueing_models}
\end{figure}

\begin{figure}[h!]
    \centering
    \includegraphics[width=0.35\textwidth]{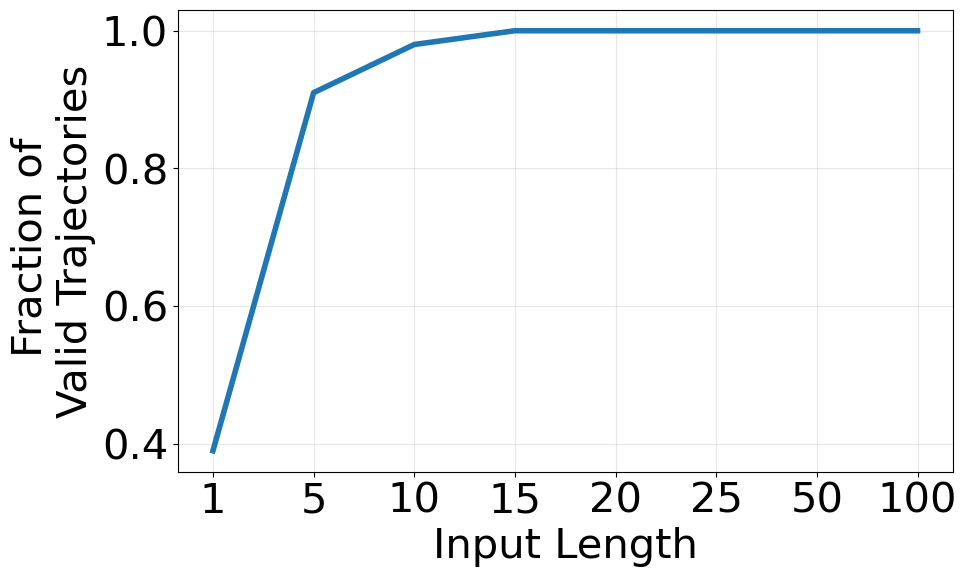}
    \caption{Fraction of times, transformer generated the trajectory corresponding to the correct input model trajectory. }
    \label{fig:fraction_correct}
\end{figure}

The results show that with as few as 10–15 initial events, the transformer consistently generates trajectories corresponding to the correct queueing model. This behavior suggests that the model can quickly infer the underlying structure from a short context and align its future predictions accordingly. While this may not be entirely unexpected -- since early events often include network-specific patterns or transitions -- it is nonetheless promising to observe such adaptation in practice. These findings suggest developing foundation models for queueing networks can be an interesting future research direction. Such foundation models could broaden access to stochastic modeling by enabling general-purpose, pretrained tools that adapt effectively across systems, even for systems with limited data.

\section{Conclusion}
\label{sec:conclusion}

In this paper, we propose a novel framework for stochastic modeling that leverages autoregressive sequence models to learn system dynamics directly from event-stream data, thereby lowering the barrier to access for sophisticated stochastic modeling tools. By recasting the traditional expert-driven process of specifying arrival processes, service mechanisms, and routing logic as a sequence distribution learning problem, our approach enables practitioners without queueing theory expertise to construct high-fidelity simulators using modern AI tools and readily available operational data. Our numerical experiments across diverse queueing systems highlight the flexibility of our method, demonstrating its effectiveness for simulation, uncertainty quantification, and counterfactual analysis.

Looking forward, our work opens several promising research directions. One is the development of foundation models for queueing systems that can adapt to new service environments with limited data. Another is the integration of policy optimization frameworks that utilize the learned simulators for automated decision-optimization. As information systems continue to generate increasingly rich event-stream data and AI tools become more accessible, we anticipate that data-driven approaches like ours will play a central role in narrowing the gap between queueing theory and practical service system management. Ultimately, this may facilitate broader adoption of stochastic modeling tools across a wide range of service domains.



\newpage

\bibliographystyle{abbrvnat}

\ifdefined\useorstyle
\setlength{\bibsep}{.0em}
\else
\setlength{\bibsep}{.7em}
\fi



  



\newpage
\ECSwitch

\ECHead{Appendix}

\section{DeFinetti's Partial Exchangeability}
\label{sec:Definneti_partial_exchangeability}
In this section, we state the DeFinetti's partial exchangeability for the continuous-time Markov chains.  Let $\mathcal S$ be a countable state space. Let $G$ be the set of functions from $[0,\infty)$ to $\mathcal S$. Let $\{Y(t):t\in [0,\infty)\}:= Y_{0:\infty}$ be a coordinate process on $G$.    Further, let $\mathcal{G}$ be the product $\sigma-$field on $G$. 
Fix $y_0 \in \mathcal S$, which will be the starting state. Consider $\Pi$ a set of probabilities $P$ over $\mathcal{G}$ with the following property:

 \textit{Let $P \in \Pi$ if and only if $P$ is a stochastic semigroup on a subset $ {S}_P$ of $ S$, with ${S}_P$ a single recurrent class of stable states and $y_0 \in \mathcal{S}_P$.}

Clearly, $\Pi$ is a standard Borel space.  Let $P_{y_0}$ make $\{Y(t)\}$ a markov chain with stationary transition $P$, starting from $y_0$. Thus $P_{y_0}$ is a probability on $\mathcal G$.

\noindent

\begin{definition} [Partial Exchangeable Stochastic process]
 Let $\mathcal{P}$ be a probability over $\mathcal G$ such that:
    \begin{enumerate}
        \item $\mathcal{P} (Y_0=y_0) =1$, and
        \item $\{{Y}_t\}$ has no fixed points of discontinuity, and
        \item $\mathcal{P} \{{Y}_n = y_0 ~ for ~  infintely ~ many ~ integers ~ n\} =1$ 
        \item for each $h>0$, the $\mathcal{P}$-law of $\{{Y}_{nh}:n=1,2,...\}$ depends only on transition counts.
    \end{enumerate}
    \label{def:definneti_assumption}
\end{definition}

Let $h = \frac{1}{2^k}$ for some $k =0,1,...$. Let $\mathcal{G}_h$ be the tail $\sigma-$field of $\{{Y}_{nh}:n=0,1,...\}$. Let $\mathcal{G}_0$ be the $\sigma-$field spanned by $ \cup_h \mathcal{G}_h$.

\begin{theorem}[De Finetti's Partial Exchangeability] If $\mathcal{P}$ satisfy definition \ref{def:definneti_assumption} then there exists a set $G_0 \in \mathcal{G}_0$ with $\mathcal{P}(G_0)=1$ and for each $\omega \in G_0$, there is a stochastic semigroup $\alpha_{\omega}$ with state space $S_{\omega}$ such that 
    \begin{enumerate}
        \item $y_0 \in S_\omega $
        \item $S_\omega$ consists of one recurrent class of state relative to $\alpha_{\omega}$, and
        \item Given $\mathcal{G}_0$, process $\{{Y}_t\}$ is conditionally Markov with stationary transitions $\alpha_{\omega}$ starting from $y_0$.
    \end{enumerate}

    The distribution $\tilde{\mu}$ of $\{\alpha_ \omega:\omega \in G_0\}$ is unique and $\mathcal{P} = \int \alpha \tilde{\mu} (d\alpha)$.
    \label{thm:definneti_thm}
\end{theorem}

\section{Proofs for Section \ref{sec:theory}}
\
In this section we present the proofs for the results presented in Section \ref{sec:theory}.

\subsection{Proof of Lemma \ref{lemma:preliminary_lemma}}
\begin{proof} Let  $ 
P_t:= P(X_{0:t})$,
$\what P_t := P(\what X_{0:t})$,
$P_t^f:= P\!\bigl(f(X_{0:t})\bigr)$, $\what P_t^f:=P\!\bigl(f(\what X_{0:t})\bigr)$.
\begin{align*}
   W\!\bigl(\what P_t^f,P_t^f\bigr)
     & =^{(i)}\sup_{\lVert h\rVert_{\mathrm{Lip}}\le1}
   \bigl[\,\E_{\what P_t^f}[h]-\E_{P_t^f}[h]\,\bigr] \\[8pt]
     & \leq^{(ii)} 2\lVert f\rVert_\infty\,\mathrm{TV}\!\bigl(\what P_t^f,P_t^f\bigr) \\
      & \leq^{(iii)} 2\lVert f\rVert_\infty\,\sqrt{\tfrac12\,\mathrm{KL}\!\bigl(P_t^f\!\parallel\!\what P_t^f\bigr)}\\
      & \leq^{(iv)} \sqrt{2}\lVert f\rVert_\infty\,\sqrt{\mathrm{KL}\!\bigl(P_t\!\parallel\!\what P_t\bigr)}.
\end{align*} 
In the above derivation, (i) follows from Kantorovich-Rubinstein duality; (ii) follows as we assume $f$ is bounded, i.e., $\what P_t^f$ and $P_t^f$ are supported in an interval of diameter $D\leq 2\|f\|_{\infty}$; (ii) follows from Pinsker's inequality; (iv) follows from the data processing inequality for KL under measurable maps.
\end{proof}

\subsection{Proof of Proposition \ref{prop:regeneration}}

\begin{proof}
In this proof, to lighten notation, we suppress the dependence on $\theta$ since it is fixed.

Recall that $X(t)$ is Harris recurrent on $(\mathcal{S}, \Sigma)$ and there exists a measurable set $C\in \Sigma$,   a constant $\varepsilon>0$, a probability measure $\varphi$ on $(\mathcal S, \Sigma)$, and $\zeta>0$,  such that
\begin{equation}\label{eq:Harris}
P_x(X(\zeta)\in B) \;\ge\; \varepsilon\,\varphi(B),\qquad x\in C,\ B\in \Sigma. \end{equation}

We now provide details of the one-dependent regeneration cycle structure. In the presence of \eqref{eq:Harris}, we can write 
${P}(X(\zeta)\in \cdot|X_0=x)$ for $x\in C$ in the form

\begin{equation}\label{eq:regen2}
{P}(X(\zeta)\in \cdot|X_0=x)=\epsilon\varphi(\cdot)+(1-\epsilon)R(x,\cdot),
\end{equation}
where $R(x,\cdot)$ is a probability measure on $(\mathcal S, \Sigma)$ defined through \ref{eq:regen2}; see, \citet{glynn11,Sigman90}. We now let $\tilde{T}(0)=\inf\{t\geq 0: X(t)\in C\}$ be the first hitting time of $C$, and set
\[\tilde{T}(k)=\inf\{t\geq \tilde{T}(k-1)+\zeta: X(t)\in C\} \qquad k\geq 1,\]
i.e., $ \tilde{T}(k)$'s are successive entry times to $C$ that are each separated by at least $\zeta$. At each time $ \tilde{T}(k)$ at which $X(t)$ visits $C$, we generate an independent Bernoulli$(\epsilon)$ rv $I_k$. If $I_k=0$, we sample $X({\tilde{T}(k)+\zeta})$ according to $R(X({\tilde{T}(k)}),\cdot)$, followed by generating $(X(t):\tilde{T}(k)<t< \tilde{T}(k)+\zeta)$ from the conditional distribution 
\begin{equation}\label{eq:R}
{P}\Big(\left(X(t):\tilde{T}(k)<t<\tilde{T}(k)+\zeta\right)\Big|X(\tilde{T}(k)),X(\tilde{T}(k)+\zeta)\Big).
\end{equation}
(We use the fact that $\mathcal{S}$ is a Polish space to ensure the existence of the conditional distribution.) If $I_k=1$, we sample $X({\tilde{T}(k)+\zeta})$ according to $\varphi(\cdot)$, followed again by generating $(X(t): \tilde{T}(k)<t<\tilde{T}(k)+\zeta)$ from the conditional distribution \eqref{eq:R}.

Let $\gamma(0)=\inf\{j\geq 0: I_j=1\}$ and for $k\geq 1$, 
\[
\gamma(k)=\inf\{j>\gamma(k-1): I_j=1\}.
\]
Let $T(k)=\tilde{T}(\gamma(k))+\zeta$ for $k\geq 0$ and
$\tau(k)= T(k)- T(k-1)$ for $k\geq 1$. 
We now define, for $k\geq 1$,
\[
\mathcal{E}_k=(\tau(k), X(t): T_{k-1} \leq t < T_k)
\]
as the $k$-th ``cycle" associated with $X(t)$. Note that $\{X(t): t\geq T(k)\}$ is independent of $\{X(t): t\leq \tilde{T}(\gamma(k))\}$.
It is also important to recognize that in general, $(X(t):\tilde{T}(\gamma(k))<t< T(k))$ is correlated with $X(\tilde{T}(k)+\zeta)$, i.e., $X(T(k))$. This correlation implies that $\mathcal{E}_k$'s are one-dependent. 

In addition, at each regeneration time $T(k)$, the future path ${X(t):t\ge T_k}$ depends on the past only through $X(T(k))$. Since $T(k)$ is measurable w.r.t. $\{ (\tilde T(j),I_j):j\le \gamma(k)\}$) and
$X(T(k))\sim\varphi$ is sampled independently of that information, $X(T(k))$ and $T(k)$ are independent. Hence,  $\{X(t): t\geq T(k)\}$ is independent of $T(k)$.

Next, by construction, $\tau(k)$ is determined by the randomness up to the success time $\tilde T(\gamma(k))=T(k)$ and does \emph{not} involve the fresh randomization used later to draw
$X(T_k)\sim\varphi$ or to generate the $\zeta$--window bridge on
$\big(\tilde T(\gamma(k)),\,T_k\big)$. Hence $X(T(k))$ is independent of $\tau(k)$.
The next spacing $\tau(k+1)$ depends only on the starting state $X(T(k))$
and on fresh randomness after time $T(k)$; therefore
$\tau(k+1)$ is independent of $\tau(k)$. Since $X(T_k)\sim\varphi$ for all $k$, the
$\tau(k)$ have a common distribution. Consequently, $(\tau(k): k\ge1)$ are i.i.d.

\end{proof}

\subsection{Proof of Lemma \ref{lemma:convergence_kl_div_lemma}}

\begin{proof}
By the chain rule for KL divergence, 
\[
\dkl{\P_{0:t+s}}{\what{P}_{0:t+s}} = \dkl{\P_{0:t}}{ \what{P}_{0:t}} + \E_{\P_{0:t}}\left[\dkl{\P(X_{t:t+s}|X_{0:t})}{\what{P}(X_{t:t+s}|X_{0:t})}\right].
\]
Hence $\dkl{\P_{0:t}}{\what{P}_{0:t}}$ is nondecreasing in $t$. Moreover, since
\[
\dkl {\mu}{\what{\mu}} = \dkl{\P_{0:t}}{\what{P}_{0:t}} + \E_{\P_{0:t}} \left[\dkl {\mu(\cdot|X_{0:t})}{\what{\mu}(\cdot|X_{0:t})} \right],
\]
we have $\dkl{\P_{0:t}}{\what{P}_{0:t}} \le \dkl{\mu}{\what{\mu}}$ for all $t$. It follows that
$\dkl{\P_{0:t}}{\what{P}_{0:t}}$ converges as $t\to\infty$ and

\begin{equation}
    \lim_{t \to \infty} \dkl{\P_{0:t}}{\what{P}_{0:t}} \leq \dkl{\mu}{\what{\mu}}.
    \label{eq:dkl_limit_posterior_rel}
\end{equation}

Next, note that
\begin{equation}
    \dkl{\P}{\what{P}} = 
\lim_{t \to \infty} \dkl{\P_{0:t}}{\what{P}_{0:t}}.
\label{eq:dkl_limit_overall}
\end{equation}
The discrete‑time version of \eqref{eq:dkl_limit_overall} appears in \cite[Corollary 7.3]{gray11}; we adapt it to continuous time in Theorem \ref{thm:kl-convergence-theorem}.
Furthermore,
\begin{equation}\label{eq:decomp2}
\dkl {\mu}{\what{\mu}}
=\dkl{\P}{\what{P}}
+ \E_{\P} \left[\dkl{\mu(\theta|X_{0:\infty})}{\what{\mu}(\theta|X_{0:\infty})}\right].
\end{equation}

Let $T(k)$ be the regeneration times of the process  $\{X_t\}$ as defined in Proposition \ref{prop:regeneration}. 
Fix $\theta$ and choose $A,B\in\Sigma$ witnessing the TV-separation in Assumption~\ref{ass:separability}:
$q_\theta:=Q_\theta(A\times B)\neq q_{\theta'}:=Q_{\theta'}(A\times B)$ for $\theta'\neq\theta$.
Let $T(j)$ be the regeneration times (we suppress $\theta$ in notation) and define for odd indices
\[
W_k\ :=\ \mathbf 1\!\left\{\,X\big(T(2k-1)\big)\in A,\ X\big(T(2k-1)+\tilde{\zeta}\big)\in B\,\right\},\qquad k\ge1.
\]
From Proposition \ref{prop:regeneration}, we have that
\[\{X(t):  t < T(2k) \} \perp \{X(t):  t \geq T(2k+1) \},\qquad k\ge1. \]
By Assumption~\ref{ass:separability}, $T(2k-1)<T(2k-1)+\tilde{\zeta}< T(2k)$ a.s.. Therefore, 
\[\{X(t): T(2k-1)\leq t \leq T(2k-1) + \tilde{\zeta}\} \perp \{X(t): T(2k+1)\leq t \leq T(2k+1) + \tilde{\zeta}\} ,\qquad k\ge1.\]
Hence,  the sequence $(W_k:k\ge1)$ is i.i.d.. Under $P_{\theta}$,
\[
\mathbb E_{\theta}[W_k]=\int_A P^\theta_{\tilde{\zeta}}(x,B)\,\varphi_\theta(dx)=q_\theta.
\]
By the strong law of large numbers,
\[
\frac{1}{K}\sum_{k=1}^{K}W_k \rightarrow q_\theta \quad \mbox{ $P_{\theta}$ a.s. as $K\rightarrow\infty$}.
\]
Set $E_\theta:=\Big\{\omega:\lim_{K\to\infty}K^{-1}\sum_{k=1}^K W_k(\omega)=q_\theta\Big\}$. Then $P_\theta(E_\theta)=1$, while for any $\theta'\neq\theta$, the same law of large numbers gives $P_{\theta'}(E_\theta)=0$ because the limit equals $q_{\theta'}\neq q_\theta$.
Thus, ${P_\theta:\theta\in\Theta}$ are pairwise mutually singular on the path space $D([0,\infty),\mathcal S)$.

As $D([0,\infty),\mathcal S)$ is standard Borel, there exists a measurable classifier $T:D([0,\infty),\mathcal S)\to\Theta$ with $P_\theta(T(X_{0:\infty})=\theta)=1$ for all $\theta$. Let $\mathcal F_t:=\sigma(X(s):0\le s\le t)$ and $\mathcal F_\infty:=\sigma\bigl(X_{0:\infty}\bigr)$. For a prior $\mu$ on $\Theta$, consider the joint law $\P(d\theta,d\omega):=\mu(d\theta)\, P_\theta(d\omega)$. Then a version of the full-path posterior is the Dirac mass
$\mu(\cdot\mid\mathcal F_\infty)=\delta_{T(X_{0:\infty})}$ a.s.. Moreover, for each $B\in\mathcal B(\Theta)$ the process $M_t(B):=\mu(B\mid \mathcal F_t)$ is a bounded martingale; by Doob’s theorem,
\[
\mu(\cdot\mid\mathcal F_t) \rightarrow \mu(\cdot\mid\mathcal F_\infty)
=\delta_{T(X_{0:\infty})} \quad \mbox{ a.s. as $t\to\infty$}.
\]
i.e., the conditional KL term in \eqref{eq:decomp2} vanishes. Therefore, 
\begin{equation}
    \dkl {\mu}{\what{\mu}}
=\dkl{\P}{\what{P}}.
\label{eq:equivalence_kl_prior}
\end{equation}
Combining \eqref{eq:dkl_limit_overall} and \eqref{eq:equivalence_kl_prior} yields
\[\lim_{t \to \infty} \dkl{\P_{0:t}}{\what{P}_{0:t}} =  \dkl{\mu}{\what{\mu}}.\]
\end{proof}

\begin{theorem}
Let $(\Omega,\mathcal F)$ be a measurable space with an increasing family
$(\mathcal F_{0:t})_{t\ge 0}$ such that $\sigma\!\big(\bigcup_{t\ge 0}\mathcal F_{0:t}\big)=\mathcal F$.
Let $\P$ and $\what{P}$ be probability measures on $(\Omega,\mathcal F)$ and write the
restrictions $\P_{0:t}:=\P|_{\mathcal F_{0:t}}$ and $\what{P}_{0:t}:=\what{P}|_{\mathcal F_{0:t}}$.
If $\P\ll \what{P}$, then
\[
\dkl{\P}{\what{P}}
\;=\;\lim_{t\to\infty}\dkl{\P_{0:t}}{\what{P}_{0:t}},
\]
with the convention that all terms equal $+\infty$ if $\P\not\ll \widehat P$.
\label{thm:kl-convergence-theorem}
\end{theorem}

\begin{proof}
Assume $\P\ll \what{P}$. Let $L:=\frac{d\P}{d\what{P}}$ and define the likelihood-ratio
martingale $\Lambda_t:=\E_{\what{P}}[\,L\,|\,\mathcal F_{0:t}]$.

\textbf{Finite-time Radon--Nikodym derivative.}
For any $A\in\mathcal F_{0:t}$,
\[
\int_A \Lambda_t\,d\what{P}
=\int_A \E_{\what{P}}[L\,|\,\mathcal F_{0:t}]\,d\what{P}
=\int_A L\,d\what{P}
=\P(A)=\P_{0:t}(A).
\]
Thus, $\frac{d\P_{0:t}}{d\widehat P_{0:t}}=\Lambda_t$ $\widehat P$-a.s.
By $\mathcal F_{0:t}$-measurability of $\Lambda_t$,
\[
\dkl{\P_{0:t}}{\what{P}_{0:t}}
=\int \log\!\Big(\frac{d\P_{0:t}}{d\what{P}_{0:t}}\Big)\,d\P_{0:t}
=\int \Lambda_t \log \Lambda_t\, d\what{P}_{0:t}
=\E_{\what{P}}[\Lambda_t\log \Lambda_t],
\]
Likewise,
\[
\dkl{\P}{\what{P}}
=\int \log L\, d\P
=\int L\log L\, d\what{P}
=\E_{\what{P}}[L\log L].
\]
Let $\psi(x):=x\log x+1-x$, convex and nonnegative on $[0,\infty)$; then
$\dkl{\P_{0:t}}{\what{P}_{0:t}}=\E_{\what{P}}[\psi(\Lambda_t)]$ and
$\dkl{\P}{\what{P}}=\E_{\what{P}}[\psi(L)]$.

\textbf{Upper bound}
Since $\Lambda_t=\E_{\what{P}}[L\,|\,\mathcal F_{0:t}]$ and $\psi$ is convex, Jensen's inequality gives
\[
\dkl{\P_{0:t}}{\what{P}_{0:t}}
=\E_{\what{P}}[\psi(\Lambda_t)]
\;\le\; \E_{\what{P}}[\psi(L)]
=\dkl{\P}{\what{P}}
\qquad\text{for all }t.
\]

\textbf{Lower bound (martingale convergence and Fatou).}
$(\Lambda_t)_{t\ge 0}$ is a nonnegative $\what{P}$-martingale with $\E_{\what{P}}[\Lambda_t]=1$.
By the martingale convergence theorem, $\Lambda_t\to L$ $\what{P}$-a.s.. 
Since $\psi\ge 0$ and is continuous, Fatou’s lemma yields
\[
\dkl{\P}{\what{P}}
=\E_{\what{P}}[\psi(L)]
\;\le\; \liminf_{t\to\infty}\E_{\what{P}}[\psi(\Lambda_t)]
=\liminf_{t\to\infty}\dkl{\P_{0:t}}{\what{P}_{0:t}}.
\]

Combining the upper and lower bounds, we have
\[
\dkl{\P}{\what{P}}
=\lim_{t\to\infty}\dkl{\P_{0:t}}{\what{P}_{0:t}}.
\]
\end{proof}

\subsection{Proof of Theorem \ref{thm:final_thm}}



\begin{proof}
Let $P_f:= P(f({\theta}))$,
$\what{P}_f := P(f(\what{\theta}))$, 
$\what P_t^f:=P\!\bigl(f(\what X_{0:t})\bigr)$, where $\theta\sim \mu$ and $\what{\theta}\sim \what{\mu} $. Also define $\text{dia}(f) = \sup_{\theta, \theta'\in \Theta} |f(\theta)-f(\theta')|$.
\begin{align*}
    W\!\bigl(\what P_t^f,P_f\bigr)
      & \leq^{(i)}
     W\!\bigl(\what P_f,P_f\bigr) + W\!\bigl(\what P_t^f,\what P_f\bigr)
      \\
      & \leq^{(ii)}  W\!\bigl({f}(\what\theta),{f(\theta)}\bigr)
      + \frac{c_2}{\sqrt{t}} \\
      & \leq^{(iii)} 2\text{ dia}(f) \,\mathrm{TV}\!\bigl(\what{P}_f,P_f\bigr) + \frac{c_2}{\sqrt{t}} \\
      & \leq^{(iv)} 2\text{ dia}(f) \,\mathrm{TV}\!\bigl(\what{\mu},\mu\bigr) + \frac{c_2}{\sqrt{t}} \\
      & \leq^{(v)} \sqrt{2}\text{ dia}(f)\,\sqrt{\mathrm{KL}\!\bigl({\mu}\!\parallel\!\what \mu\bigr)} + \frac{c_2}{\sqrt{t}} \\ 
      & =^{(vi)} c_1 \sqrt{\mathrm{KL}\!\bigl({\mu}\!\parallel\!\what \mu\bigr)} + \frac{c_2}{\sqrt{t}} 
\end{align*}
For the above derivation, (i) follows from the triangle inequality applied to the Wasserstein distance; (ii) follows from Lemma \ref{lemma:performance_measure_dist_CLT}; (iii) follows from Kantorovich-Rubinstein duality and the bound $|\E_Ph-E_{\what P}h|\leq 2\|h\|_{\infty}\mathrm{TV}(P, \what P)$; (iv) follows from the data processing inequality under measurable maps; (v) follows from Pinsker's inequality. 
In (vi), we set $c_1 = \sqrt{2} \text{ dia}(f)  $.

\end{proof}

\begin{lemma} 
    Let $X_{0:t}\sim \P$, where ${{\P}}({X}_{0:\infty}) = \int P_{\theta}({X}_{0:\infty}) d\mu(\theta)$, then 
    \begin{align}
        W(f({X}_{0:t}), f(\theta)) & \leq \frac{c_2}{\sqrt{t}}.
    \end{align}
   \label{lemma:performance_measure_dist_CLT}
\end{lemma}

\begin{proof}
Let $P_f:= P(f({\theta}))$ and $P_t^f:= P(f(X_{0:t}))$. 
\begin{align*}
        W(P_t^f, P_f)  & =^{(i)} \inf_{\pi\in \Pi\Big( P_t^f, P_f \Big)} \mathbb{E}_{x,y\sim \pi} [\|x-y\|_2]
\\ & \leq^{(ii)} \mathbb{E}_{\theta \sim \mu} \mathbb{E}_{{X}_{0:t} \sim P_\theta} [\|f({X}_{0:t})-f(\theta)\|_2]
\\ & \leq^{(iii)} \mathbb{E}_{\theta \sim \mu} \left( \frac{a_1(\theta)}{\sqrt{t}} + \frac{b_1(\theta)}{t} \right)
        \\ & \leq^{(iv)} \frac{c_2}{\sqrt{t}}.
    \end{align*}
For the above derivation:
(i) follows from the definition of 1-Wasserstein distance.
(ii) follows if we couple \(\theta\sim\mu\) with \(X_{0:t}\sim P_\theta\), and set \(x=f(X_{0:t})\), \(y=f(\theta)\).
(iii) is established below in Theorem \ref{thm:L1_bound_harris_recurrent}; where
  \[a_1(\theta) = \sqrt{\frac{6\big(\ell_2(\theta)+f^2(\theta)\,m_2(\theta)\big)}{m_1(\theta)}}
  \]
  \[ b_1(\theta) = k_1(\theta)+ \ell_1(\theta)+\frac{\sqrt{m_2(\theta) \ell_2(\theta)}}{m_1(\theta)}
+|f(\theta)|\left(m_1(\theta)+\tfrac{m_2(\theta)}{2m_1(\theta)}\right)\]
with $
\ell_1(\theta) := \mathbb E_{P_\theta}\!\left[\int_{T(0)}^{T(1)} |h(X_s)|\,ds\right],$ $
\ell_2(\theta):=\mathbb E_{P_\theta}\!\left[\left(\int_{T(0)}^{T(1)} |h(X_s)|\,ds\right)^2\right],$
$m_1(\theta):=\mathbb E_{P_\theta}[T_1]$, $
m_2(\theta):=\mathbb E_{P_\theta}[T_1^2],$ and $
 k_1(\theta) := \mathbb E_{P_\theta}\left[ \int_{0}^{T(0)} \left|h(X_s)\right|\,ds\right]  + |f(\theta)| \mathbb E_{P_\theta}\left| T(0)\right|. 
    $
(iv) follows by choosing 
        \[
           c_2
           \;=\;
           \sup_\theta
           \{a_1(\theta)
           \;+\; b_1(\theta)\},
        \]
        which is finite under the  assumptions
        \(\sup_\theta\mathbb E_{P_\theta}[T(1)]<\infty\), \(\sup_\theta\mathbb E_{P_\theta}[T(1)^2]<\infty\),
        \(\inf_\theta\mathbb E_{P_\theta}[T(1)]>0\),  \(\sup_\theta\mathbb E_{P_\theta}[T(0)]<\infty\), $\sup_{\theta}\mathbb{E}_{P_\theta}[\int_{s=T(0)}^{T(1)} |h(X(s))| ds] < \infty$, $\sup_{\theta}\mathbb{E}_{P_\theta}[\int_{s=T(0)}^{T(1)} |h(X(s))|^2 ds] < \infty$ and $\sup_{\theta}\mathbb{E}_{P_\theta}[\int_{s=0}^{T(0)} |h(X(s))| ds] < \infty$.
\end{proof}


\newcommand{\Var}{\operatorname{Var}}
\newcommand{\Cov}{\operatorname{Cov}}
\newcommand{\EE}{\mathbb{E}}
\newcommand{\PP}{\mathbb{P}}

\begin{theorem}
Let $(X(t))_{t\ge0}$ be a continuous-time enhanced positive Harris recurrent Markov process on a Polish space $\mathcal{S}$ with $\sigma$-algebra $\Sigma$ defined in Assumption \ref{ass:PHR}.
$0=T(-1)<T(0)<T(1)<\cdots$ are the regeneration times defined in Proposition \ref{prop:regeneration}.
Define the  cycle integral
\[ Z_j\;:=\;\int_{T(j-1)}^{T(j)} h(X(s))\,ds,
\]
and set $\beta:=\dfrac{\mathbb E[Z_1]}{\mathbb E[\tau(1)]}$ (assumed finite).
Let
\[
A(t):=\int_0^t h(X(s))\,ds,
\qquad N(t):=\max\{n:\ T(n)\le t\}.
\]
Suppose Assumptions \ref{ass:cycle-moments} and \ref{ass:function_assumption} hold, i.e.,
\[
\ell_1 := \mathbb E\!\left[\int_{T(0)}^{T(1)} |h(X(s))|\,ds\right]<\infty,\quad
\ell_2:=\mathbb E\!\left[\left(\int_{T(0)}^{T(1)} |h(X(s))|\,ds\right)^2\right]<\infty,\]
\[
m_1:=\mathbb E[\tau(1)]<\infty,\quad
m_2:=\mathbb E[\tau(1)^2]<\infty.
\]

\[ k_1 := \mathbb E\left[ \int_{0}^{T(0)} \left|h(X(s))\right|\,ds\right]  + |\beta| \mathbb E\left| T(0)\right| <\infty 
    \]  
Then,
\[ \lim_{t \to \infty}\frac{A(t)}{t} = \beta  \quad a.s.\]
Additionally, for all $t>0$,
\begin{equation}\label{eq:explicit-second-moment}
\E\left|\frac{A(t)}{t}-\beta\right|
\;\le\;
\sqrt{\frac{6\big(\ell_2+\beta^2\,m_2\big)}{m_1}}\;\frac{1}{\sqrt t}
+\left(k_1+ \ell_1+\frac{\sqrt{m_2 \ell_2}}{m_1}
+|\beta|\left(m_1+\tfrac{m_2}{2m_1}\right)\right)\frac{1}{t}.
\end{equation}






\label{thm:L1_bound_harris_recurrent}
\end{theorem}

\begin{proof}
We prove \eqref{eq:explicit-second-moment} first.
Since $T(N(t))\le t< T(N(t)+1)$, write
\[
A(t)-\beta t
= \int_{0}^{T(0)} h(X_s)\,ds - \beta T(0) + \sum_{j=1}^{N(t)} \bigl(Z_j-\beta \tau_j\bigr)
+ \underbrace{\int_{T(N(t))}^{t} h(X_s)\,ds}_{=:R(t)}
- \beta\,\underbrace{\bigl(t-T\bigl(N(t)\bigr)\bigr)}_{=:R_T(t)}.
\]
Define $\tilde{Z}_j:=Z_j-\mu \tau_j$ for $j\geq 1$. By the triangle inequality,
\begin{equation}\label{eq:two-terms}
\mathbb E\left|\frac{A(t)}{t}-\beta\right|
\;\le\; \underbrace{\frac{1}{t}\mathbb E\left| \int_{0}^{T(0)} h(X(s))\,ds - \beta T(0)\right|}_{\mathrm{(I)}} + 
\underbrace{\mathbb E\left|\frac{1}{t}\sum_{j=1}^{N(t)} \tilde{Z}_j\right|}_{\mathrm{(II)}}
\;+\;
\underbrace{\frac{1}{t}\,\mathbb E|R(t)|}_{\mathrm{(III)}} + \underbrace{\frac{|\beta|}{t}\,\mathbb E[R_T(t)]}_{\mathrm{(IV)}}.
\end{equation}

\medskip

\begin{enumerate}
    \item \textbf{Term  \(\mathrm{(I)}\): Initial cycle term.}

    \begin{align*}
    \frac{1}{t}\mathbb E\left| \int_{0}^{T(0)} h(X(s))\,ds - \beta T(0)\right| &\leq \frac{1}{t}\mathbb E\left[ \int_{0}^{T(0)} \left|h(X_s)\right|\,ds\right]  +\frac{1}{t} |\beta| \mathbb E\left| T(0)\right| 
    = \frac{k_1}{t}.   
    \end{align*}

    \item  \textbf{Term \(\mathrm{(II)}\): One-dependent variance bound.}
Condition on $N(t)$ and note $\mathbb E[\tilde{Z}_j]=0$ by definition of $\beta$.
For any $n\ge1$,
\[
\text{Var}\!\Big(\sum_{j=1}^n \tilde{Z}_j\Big)
= n\,\text{Var}(\tilde{Z}_1) + 2\sum_{k=1}^{n-1} (n-k)\,\text{Cov}(\tilde{Z}_1,\tilde{Z}_{1+k}).
\]
Since the sequence is one-dependent, $\text{Cov}(\tilde{Z}_1,\tilde{Z}_{1+k})=0$ for $k\ge2$, thus
\[
\text{Var}\!\Big(\sum_{j=1}^n \tilde{Z}_j\Big)
= n\,\text{Var}(\tilde{Z}_1) + 2(n-1)\,\text{Cov}(\tilde{Z}_1,\tilde{Z}_2)
\;\le\; n\,\sigma_\star^2,
\]
where $\sigma_\star^2 \;:=\; \text{Var}(\tilde{Z}_1) + 2\,\bigl|\text{Cov}(\tilde{Z}_1,\tilde{Z}_2)\bigr| \;<\;\infty.$ Therefore,
\[
\mathbb E\!\left[ \Big(\sum_{j=1}^{N(t)} \tilde{Z}_j\Big)^2 \,\Big|\, N(t)\right]
\le N(t)\,\sigma_\star^2.
\]
Taking expectations and using Jensen's inequality,
\[
\mathrm{(II)}
= \frac{1}{t}\,\mathbb E\left|\sum_{j=1}^{N(t)} \tilde{Z}_j\right|
\le \frac{1}{t}\,\sqrt{\,\mathbb E\Big(\sum_{j=1}^{N(t)} \tilde{Z}_j\Big)^2\,}
\le \frac{1}{t}\,\sqrt{\,\sigma_\star^2\,\mathbb E[N(t)]\,}.
\]
As $\tau(j)$'s are i.i.d.\ with mean $m_1$,
$
\E[N(t)]=\frac{1}{m_1}\E[T(N(t))]\le t/m_1$. Then,
\[
\mathrm{(II)}
= \frac{1}{t}\,\mathbb E\left|\sum_{j=1}^{N(t)} \tilde{Z}_j\right|
\le {\frac{\sqrt{\sigma_\star^2}}{\sqrt{m_1}}\,} \frac{1}{\sqrt{t}}.
\]
We provide an explicit bound on  $\sigma_\star^2$ towards the end of the proof.

\item \textbf{Term \(\mathrm{(III)}\): Reward Residual. } Note that \[R(t) \leq \int_{T(N(t))}^{T(N(t)+1)} |h(X_s)|\,ds.\]
Further, $(\tau(n+1),Z_{n+1})$ is independent of $\sigma(\tau(1),\dots,\tau(n))$ for each $n\geq1$, $\tau(j)$'s are i.i.d., and the renewal measure $U([0,s])=\sum_{n\geq 0} \P(T(n)\leq s) =\E[N(s)]+1$, for $s\geq 0$. Hence,
    \begin{align*}
       \E\left[\int_{T(N(t))}^{T(N(t)+1)} |h(X_s)|\,ds\right]
& =\sum_{n\ge0}\E\!\left[1\{T(n)\le t<T(n)+\tau_{n+1}\}\int_{T(n)}^{T(n+1)} |h(X_s)|\,ds\right]
\\ & =\int_0^t G(t-s)\,dU(s) 
    \end{align*}
where $G(u):=\E[1\{\tau_1>u\}\int_{0}^{\tau_1} |h(X_s)|\,ds]$. Using $U(s)\le 1+s/m_1$ 
and integration by parts yields
\begin{align*}
\E\left[\int_{T(N(t))}^{T(N(t)+1)} |h(X_s)|\,ds\right]
&\le G(0)+\frac{1}{m_1}\int_0^\infty G(u)\,du
\\ &=^{(i)} \E\left[\int_{0}^{\tau_1} |h(X_s)|\,ds\right]+\frac{1}{m_1}{\E\left[\tau_1\int_{0}^{\tau_1} |h(X_s)|\,ds\right]}
\\ &\leq^{(ii)} \ell_1+\frac{1}{m_1}{\sqrt{\E\left[(\tau_1)^2\right]\E\left[\left(\int_{0}^{\tau_1} |h(X_s)|\,ds\right)^2\right]}}
\\ &= \ell_1+\frac{1}{m_1}{\sqrt{m_2 \ell_2}},
\end{align*}
where (i) follows as  $\int_0^\infty G(u)\,du=\E[\tau_1 \int_{0}^{\tau_1} |h(X_s)|\,ds]$ by Fubini, and (ii) follows using Cauchy–Schwarz inequality.

\item \textbf{Term \(\mathrm{(IV)}\): Age Residual. } 
Fix $t\geq 0$. By the definition of $R_T(t)$,
\[R_T(t) = \sum_{n\geq 0} (t-T(n)) {\mathbf 1}\{T(n) \leq t < T(n+1)\}\]
Taking expectations and conditioning on $T(n)$ gives
\begin{align*}
    \E[R_T(t)] &= \sum_{n\geq 0} \E \left[ \E \left[ (t-T(n))\mathbf{1}\{t<T(n)+\tau(n+1)\} |T(n)\right] \mathbf{1}\{T(n)\leq t\} \right]
    \\ & =^{(i)}  \sum_{n\geq 0} \E \left[ \underbrace{\E \left[ (t-T(n)){\mathbf 1}\{\tau(n+1)>t-T(n)\}\right]}_{G(t-T(n))} \mathbf{1}\{T(n)\leq t\} \right]
    \\ &=^{(ii)} \int_0^t G_T(t-s)\,dU(s),
\end{align*}
where for (i) we used independence and identical distribution of $\tau_{n+1}$ and for (ii) set $ G_T(u):=\E[u\mathbf{1}\{\tau(1)>u\}]$. 
Further, using $U(s)\le 1+s/m_1$ 
and  integration by parts yields
\begin{align*}
\E[R_T(t)]&\le G_T(0)+\frac{1}{m_1}\int_0^\infty G_T(u)\,du
\\ &= m_1+\frac{\E[\tau(1)^2]}{2m_1}. 
= m_1+\frac{m_2}{2m_1}
\end{align*}

\item \textbf{Explicit bound for $\sigma_\star^2$.} 
\begin{align*}
   \sigma_\star^2 &= \text{Var}(\tilde Z_1) + 2 \text{Cov}(\tilde Z_1,\tilde Z_2)
    \\&\leq  \text{Var}(\tilde Z_1) + 2 \sqrt{\text{Var}(\tilde Z_1)\text{Var}(\tilde Z_2)} 
    \\ &=  3\text{Var}(\tilde Z_1)
    \\ &=  3\text{Var}\left( Z_1-\beta \tau_1\right)
    \\ &\leq 6 \left(\E[Z_1^2] + \beta^2 \E[\tau_1^2] \right)
   \leq 6 \left(\ell_2 + \beta^2 m_2 \right).
\end{align*}

\end{enumerate}
Putting the bounds for the four terms in \eqref{eq:two-terms} together we have \eqref{eq:explicit-second-moment}.\\

We next prove the first part (SLLN).  Note that 
\begin{align}
    \frac{A(t)}{t}
& =\frac{\int_0^{T(0)}h(s) ds}{t} +  \frac{\sum_{i=1}^{N(t)} Z_i}{t}+\frac{R(t)}{t}
\nonumber\\ & =  \underbrace{\frac{\int_0^{T(0)}h(s) ds}{t}}_{(i)} + \underbrace{\Big(\frac{\sum_{i=1}^{N(t)} Z_i}{N(t)}\Big)}_{(ii)}\cdot \underbrace{\Big(\frac{N(t)}{t}\Big)}_{(iii)}
\ +\ \underbrace{\frac{R(t)}{t}}_{(iv)}. \label{eq:At_decomp}
\end{align}

\begin{enumerate}
\item \textbf{Term (i):} As $\E\left[\int_0^{T(0)} h(s) ds \right]<\infty$, $\int_0^{T(0)} h(s) ds <\infty$ a.s.. Thus, \[\frac{\int_0^{T(0)}h(s) ds}{t} \rightarrow 0 \quad \mbox{ a.s. as $t\rightarrow\infty$.}\] 
    \item \textbf{Term (ii):} In a one–dependent sequence, the
subsequences $\{Z_{2k}\}_{k\ge1}$ and $\{Z_{2k-1}\}_{k\ge1}$ are independent
and i.i.d.\ (any two indices in the same subsequence differ by at least~2, hence
belong to independent $\sigma$–fields by the defining property). By SLLN, 
\[
\frac{1}{n}\sum_{k=1}^n Z_{2k}\to \E[Z_1],\qquad
\frac{1}{n}\sum_{k=1}^n Z_{2k-1}\to \E[Z_1]\qquad\text{a.s.}
\]
Thus, $\frac{1}{n}\sum_{j=1}^n Z_j\to \E[Z_1]$ a.s..

\item \textbf{Term (iii):} Since $\tau(j)$'s are
i.i.d.\ with $0<\E[\tau(1)]<\infty$ and $\E[T(0)]<\infty$, we have $T(n)/n\to m_1$ a.s. as $n\rightarrow\infty$, and
\[
\frac{N(t)}{t}\ \rightarrow\ \frac{1}{m_1}\qquad\text{a.s. $t\rightarrow\infty$.}
\]
To see this, note that $T(N(t))\le t< T(N(t)+1)$ implies
\[
\frac{T(N(t))}{N(t)}\ \le\ \frac{t}{N(t)}\ <\
\frac{T(N(t))}{N(t)}+\frac{\tau(N(t)+1)}{N(t)}.
\]
By the SLLN, $T(N(t))/N(t)\to m$ a.s. Moreover, since $\E[\tau(1)]<\infty$,
$\tau(n)/n\to0$ a.s., which implies, $\tau(N(t)+1)/N(t)\to0$ a.s..
Thus, $t/N(t)\to m_1$ and $N(t)/t\to 1/m_1$ a.s.
\item\textbf{Term (iv):}   
Since $T(N(t))\leq t$ and $|R(t)|\leq \int_{T(N(t))}^{T(N(t)+1)}|h(X(s))|ds$
\[
\frac{|R(t)|}{t}\ \le\ \frac{\int_{T(N(t))}^{T(N(t)+1)}|h(X(s))|ds}{T(N(t))}
=\frac{\int_{T(N(t))}^{T(N(t)+1)}|h(X(s))|ds}{N(t)}\frac{N(t)}{T(N(t))}
\rightarrow 0 \quad\text{a.s. as $t\rightarrow\infty$.}
\]
\end{enumerate} 
Putting the bounds for the four terms in \eqref{eq:At_decomp} together, we have $A(t)/t\to \E[Z_1]/\E[\tau(1)]=\beta$ a.s..
\end{proof}

\section{Alternative Conventional Bayesian Approach}
\label{sec:alt_bayesian}



\begin{algorithm}[H]
\caption{Bayesian Bootstrap (Alternate Procedure)} 
\label{alg:alt_bayesian}
\begin{flushleft}
\textbf{Input:} Prior $\mu$, stochastic model $P_{\pi,\theta}$, history $\mathcal{H}_n$, number of trajectories $J$, length of trajectory $N$, performance measure $f(\cdot)$
\begin{enumerate}
    \item \textbf{for} $j \leftarrow 1$ \textbf{to} $J$ \textbf{do}
    \begin{enumerate}
        \item[2.] \textbf{Initialization:} $\mathcal{H}^j =\mathcal{H}_n$
            \item[3.] Evaluate $\mu(\theta|\mathcal{H}^j, \pi) = \frac{ P_{\pi,\theta}(\mathcal{H}^j) \mu(\theta)}{\int P_{\pi,\theta}(\mathcal{H}^j) \mu(\theta) d\theta}$
            \item[4.] Sample $\theta\sim \mu(\cdot|\mathcal{H}^j, \pi)$  
            \item[5.] Sample $(T_{n:N}^{j},E_{n:N}^{j}) \sim P_{\pi,\theta}(\cdot|\mathcal{H}^j)$   
            \item[6.] Update $\mathcal{H}^j \leftarrow \mathcal{H}^j \cup \{T_{n:N}^{(j)}, E_{n:N}^{(j)}\}$
        \item[7.] Estimate $f(\mathcal{H}^j)$ 
    \end{enumerate}
    \item[8.] \textbf{end for}
    \item[9.] \textbf{Return:} $\{f({\mathcal{H}}^{1}),\dots, f({\mathcal{H}}^{J})\}$
\end{enumerate}
\end{flushleft}
\end{algorithm}


\section{Optimal losses}
\label{sec:optimal_losses}

\subsection{Optimal loss for M/M/1}

\begin{lemma} Consider an M/M/1 queue with arrival rate $\lambda$ and service rate $\nu$, and let the traffic intensity be $\rho = \frac{\lambda}{\nu}$. 
\begin{itemize} \item If $\rho<1$, then
\[\lim_{N\to \infty}  -\frac{1}{N}\sum_{i=1}^N  \log \left[ {\P}_{i}^{event}(E_i^{(j)}) \right] = - \frac{1}{2}\left[ \frac{\lambda}{\nu}  \log\left(\frac{\lambda}{\lambda+\nu}\right) + \log\left(\frac{\nu}{\lambda+\nu}\right)\right], \]
\[
\lim_{N\to \infty}  \frac{1}{N}\sum_{i=1}^N   \left[ \frac{1}{{\lambda}_{i}^{(j)}} - T_i^{(j)} \right]^2 =\frac{1}{2\lambda^2} \left[ \frac{\nu}{\lambda+\nu}\right].\]

\item If $\rho\geq 1$, then
 \[\lim_{N\to \infty}  -\frac{1}{N}\sum_{i=1}^N  \log \left[ {\P}_{i}^{event}(E_i^{(j)}) \right] =-\frac{\lambda}{\lambda+\nu}  \log \left(\frac{\lambda}{\lambda+\nu}\right)- \frac{\nu}{\lambda+\nu} \log\left(\frac{\nu}{\lambda+\nu}\right)\]
 \[\lim_{N\to \infty}  \frac{1}{N}\sum_{i=1}^N   \left[ \frac{1}{{\lambda}_{i}^{(j)}} - T_i^{(j)} \right]^2 = \frac{1}{(\lambda+\nu)^2}.\]
    \label{lm:m_m_1_lower_bound}
    \end{itemize}
\end{lemma}
\begin{proof} We analyze the event losses and even-time losses separately.

\textbf{Event Losses:} 
Consider a system where we observe $N$ events, with $N$ being sufficiently large. Under the condition $\rho < 1$, the system reaches a steady state where approximately half of the events correspond to arrivals and half to departures, yielding $N/2$ events of each type.
For optimal prediction when the queue is non-empty, the oracle predictor assigns probabilities based on the relative rates: departures occur with probability $\frac{\nu}{\lambda+\nu}$ and arrivals with probability $\frac{\lambda}{\lambda+\nu}$. However, departures can only occur when the queue contains at least one customer, which introduces a dependency on the queue state.

\textit{Departure Loss Analysis:} When the queue is non-empty and a departure occurs, the log-probability of the correct prediction is $\log\left(\frac{\nu}{\lambda+\nu}\right)$. Since we expect $N/2$ departures in total, the cumulative departure loss becomes:
\begin{equation}
    -\frac{N}{2} \log\left(\frac{\nu}{\lambda+\nu}\right)
\end{equation}

\textit{Arrival Loss Analysis:} The arrival loss calculation requires considering two distinct queue states. When the queue is empty (occurring with probability $\frac{\nu-\lambda}{\nu}$), the oracle predictor assigns probability 1 to arrivals, yielding a log-probability of $\log(1) = 0$. When the queue is non-empty (with probability $\frac{\lambda}{\nu}$) and an arrival occurs, the log-probability is $\log\left(\frac{\lambda}{\lambda+\nu}\right)$.
The total arrival loss is therefore:
\begin{equation*}
    -\left(\frac{N}{2} \cdot \frac{\nu-\lambda}{\nu} \cdot \log(1) + \frac{N}{2} \cdot \frac{\lambda}{\nu} \cdot \log\left(\frac{\lambda}{\lambda+\nu}\right)\right) = -\frac{N}{2} \cdot \frac{\lambda}{\nu} \cdot \log\left(\frac{\lambda}{\lambda+\nu}\right)
\end{equation*}

\textit{Combined Event Loss:} Combining both departure and arrival losses, the total loss for $N$ events is:
\begin{equation*}
    -\left(\frac{N}{2} \log\left(\frac{\nu}{\lambda+\nu}\right) + \frac{N}{2} \cdot \frac{\lambda}{\nu} \log\left(\frac{\lambda}{\lambda+\nu}\right)\right)
\end{equation*}

The optimal average loss per event :
\begin{equation*}
    -\frac{1}{2}\left(\frac{\lambda}{\nu} \log\left(\frac{\lambda}{\lambda+\nu}\right) + \log\left(\frac{\nu}{\lambda+\nu}\right)\right)
\end{equation*}

\textbf{Inter-event Time Losses:} 
The inter-event time distribution depends critically on the queue state. When the queue is non-empty, the next event (either arrival or departure) occurs according to an exponential distribution with rate $\lambda + \nu$, reflecting the competition between both processes. Conversely, when the queue is empty, only arrivals are possible, resulting in an exponential distribution with rate $\lambda$.

\textit{Departure Time Loss:} Since departures only occur when the queue is non-empty, all $N/2$ departure events follow the exponential distribution with rate $\lambda + \nu$. The corresponding loss is:
\begin{equation*}
    \frac{N}{2} \cdot \frac{1}{(\lambda+\nu)^2}
\end{equation*}

\textit{Arrival Time Loss:} Arrival events occur under both queue states. When the queue is empty (probability $1-\frac{\lambda}{\nu}$), arrivals follow rate $\lambda$. When non-empty (probability $\frac{\lambda}{\nu}$), they follow the combined rate $\lambda + \nu$. The total arrival loss becomes:
\begin{equation*}
    \frac{N}{2}\left(1-\frac{\lambda}{\nu}\right) \frac{1}{\lambda^2} + \frac{N}{2} \cdot \frac{\lambda}{\nu} \cdot \frac{1}{(\lambda+\nu)^2}
\end{equation*}

\textit{Average Inter-event Time Loss:} Combining both components and normalizing by $N$, the optimal average inter-event time loss is:
\begin{equation*}
    \left(\frac{1}{2} + \frac{\lambda}{2\nu}\right) \frac{1}{(\lambda+\nu)^2} + \frac{1}{2}\left(1-\frac{\lambda}{\nu}\right) \frac{1}{\lambda^2} = \frac{1}{2\nu(\lambda+\nu)} + \frac{\nu-\lambda}{2\nu\lambda^2}
\end{equation*}

The analyses for $\rho>1$ follow the same lines of argument. 
\end{proof}

\subsection{Optimal loss for M/M/N}

For an M/M/N queue with k customer classes, we evaluate the optimal loss on the test set $\mathcal{D}^{test}$ by exploiting the knowledge of underlying queueing system. After observing all events up to and including the $(t-1)^{th}$ event, the system’s state is fully determined by the initial state and these observations.  In particular, we can identify
\begin{enumerate}
    \item the number of customers of each type in system $(n_1,\ldots,n_k)$;
    \item for each server $j\in \{1,\ldots,N\}$;
    \begin{itemize}
        \item  whether it is busy—denoted by the indicator ($1_j$); and
        \item if busy, the class of the customer it is serving.
    \end{itemize} 
\end{enumerate}

\noindent Let $\lambda_i$ and $\nu_i$ be the arrival rate  and service rate of customer $i$. Using above information, we can evaluate the following probabilities: 

\paragraph{Event-type probabilities}

 \begin{enumerate}
     \item \textbf{Arrival:} The probability that the next event is an arrival is 
\[ P(\text{arrival}) = \frac{\sum_i \lambda_i}{\sum_i \lambda_i + \sum_{j=1}^{N} 1_j \cdot \nu_{\text{customer-type at j-th server}}}\]
where $1_j \in \{0,1\}$ indicates whether server 
$j$ is busy and $\nu_{\text{customer-type at j-th server}}$
  is the service-completion rate of the customer currently being served at that server $j$.
Conditioned on an arrival, the probability that the arriving customer is of class $i$  is 
 \[P(\text{class = i $|$ arrival}) = \frac{\lambda_i}{\sum_j \lambda_j}\]

\item \textbf{Departure:} The probability that the next event is the departure from server $j$  is 
\[P(\text{departure from server $j$}) =\frac{1_j \cdot \nu_{\text{customer-type at j-th server}}}{\sum_i \lambda_i + \sum_{m=1}^{N} 1_m \cdot\nu_{\text{customer-type at m-th server}}}\]

In this case the customer class is already known:

\[ \text{Class $|$ departure from server $j$ $=$ Customer class at j-th server }\]

 \end{enumerate}

\paragraph{Inter-event time distribution}
Given the current state, the time until the next event is exponentially distributed with rate

\[r_t = \sum_i \lambda_i + \sum_{j=1}^{N} 1_j \cdot\nu_{\text{customer-type at j-th server}}\]

\paragraph{Optimal loss:}
For the 
$t^{th}$ event we can estimate optimal loss for from the above probabilities as follows
\begin{align*}
\text{Loss}_t &= -\log[Pr(\text{event}_t)] - \log[Pr(\text{customer-type}_t|\text{event}_t)] + \left(\frac{1}{r_t} - (\text{inter-event-time})_t\right)^2
\end{align*}

\paragraph{Average Optimal Loss}
Overall average loss over $T$ events is:
\[
\text{Average-Loss} = \frac{1}{T}\sum_{t=1}^T \text{Loss}_t
\]

Similarly, we can also estimate the average optimal losses for event type, customer type and inter-event times.






\section{Experimental details}
\label{sec:experimental-details}

\subsection{Transformer Architectural Details}
\label{sec:experimental-details-core-modeling-architecture}
Throughout all experiments, we use a decoder-only Transformer architecture. While we adapt the embedding layers to accommodate different task-specific requirements (e.g., event-type embeddings, temporal embeddings, customer-type embeddings), the core transformer architecture remains the same throughout our study:
\begin{itemize}
    \item Input embedding dimension: 64
    \item Hidden dimension: 512
    \item Number of attention heads: 4
    \item Number of transformer layers: 8
\end{itemize}


\subsubsection{Positional embeddings}
We investigate two distinct approaches for incorporating positional information into our model architecture:
\begin{enumerate}
\item \textbf{Sinusoidal Position Encoding:} In this method, our positional embedding is the sum of two complementary \textit{sinusoidal position embeddings:} (1) a linearly increasing component that maintains identical values across three consecutive positions (0,0,0,1,1,1,...), where the period of three corresponds to the event-type, customer-type, and inter-event time components that collectively represent a single event, and (2) a cyclical component that repeats every three positions (0,1,2,0,1,2,...), which enables the transformer to distinguish between event-type, customer-type, and inter-event time within each event. The period length of three can be adjusted according to the data structure—the design principle is that one component should identify all components of each event collectively, while the other should enable differentiation within the components of each event. Both components utilize sinusoidal functions operating at different frequencies to encode positional information into vectors with dimensions matching the event embeddings.

\item \textbf{Learnable Periodic Encoding:} Under this approach, our positional embedding is sum of  two \textit{trainable embedding layers}. As earlier, the first embedding layer captures general positional information across sequences of arbitrary length, and other layer captures the periodic patterns that recur at a configurable interval (set to three positions in our experimental setup).  The embeddings from both layers are combined additively to produce the final positional encoding, enabling the model to simultaneously learn long-range positional dependencies and local periodic structures.
\end{enumerate}
In our implementation, we adopted the sinusoidal position encoding approach due to its deterministic handling of arbitrary sequence lengths without requiring additional learnable parameters. This design choice demonstrates robust performance across our experimental evaluations.



\subsubsection{Embedding layers}
To maintain architectural consistency, every embedding layer—including those for event types and time—projects its inputs into a common vector dimensionality. We now describe each embedding in greater detail:

\paragraph{Event type embedding:}
For event type representation, we employ a learnable embedding layer that maps each event type to a vector whose dimensionality corresponds to the model dimension. To accommodate sequences of varying length, we introduce padding tokens assigned a unique index (-100) and a dedicated embedding vector. This architecture facilitates effective event type processing while appropriately handling padding in variable-length sequences.

\paragraph{Time Embedding:}
We investigate two distinct approaches for time embeddings.
\begin{itemize}
    \item 
\textbf{Continuous Time Embedding.} Our first approach treats time as a continuous variable, employing Time2Vec embeddings to capture temporal information effectively. This method preserves the continuous nature of temporal data while enabling the model to learn meaningful representations directly from raw time values.

\item
\textbf{Discretized Time Embedding.} Following \citep{bellemare2017,MullerHoArGrHu22}, our second approach discretizes inter-event time distributions using Riemann distributions, transforming time prediction into a classification task. We partition the time domain into $n$ bins, where the first $(n-1)$ bins maintain equal width $w$, covering the range $[0, (n-1)w]$, with time assumed uniformly distributed within each bin. The final bin captures all times $t \geq (n-1)w$ using a half-normal distribution.

For the discretized approach, we carefully tune parameters $n$ and $w$ to balance adequate coverage of observed time distributions with computational efficiency. To determine bin resolution and coverage, we rely on the empirical distribution of inter-event times in the training data.  Parameter $w$ is selected to maintain sufficient resolution within each bin while keeping $n$ computationally manageable. We ensure that $(n-1)w$ covers the full range of observed times, with the final bin capturing any remaining distribution tail.
\end{itemize}

The selection between these embedding strategies depends on task-specific requirements and prior knowledge about the parametric family of underlying time distributions. Continuous embedding offers greater computational efficiency and suits applications with known parametric families of inter-event time distributions, while discretized embedding provides greater generality for handling arbitrary inter-event time distributions in complex scenarios, albeit with increased computational overhead. 


\subsubsection{Modifications for multiple customer types}
To handle multiple customer types in a general setting, we explored various prediction approaches. Initially, we attempted to predict event types and customer types independently in parallel. However, this approach proved suboptimal, as it failed to capture the inherent dependencies between events and the customers involved.

To capture the dependencies between event type and customer type, two approaches are possible: (1) assigning a unique index to each event–customer type combination, or (2) using a conditional modeling approach where the event type is predicted first, followed by the customer type conditioned on the event.

The first approach leads to a total event space equal to the product of the number of event types and customer types, while the second yields a space equal to their sum. We adopt the conditional approach to reduce the prediction space size. Specifically, we decompose each prediction step into three sequential components:
\begin{enumerate}
\item \textbf{Event type prediction}: First, we predict the next event type.
\item \textbf{Customer type prediction}: We then predict the customer type  conditioned on the predicted event.
\item \textbf{Inter-event time prediction}: Finally, we predict the time until the event occurs, conditioned on both the event and customer type predictions.
\end{enumerate}

This conditional approach enables the model to better capture the dependencies within the stochastic dynamics. The type of event influences which customers are likely to be involved, and both the event and customer type jointly affect the inter-event time.

We implement this approach by treating these components as separate tokens within our transformer architecture, where each prediction is conditioned on all previous tokens, including predictions made in earlier steps of the same decomposition. We found that the conditional modeling approach yielded superior results compared to predicting customer type and event type independently in parallel.

Customer type embeddings are generated using the same methodology as event type embeddings, employing a learnable embedding layer that maps each customer type to a vector whose dimensionality corresponds to the model dimension. To accommodate variable-length sequences, we introduce padding tokens with a unique index (-100) and a dedicated embedding vector.

\subsubsection{Including initial states of the system in model architectures}

The initial state of the service system can be easily incorporated within our framework. We explain how to incorporate initial state in a transformer model using an M/M/N queue with 
K customer types. An initial state in this case consists of the following components:

\begin{enumerate}
    \item for each server $j\in \{1,\ldots,N\}$;
    \begin{itemize}
        \item  whether it is busy or idle; and
        \item if busy, the class of the customer it is serving.
    \end{itemize} 
    \item For the queue: the sequential order of waiting customers, including both the number of customers of each type in the system
 $(n_1,\ldots,n_K)$ and their sequential ordering.
\end{enumerate}

We employ a structured encoding scheme for servers using discrete status values: 0 for idle servers, and values 1 through 
K for servers actively serving customers of the respective types. Although theoretically the queue can be of infinite length, we model systems with a maximum queue length of
Q, which is realistic in practice. The queue is represented as a
Q-length vector with each entry recording 0 for empty positions and values 1 through
K for positions occupied by customers of the respective types. This state information is subsequently projected to the model's embedding dimension through a linear transformation layer, enabling simulations under diverse initial conditions while maintaining dimensional consistency throughout the neural network architecture.

When dealing with more general queues, we must also include the remaining service times of customers being served and the time since last arrival to fully capture the initial state. The initial state of more complex service systems can be similarly handled by embedding all relevant information for each server and queue position.

Including initial state in the modeling facilitates reusing the same transformer repeatedly, thereby increasing the effective prediction length. For example, if the transformer prediction length is N and we start with state $S_1$, we can predict next N events $E_{1:N}$ and inter-event times $T_{1:N}$, obtain state $S_{N+1}$ from $S_1$ and $E_{1:N}$, then reuse the transformer to obtain next $N$ events $E_{N+1:2N}$, and continue this process iteratively.




\subsubsection{Architecture incorporating Control Policy as an input}
To support policy evaluation and counterfactual analysis (Section \ref{sec:additional_experiments-counterfactuals}), we encode the current policy as an additional input token that is prepended to the Transformer at initialization. A policy can specify, for example, the total number of servers 
$N$ and the day-to-day fluctuation coefficient 
$c$ of the arrival rate.

Each policy variable is projected into the model dimension through a learnable embedding layer. Continuous variables are processed through a learnable neural network consisting of two linear layers with ReLU activation, while categorical variables are mapped to distinct trainable embeddings. This design allows the model to (i) quantify the effect of alternative staffing policies, (ii) compare system performance under different configurations, and (iii) analyze the transient dynamics of the queueing system under different policies.


\subsection{Transformer Training details}

We train the Transformer with the loss defined in Eq.~\eqref{eq:loss_quote_everywhere} using the Adam optimizer (learning rate $5 \times 10^{-4}$, weight decay $1 \times 10^{-5}$).  Training is performed with a batch size of 32, and we clip gradients to a global norm of 1 for stability. The learning rate follows a 30-epoch linear warm-up from $1 \times 10^{-7}$ to $5 \times 10^{-4}$, after which it decays cosinusoidally $5 \times 10^{-6}$ for the remainder of training. 

The model comprises eight Transformer layers, each with four attention heads; dropout is disabled in both the positional-encoding and Transformer blocks. 

\subsection{Data generation and discrete event simulator}

We employ the Ciw library \citep{palmer2017ciw} to conduct our discrete event simulations. This library comprehensively tracks and records all individual events occurring at different nodes throughout the simulation period, including customer arrivals, departures, and service completions.

To illustrate our approach, consider a multi-class M/M/5 queueing system comprising 5 distinct customer types. Each customer type exhibits different arrival rates ranging from 1.0 to 5.0 customers per unit time, while maintaining a uniform service rate of 1.0 customers per unit time across all types. The system operates with 5 parallel servers under a First-In-First-Out (FIFO) discipline without priority differentiation. Figure \ref{fig:code_discrete_event_simulator}  demonstrates the code for executing such a simulation.




\begin{figure}[h!]
    \centering
    \includegraphics[width=0.65\linewidth]{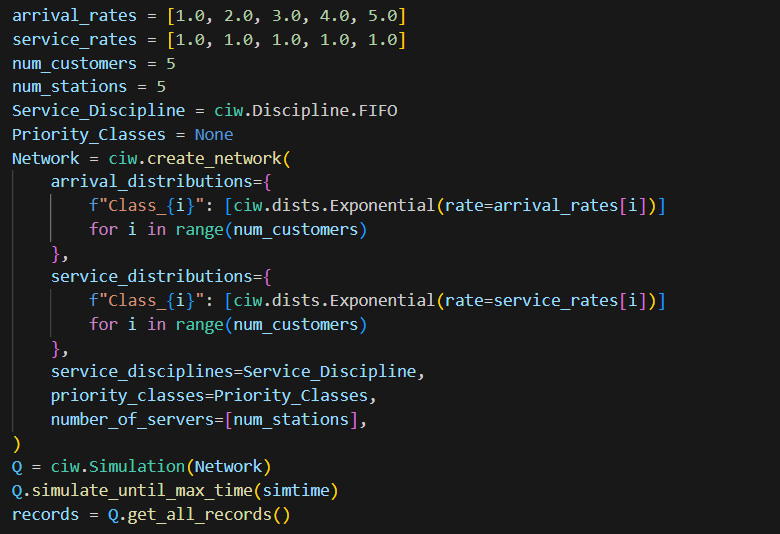}
    \caption{M/M/5 discrete event simulator implemented using Ciw library}
    \label{fig:code_discrete_event_simulator}
\end{figure}

The simulation library provides comprehensive tracking capabilities that monitor system state and events throughout the queuing process. This tracking framework encompasses individual customer journeys, recording precise arrival times, service initiation and completion timestamps, and customer class classifications. Additionally, the system continuously monitors queue lengths at each service point, tracks server utilization patterns and state transitions, and implements sophisticated priority class handling with appropriate service discipline protocols. The library generates detailed event logs that facilitate thorough post-simulation analysis and validation.

These comprehensive event logs serve as the foundation for our training dataset generation and oracle benchmark establishment across the respective queuing models. The recorded data encompasses the complete spectrum of system dynamics, including customer arrivals, service completions, real-time queue lengths at each network node and temporal checkpoint, server operational states, and additional relevant performance metrics.

The simulation records undergo systematic processing to create structured event sequences that effectively capture the temporal dynamics and state transitions inherent in the queuing system. This methodical data transformation ensures that the resulting datasets accurately represent the underlying system behavior while maintaining the chronological integrity of events. The processed data subsequently forms the empirical foundation for our training datasets and enables the establishment of ground truth benchmarks essential for rigorous model evaluation and validation procedures.

\subsection{Extracting performance metrics from events table}
\label{sec:extracting_events_table}
Consider a single-server queue with a FIFO discipline as an example. Let $A_j$ denote the event index of the $j$th arrival event. The time between the $j$th arrival and the $(j+1)$th arrival can then be calculated as
$\sum_{i=A_j+1}^{A_{j+1}} T_{i}$. For service times, 
let $D_j$ denote the even index of the $j^{th}$ departure event.  Further define $E_j = \max\{i\geq D_j: \text{queue is 0 after the $i$th event}\}$.  
Then, the $j^{th}$ service time can be estimated as $  \sum_{i=\max\{D_j,E_j\}+1}^{D_{j+1}} T_{i}$. For waiting times, assuming the initial queue length is $q_{0}$, the $j^{th}$ waiting time can be estimated as $\sum_{i=A_j}^{S_{q_{0}+j}-1} T_{i+1}$.







\subsection{Experimental Details for Empirical Validation (Sections \ref{sec:part_1_approach_validation},  \ref{sec:Approach Validation})}

\subsubsection{M/M/1 with single customer type}
\label{sec:Additional experiments:M/M/1}
We validate our method on a standard 
M/M/1 queue with a single customer class served under a first-in, first-out (FIFO) discipline. The Transformer is trained on 10,000 simulated trajectories, each containing 400 events. After training, we generate 2,000 trajectories from the trained transformer and analyze the resulting performance-metric distributions, as shown in Figure \ref{fig:M-M-1-distributions}.

\subsubsection{M/M/5 with 5 customer types}
\label{sec:Additional experiments:M/M/5}
We test our method on an M/M/5 queue that serves five customer classes under a first-in, first-out (FIFO) discipline.  The Transformer is trained on 40,000 simulated trajectories, each containing 800 events. To assess its performance, we then generate 2,000 trajectories from the trained transformer and examine the resulting performance metrics for every class. Figure \ref{fig:M-M-5-distributions-individual-customers} demonstrates that our transformer faithfully replicates the discrete-event simulator’s inter-arrival-time, service-time, and waiting-time distributions for every customer class.

\begin{figure}[h!]
  \centering
    \centering{\textbf{Customer 1}}\\
  \begin{minipage}[b]{0.325\textwidth}   
    \includegraphics[width=\textwidth, height=3cm]{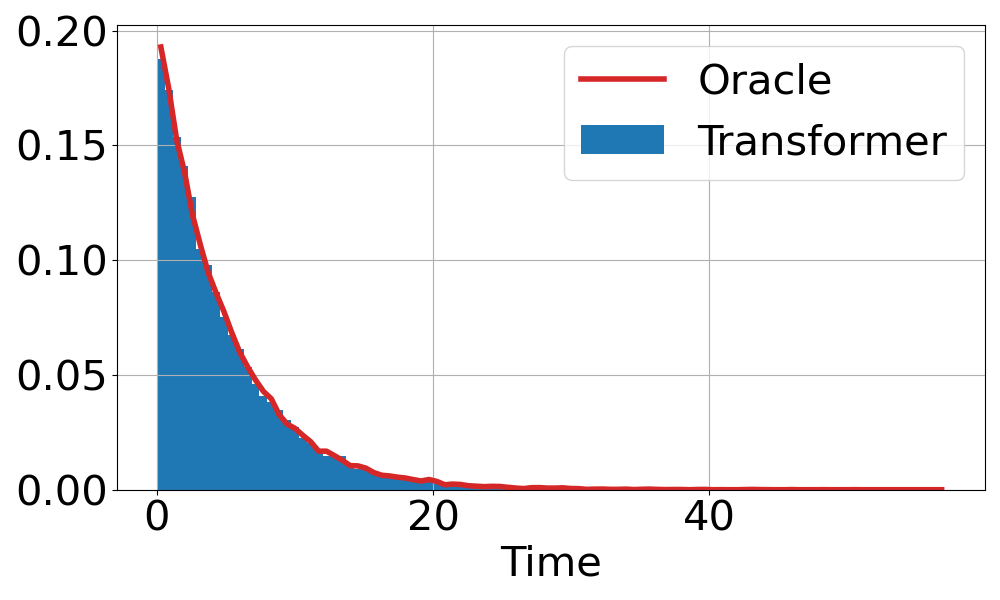}
  \end{minipage}
  \hfill
  \begin{minipage}[b]{0.325\textwidth}   
    \includegraphics[width=\textwidth, height=3cm]{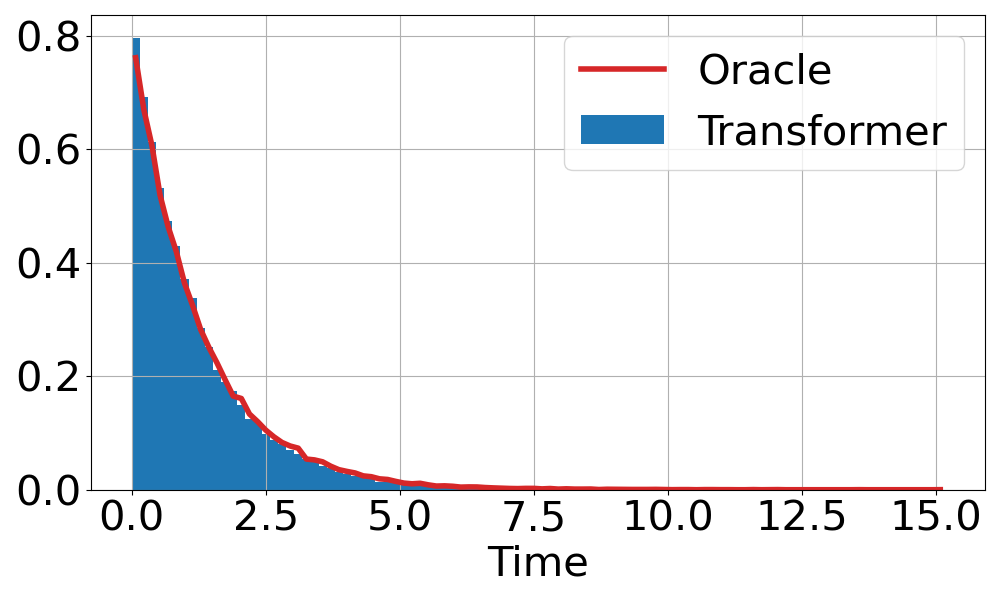}
  \end{minipage}
  \hfill
  \begin{minipage}[b]{0.325\textwidth}  
    \includegraphics[width=\textwidth, height=3cm]{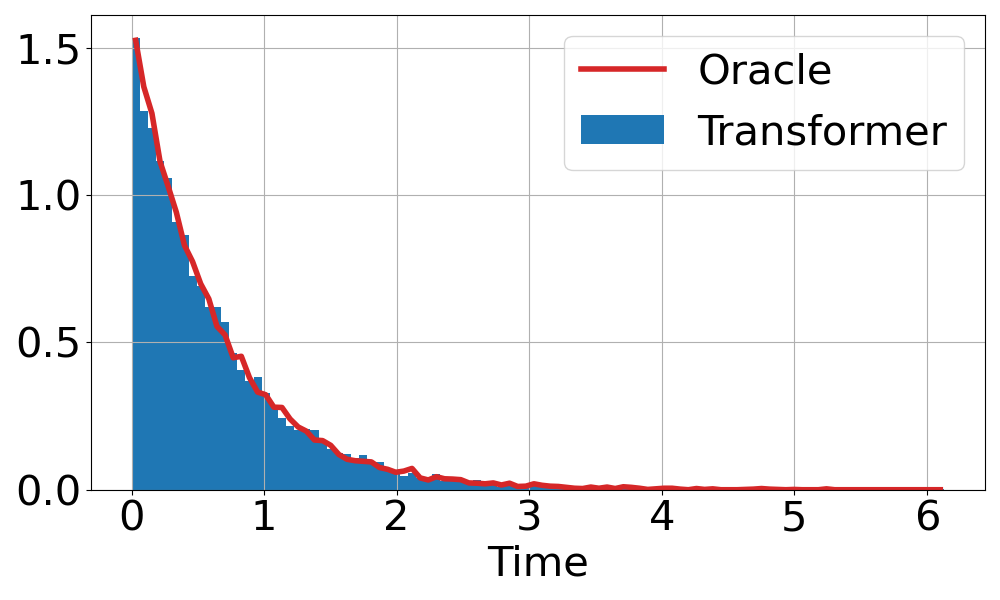}
  \end{minipage}
\hfill
\centering{\textbf{Customer 2}}\\
    \begin{minipage}[b]{0.325\textwidth}   
      \includegraphics[width=\textwidth, height=3cm]{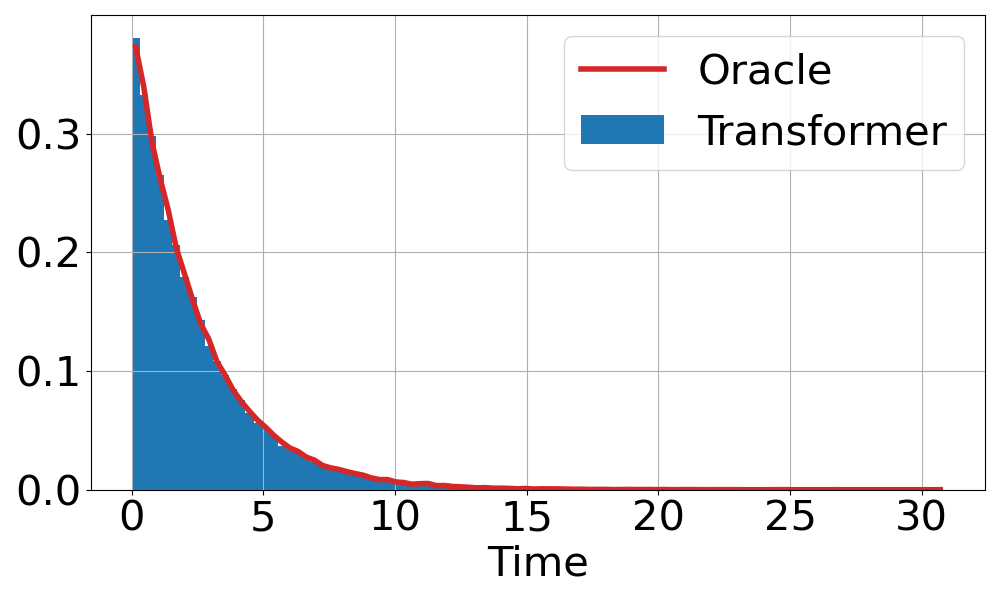}
    \end{minipage}
    \hfill
    \begin{minipage}[b]{0.325\textwidth}   
      \includegraphics[width=\textwidth, height=3cm]{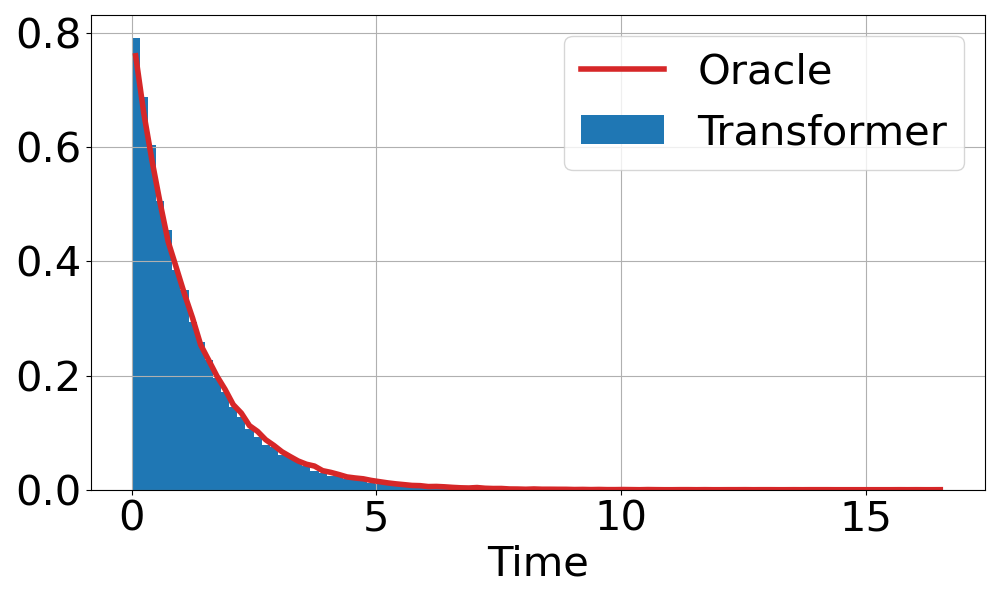}
    \end{minipage}
    \hfill
    \begin{minipage}[b]{0.325\textwidth}   
      \includegraphics[width=\textwidth, height=3cm]{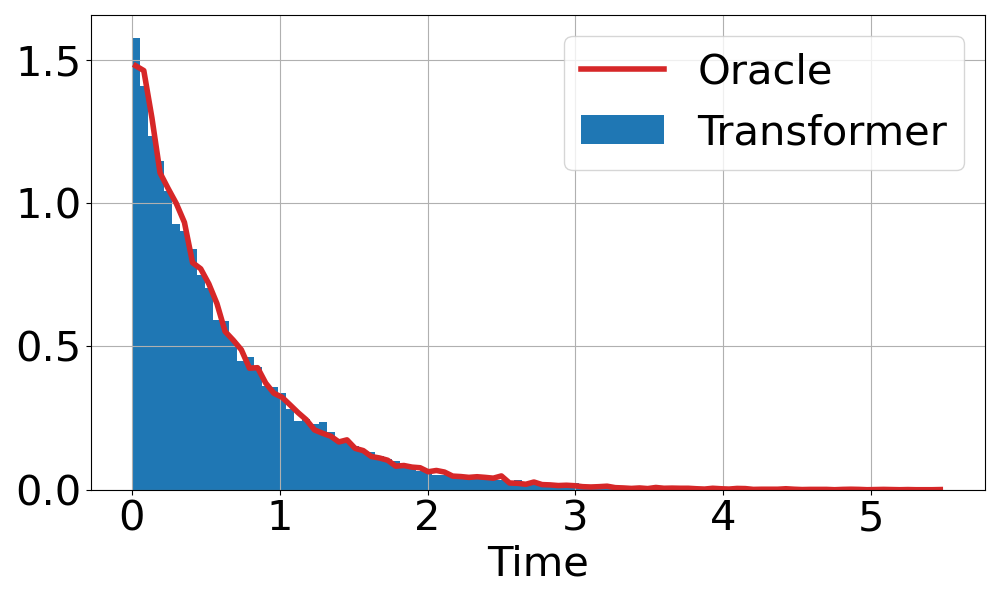}
    \end{minipage}
\hfill
\centering{\textbf{Customer 3}}\\
    \begin{minipage}[b]{0.325\textwidth}   
      \includegraphics[width=\textwidth, height=3cm]{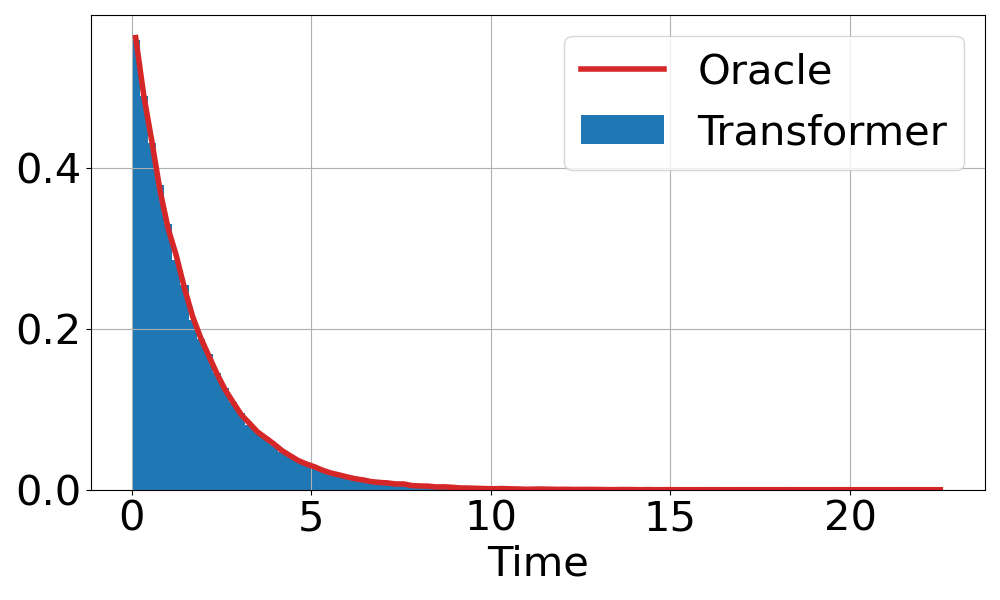}
    \end{minipage}
    \hfill
    \begin{minipage}[b]{0.325\textwidth}   
      \includegraphics[width=\textwidth, height=3cm]{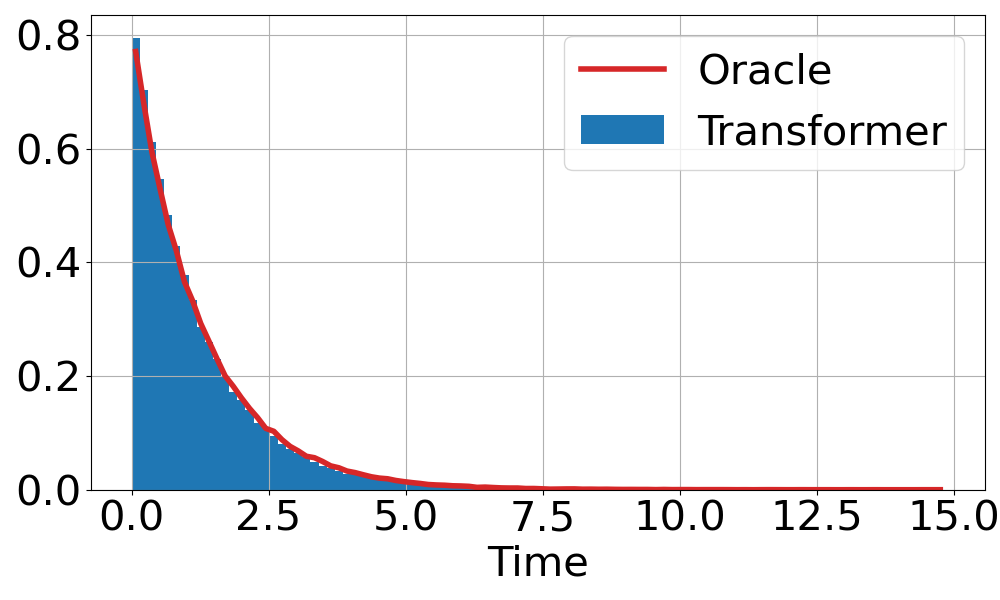}
    \end{minipage}
    \hfill
    \begin{minipage}[b]{0.325\textwidth}   
      \includegraphics[width=\textwidth, height=3cm]{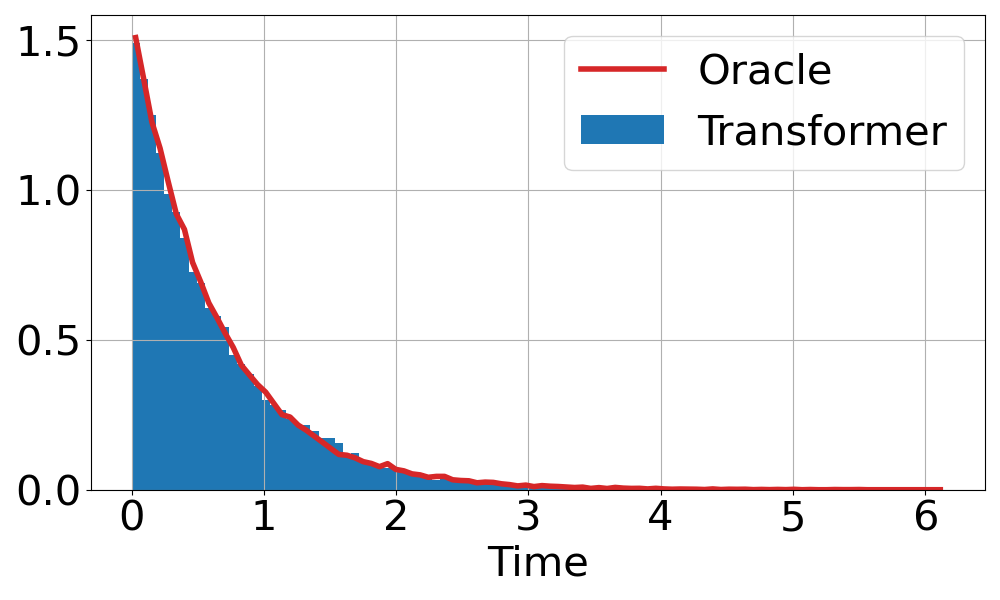}
    \end{minipage}
\hfill
      \centering{\textbf{Customer 4}}\\
    \begin{minipage}[b]{0.325\textwidth}   
      \includegraphics[width=\textwidth, height=3cm]{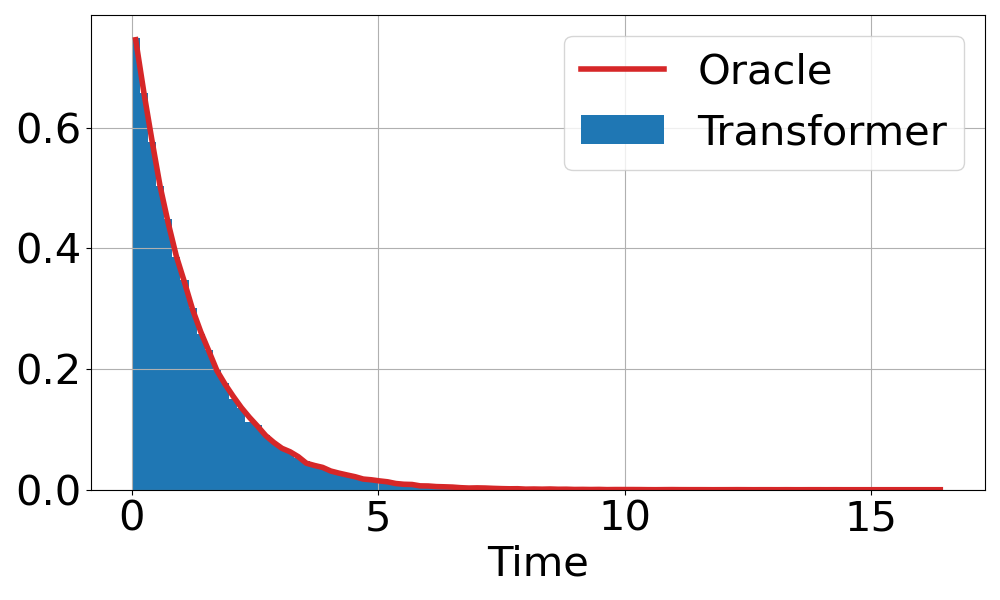}
    \end{minipage}
    \hfill
    \begin{minipage}[b]{0.325\textwidth}   
      \includegraphics[width=\textwidth, height=3cm]{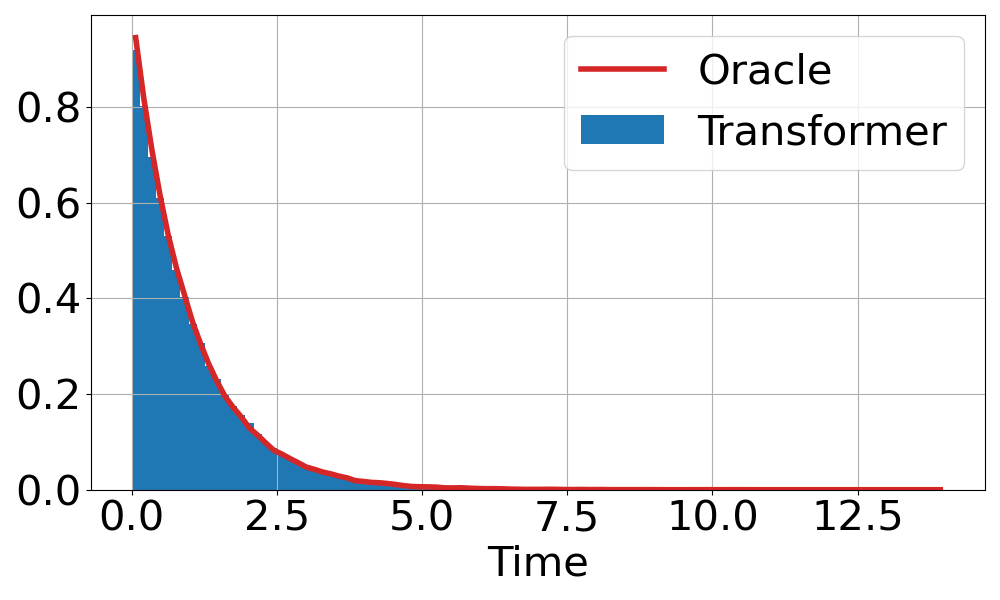}
    \end{minipage}
    \hfill
    \begin{minipage}[b]{0.325\textwidth}   
      \includegraphics[width=\textwidth, height=3cm]{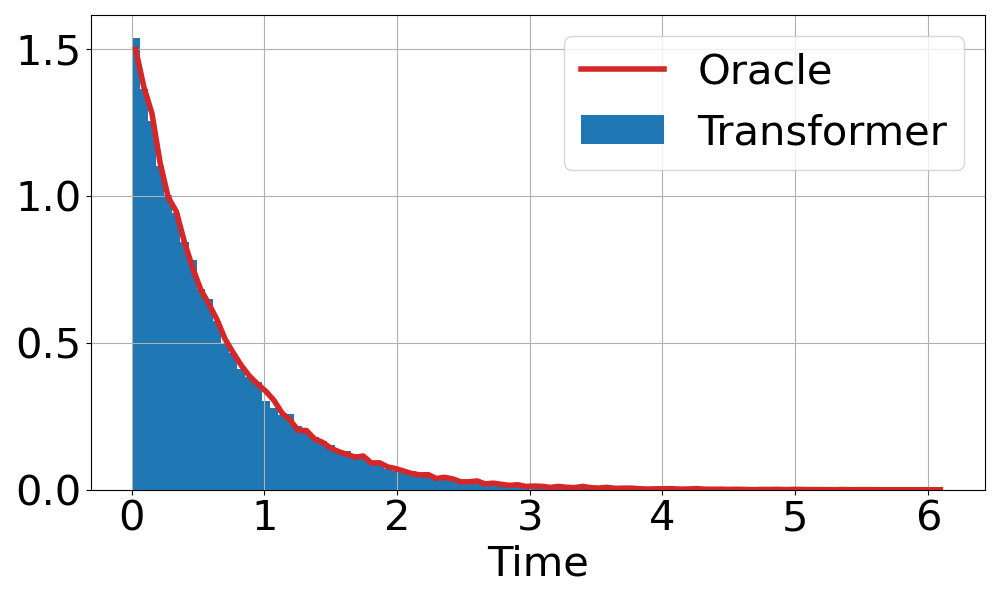}
    \end{minipage}
\hfill
      \centering{\textbf{Customer 5}}\\
    \begin{minipage}[b]{0.325\textwidth}   
      \includegraphics[width=\textwidth, height=3cm]{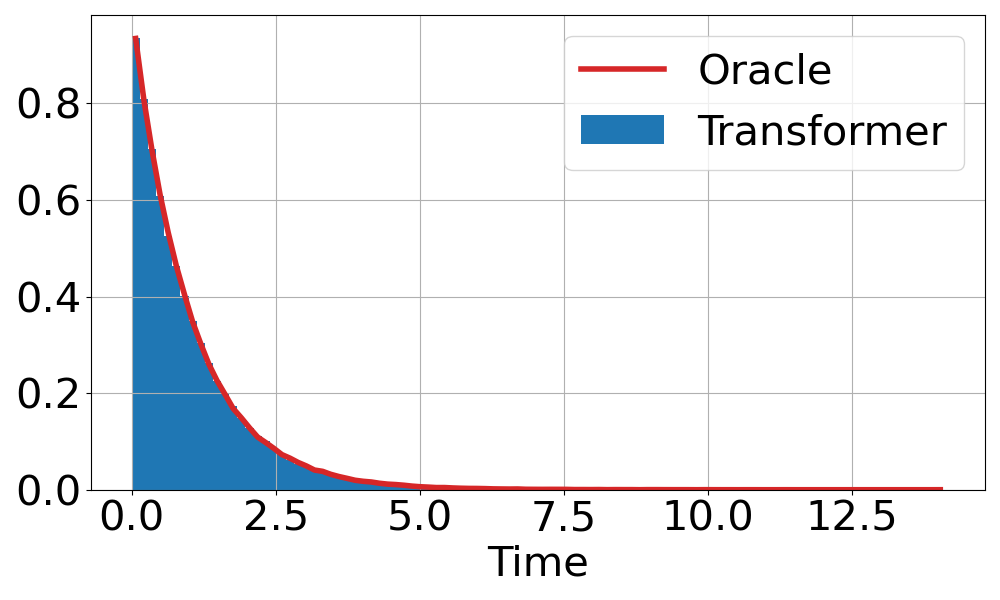}
      \centering{\textbf{Inter-arrival time}}
    \end{minipage}
    \hfill
    \begin{minipage}[b]{0.325\textwidth}   
      \includegraphics[width=\textwidth, height=3cm]{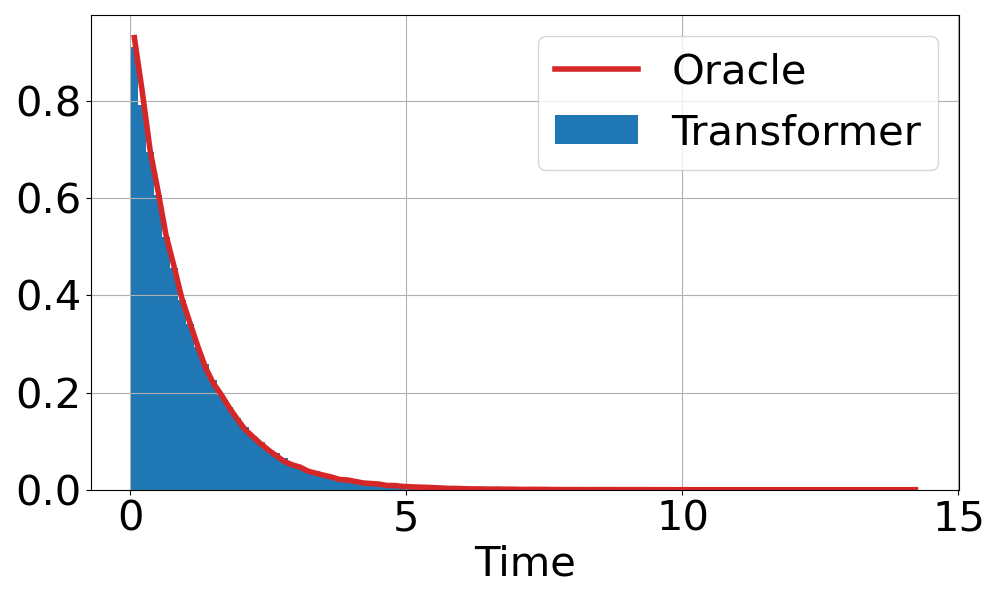}
        \centering{\textbf{Service time}}
    \end{minipage}
    \hfill
    \begin{minipage}[b]{0.325\textwidth}   
    \includegraphics[width=\textwidth, height=3cm]{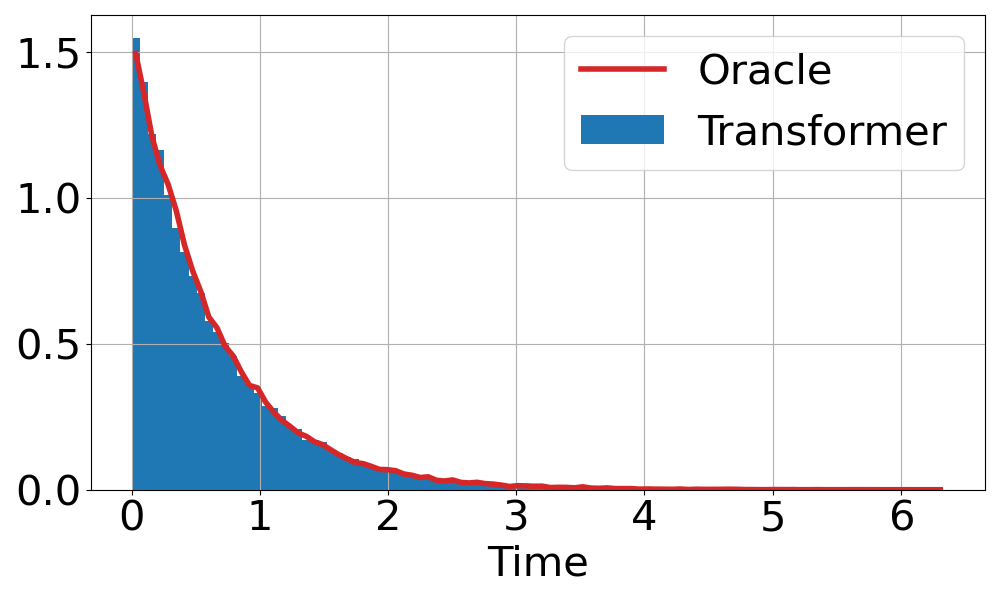}
    \centering{\textbf{Positive waiting time}}
    \end{minipage}
    \caption{\textbf{Multi-customer type M/M/5 system:} Comparison of performance measure distributions predicted by the sequence model (transformer) with the true distributions derived from the underlying M/M/5 queuing network. }
    \label{fig:M-M-5-distributions-individual-customers}
  \end{figure}

\subsubsection{General distributions - $G/G/1$ queue:} 
To demonstrate the model’s ability to handle general (non-exponential) timing patterns, we test it on a G/G/1 queue in which both inter-arrival and service times are uniformly distributed. Inter-arrival times are drawn from $U[3,6]$ and service times from $U[2,4]$.  The Transformer is trained on 
10,000 simulated trajectories from this system. We discretize time with 0.1-second bins—fine enough to capture the uniform densities while remaining computationally efficient. Performance is then evaluated on 
500 trajectories generated from the trained transformer, from which we estimate the corresponding performance-metric distributions.

\subsubsection{Non stationary arrival patterns} 
\label{Appendix:Non-Stationary-Arrival}
To test how well the model handles non-stationary arrivals, we consider an $M_t/M/N$ queue with $N=11$ servers, each with a service rate of 1.6 customers per hour.
The arrival rate $\lambda(t)$ changes every hour according to  $[8,8,8,8,8,14,15,16,17,18,19,18,17,16,15,11,11]$ customers per hour.

A Transformer with a maximum sequence length of 400 events is trained on 10,000 trajectories generated from this process. Timestamps are recorded in minutes and discretised into uniform bins of width  $0.01$ min, providing fine temporal resolution without excessive computational cost. Model performance is evaluated on 500 trajectories sampled from the trained network, from which we estimate the corresponding performance-metric distributions.

\subsubsection{Call-center}

We benchmark our model on a semi-synthetic data set derived from the call-centre operations of an Israeli bank, documented by \citet{seelab-report}. We recreate the  original environment in a simulator while preserving the empirical distributions observed between period January and December~1999.

\paragraph{Simulator architecture.}
The system comprises two sequential service stages:
\begin{enumerate}
  \item \textbf{Voice-Response Unit (VRU).}  
        All callers first enter an automated VRU with effectively infinite capacity, which we approximate by modelling \(1\,000\) parallel servers.
  \item \textbf{Human-agent stage.}  
        Upon VRU completion, callers join a single queue served by six human agents—the historical average of \(6.08\) simultaneous agents (the original system allowed up to fifteen).
\end{enumerate}
When every agent is occupied, incoming callers join a queue—about 65 \% of arrivals experience this wait. While queued, roughly 15 \% abandon the call before reaching an agent.

\paragraph{Arrival, service, and patience processes.}
\begin{itemize}
  \item \textbf{Arrivals.} Inter-arrival times follow an exponential distribution. Callers are partitioned into six customer types with empirical proportions \([50\%, 20\%, 10\%, 5\%, 5\%, 10\%]\).
  \item \textbf{Service times.} For both stages we sample service times from the empirical histograms reported in \citet{seelab-report}.
  \item \textbf{Patience (abandonment).} Patience times are modelled as exponential. Parameters are calibrated via a method-of-moments fit to the Kaplan–Meier means and standard deviations of \citet{seelab-report}, yielding mean patience times (seconds) \([521,\,644,\,528,\,703,\,647,\,491]\) for customer types~1–6.
\end{itemize}

\paragraph{Queue discipline.}
A non-pre-emptive priority rule is enforced: types~2 and~5 have high priority, whereas types~1, 3, 4, and~6 are low priority. Within each priority class, service is first-in-first-out (FIFO).


Figures~\ref{fig:representative-average-arrival-time-distributions-see-lab}–\ref{fig:representative-average-waiting-time-distributions-see-lab-appendix}
compare the Transformer-generated inter-arrival, service-time, and average waiting-time distributions with those produced by the discrete-event simulator for each customer type, demonstrating the fidelity of the learned generative process.

\begin{figure}[h!]
    \centering
    \begin{minipage}[b]{0.325\textwidth} 
    \centering
    \textbf{Customer 1}
    \includegraphics[width=\textwidth, height=3cm]{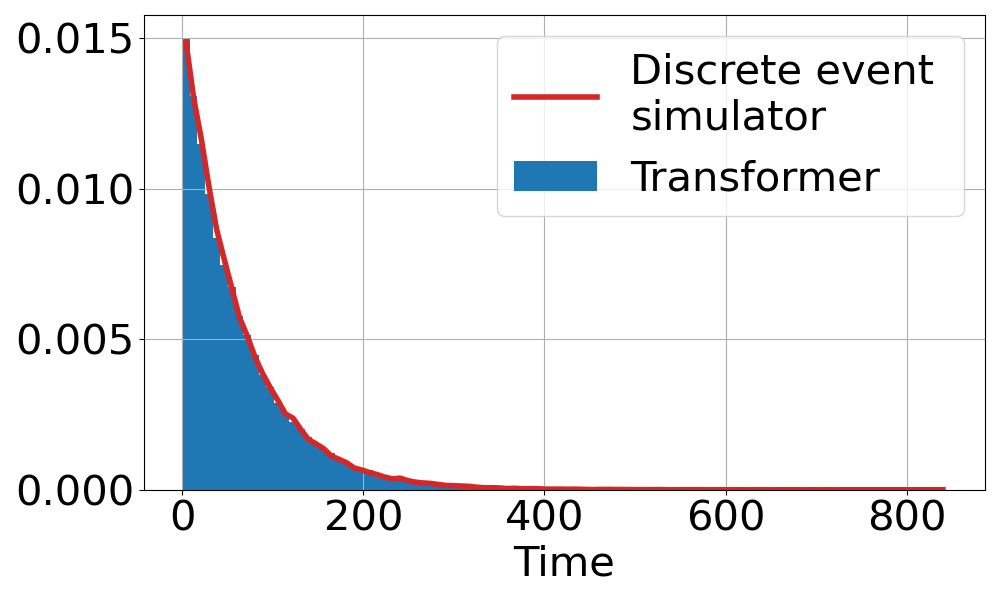}
    \end{minipage}
    \hfill
    \begin{minipage}[b]{0.325\textwidth}  
    \centering
    \textbf{Customer 2}
    \includegraphics[width=\textwidth, height=3cm]{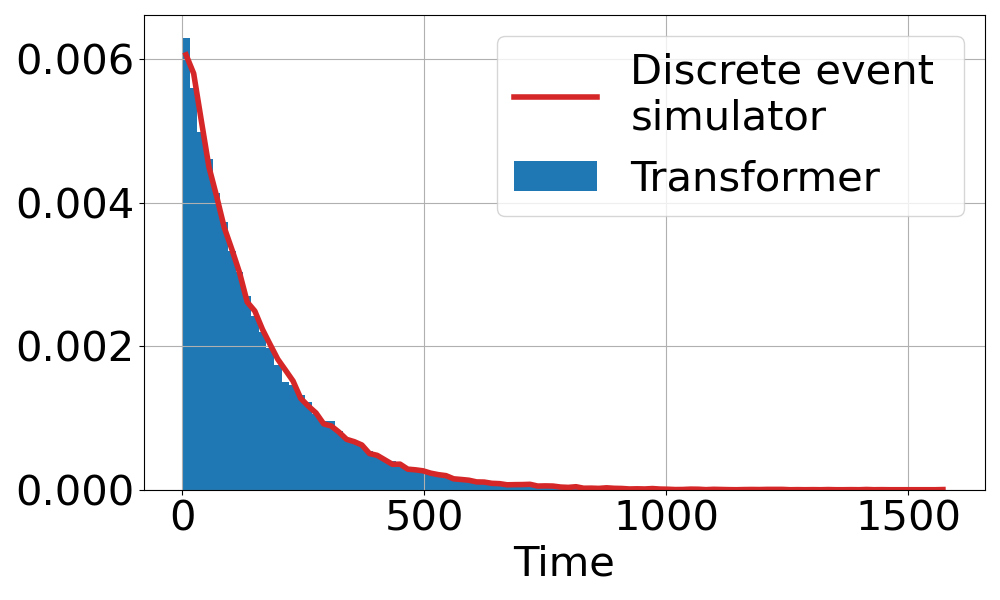}
    \end{minipage}
    \hfill
    \begin{minipage}[b]{0.325\textwidth}
    \centering
    \textbf{Customer 3}
    \includegraphics[width=\textwidth, height=3cm]{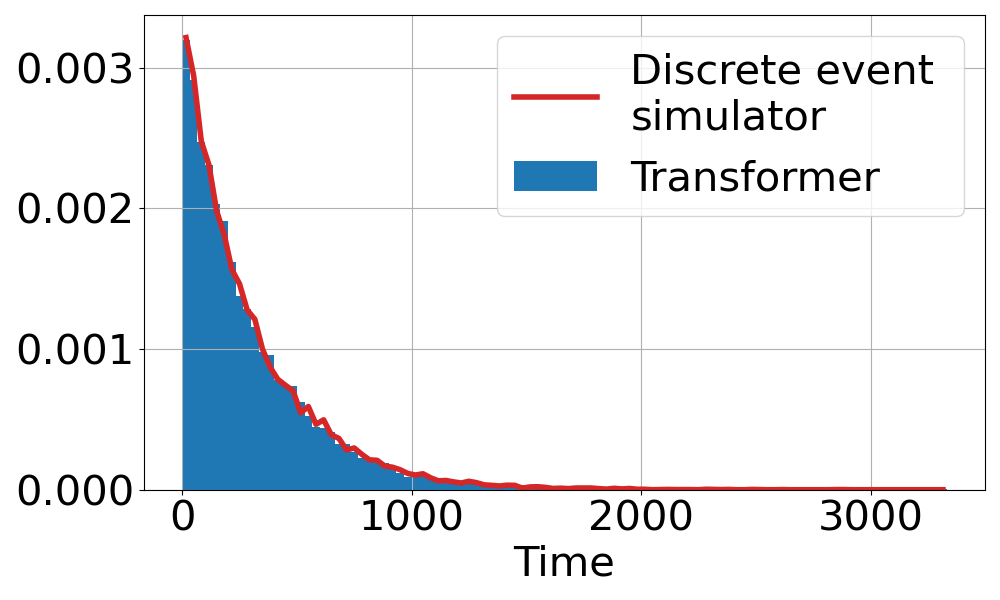}
    \end{minipage}
    
    \begin{minipage}[b]{0.325\textwidth} 
    \centering
    \textbf{Customer 4}
    \includegraphics[width=\textwidth, height=3cm]{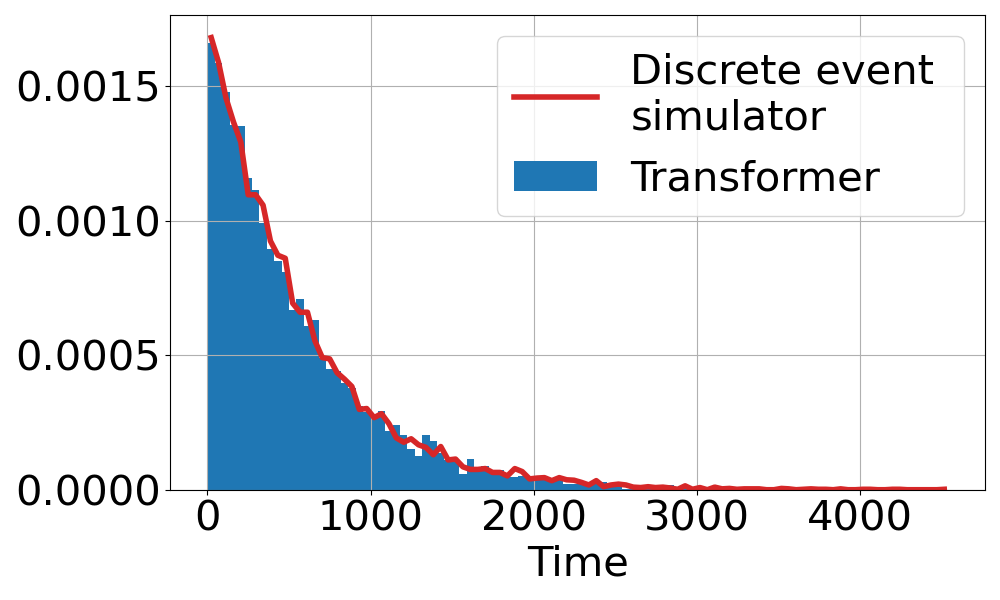}
    \end{minipage}
    \hfill
    \begin{minipage}[b]{0.325\textwidth} 
    \centering
    \textbf{Customer 5}
    \includegraphics[width=\textwidth, height=3cm]{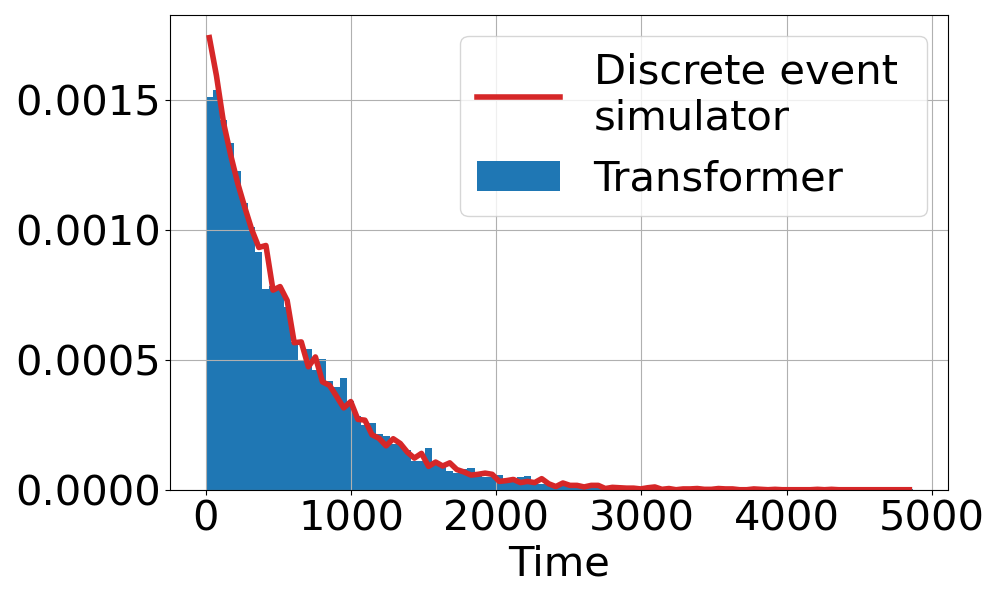}
    \end{minipage}
    \hfill
    \begin{minipage}[b]{0.325\textwidth}
    \centering
    \textbf{Customer 6}
    \includegraphics[width=\textwidth, height=3cm]{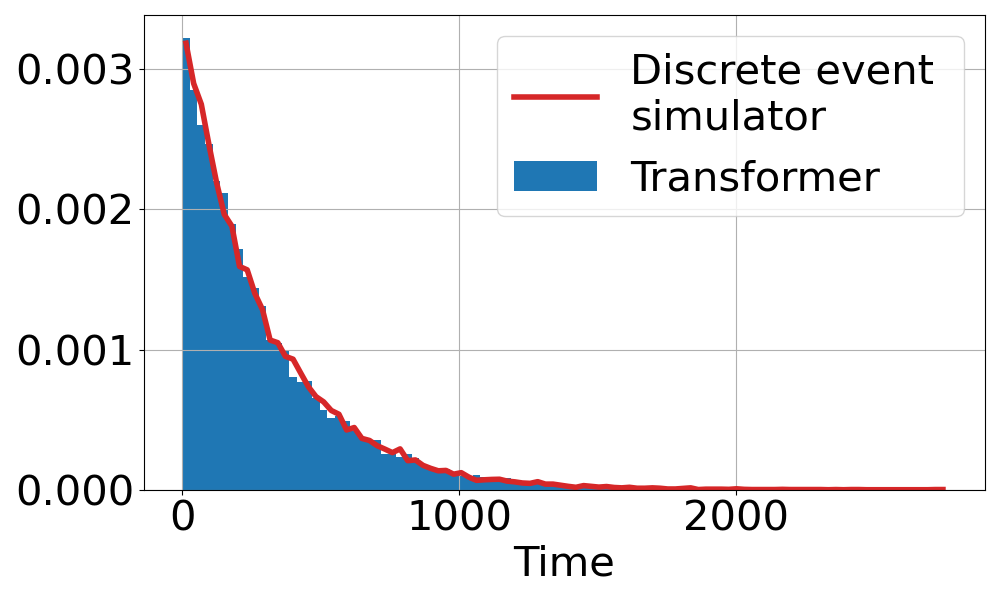}
    \end{minipage}
    \caption{\textbf{Call-center (Inter-arrival time distributions for different customers):} Comparison of average inter-arrival time distributions predicted by sequence model (transformer) and discrete event simulator (derived using the knowledge of underlying queuing network) for a call-center.}
    \label{fig:representative-average-arrival-time-distributions-see-lab}
\end{figure}

\begin{figure}[h!]
    \centering
    \begin{minipage}[b]{0.325\textwidth} 
    \centering
    \textbf{Customer 1}
    \includegraphics[width=\textwidth, height=3cm]{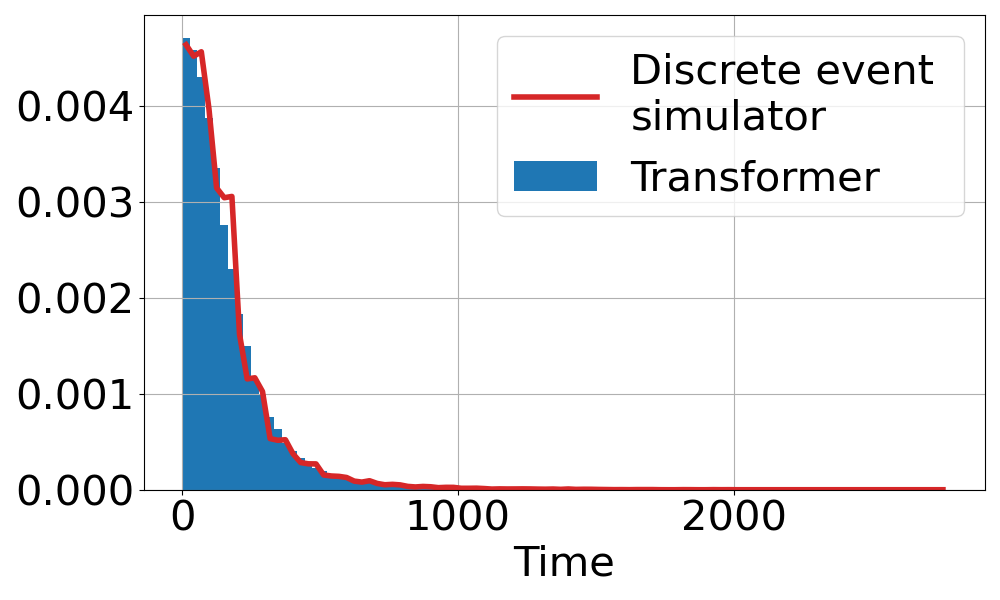}
    \end{minipage}
    \hfill
    \begin{minipage}[b]{0.325\textwidth}
    \centering
    \textbf{Customer 2}
    \includegraphics[width=\textwidth, height=3cm]{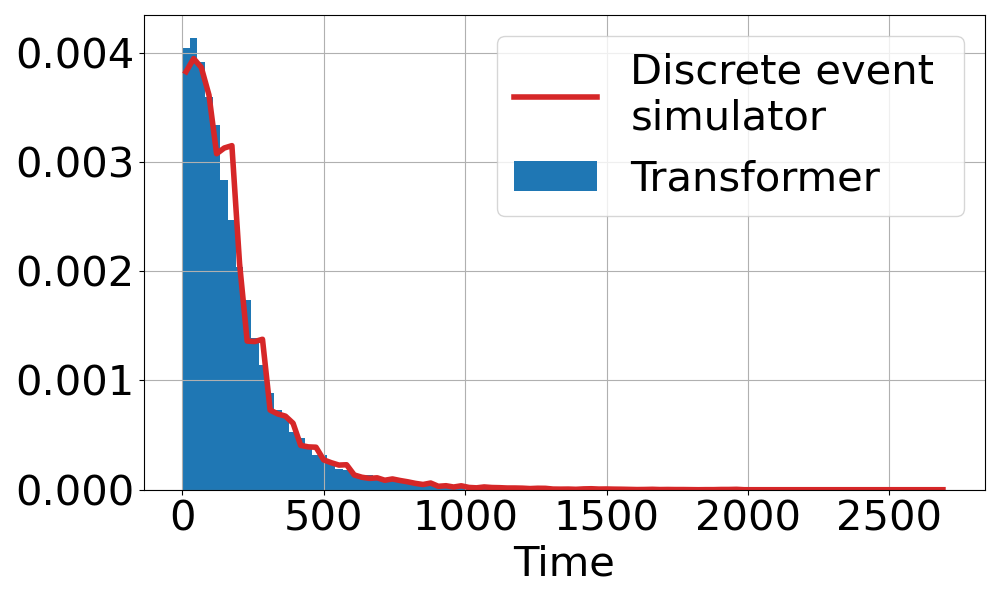}
    \end{minipage}
    \hfill
    \begin{minipage}[b]{0.325\textwidth}
    \centering
    \textbf{Customer 3}
    \includegraphics[width=\textwidth, height=3cm]{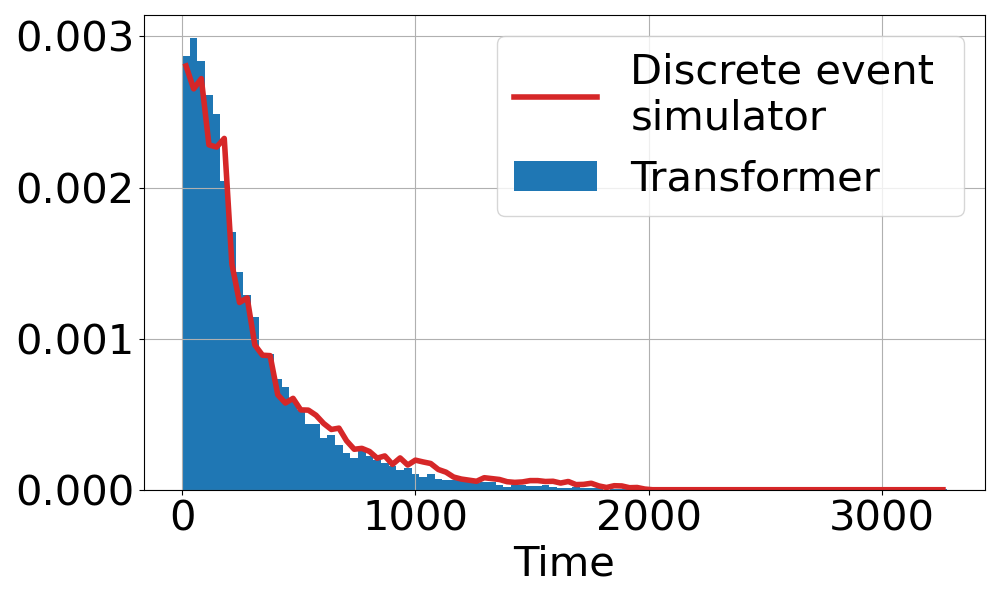}
    \end{minipage}
    
    \begin{minipage}[b]{0.325\textwidth}
    \centering
    \textbf{Customer 4}
    \includegraphics[width=\textwidth, height=3cm]{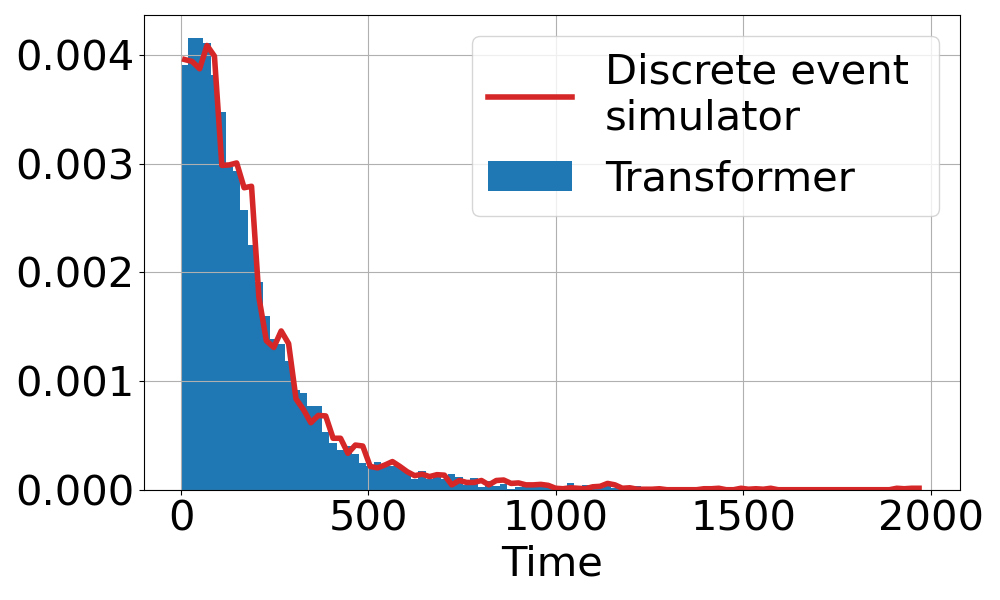}
    \end{minipage}
    \hfill
    \begin{minipage}[b]{0.325\textwidth} 
    \centering
    \textbf{Customer 5}
    \includegraphics[width=\textwidth, height=3cm]{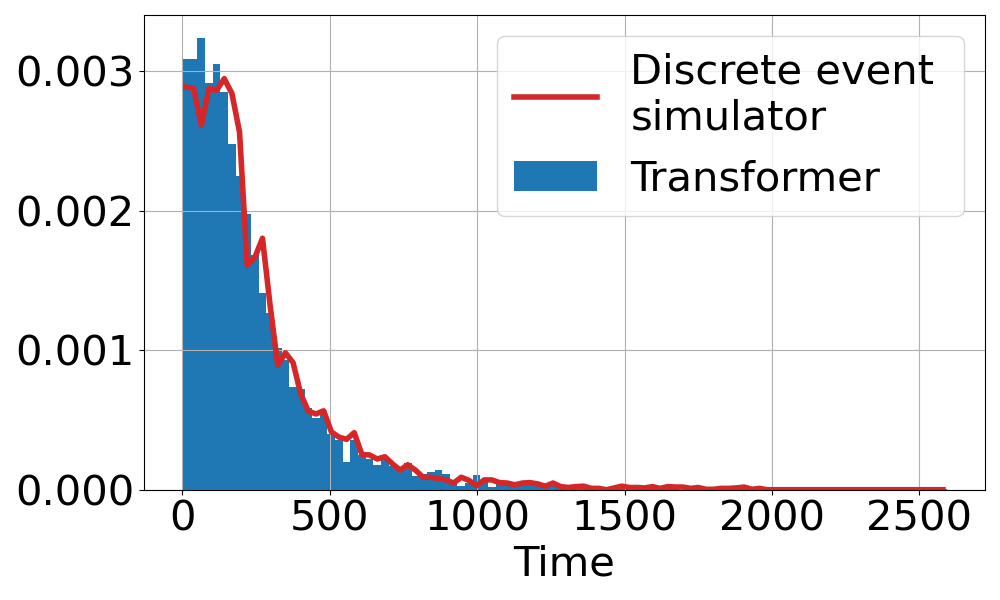}
    \end{minipage}
    \hfill
    \begin{minipage}[b]{0.325\textwidth}
    \centering
    \textbf{Customer 6}
    \includegraphics[width=\textwidth, height=3cm]{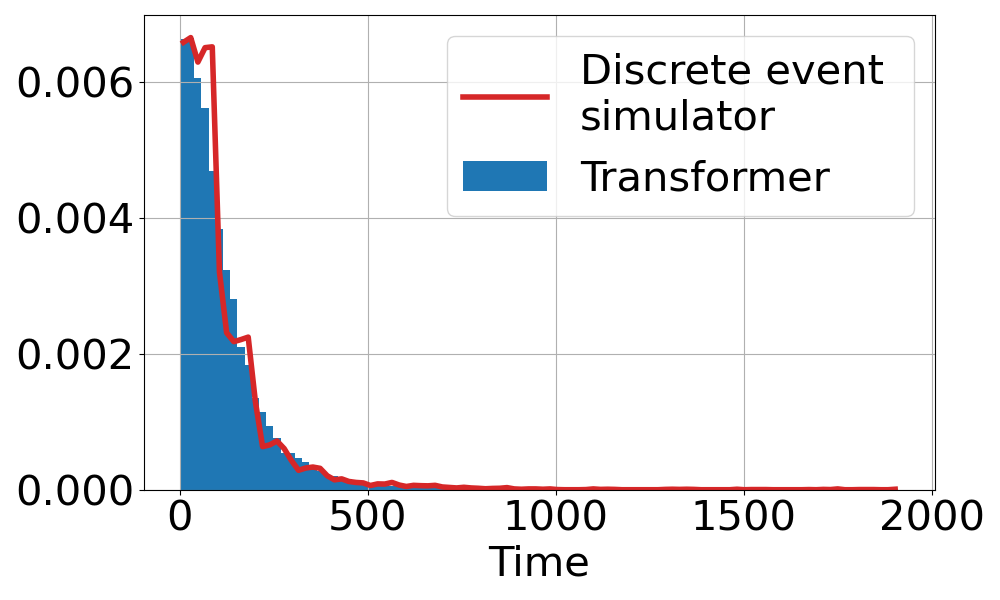}
    \end{minipage}
    \caption{\textbf{Call-center (Service time distributions for different customers):} Comparison of average service time distributions predicted by sequence model (transformer) and discrete event simulator (derived using the knowledge of underlying queuing network) for a call-center.}
    \label{fig:representative-average-service-time-distributions-see-lab}
\end{figure}

\begin{figure}[h!]
    \centering
    \begin{minipage}[b]{0.325\textwidth}
    \centering
    \textbf{Customer 1}
    \includegraphics[width=\textwidth, height=3cm]{tex/fig/fig_see_lab/waiting_customer_0.png}
    \end{minipage}
    \hfill
    \begin{minipage}[b]{0.325\textwidth} 
    \centering
    \textbf{Customer 2}
    \includegraphics[width=\textwidth, height=3cm]{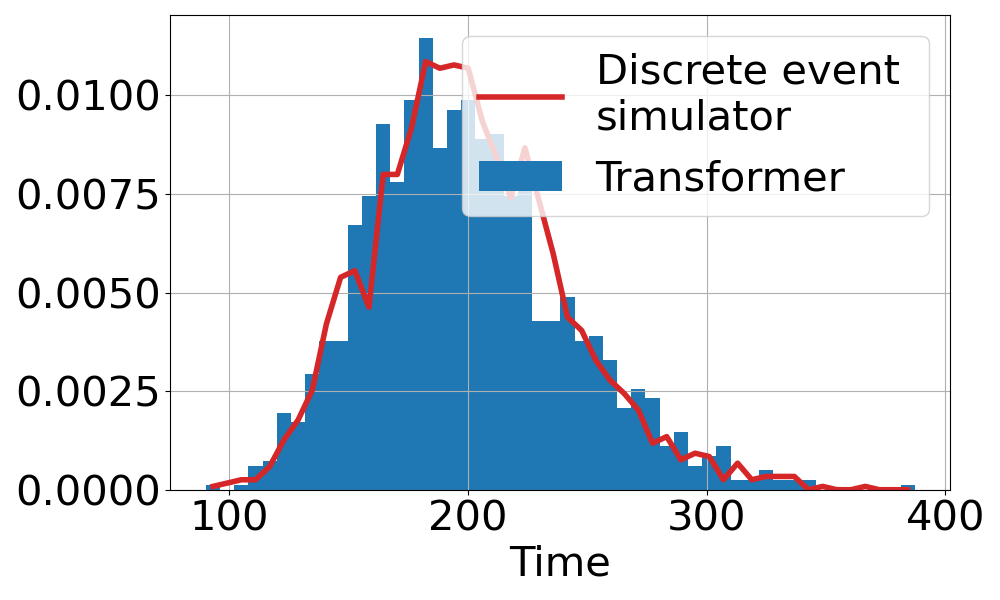}
    \end{minipage}
    \hfill
    \begin{minipage}[b]{0.325\textwidth} 
     \centering
    \textbf{Customer 3}
    \includegraphics[width=\textwidth, height=3cm]{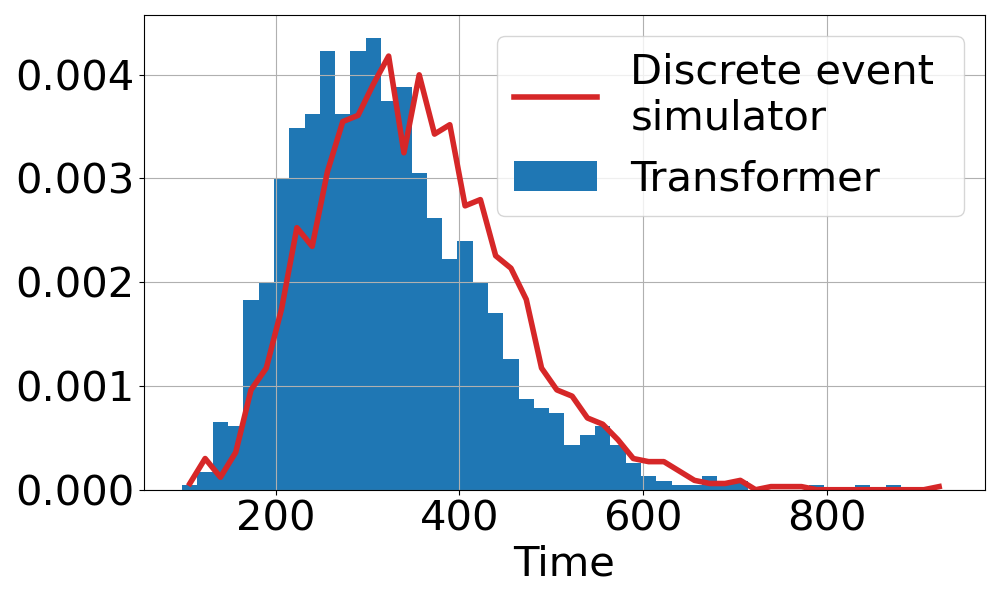}
    \end{minipage}
    
    \begin{minipage}[b]{0.325\textwidth} 
    \centering
    \textbf{Customer 4}
    \includegraphics[width=\textwidth, height=3cm]{tex/fig/fig_see_lab/waiting_customer_3.png}
    \end{minipage}
    \hfill
    \begin{minipage}[b]{0.325\textwidth}  
    \centering
    \textbf{Customer 5}
    \includegraphics[width=\textwidth, height=3cm]{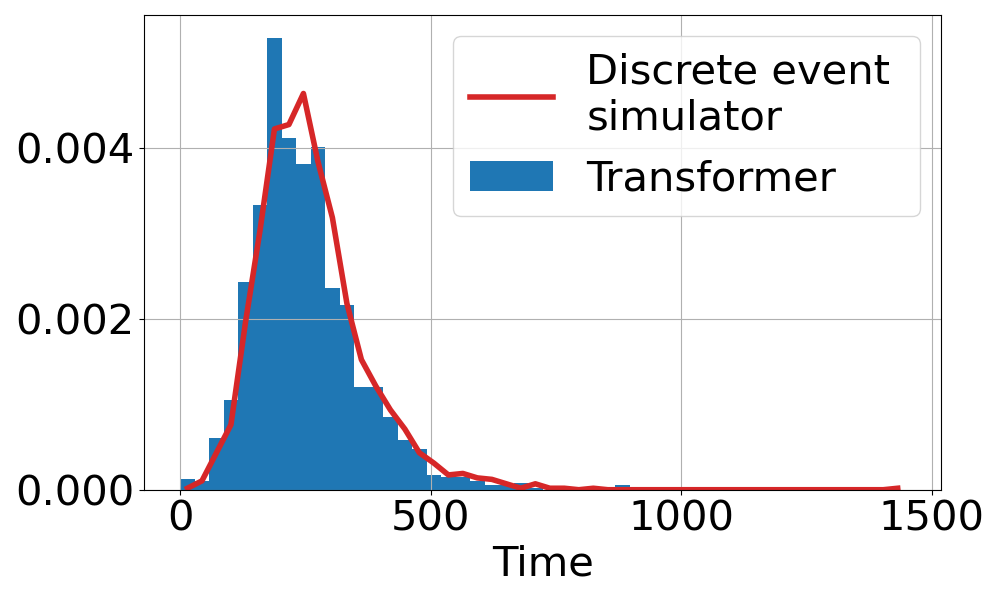}
    \end{minipage}
    \hfill
    \begin{minipage}[b]{0.325\textwidth}
    \centering
    \textbf{Customer 6}
    \includegraphics[width=\textwidth, height=3cm]{tex/fig/fig_see_lab/waiting_customer_5.png}
    \end{minipage}
    \caption{\textbf{Call-center (Average waiting time distributions for different customers):} Comparison of average waiting time distributions predicted by sequence model (transformer) and discrete event simulator (derived using the knowledge of underlying queuing network) for a call-center.}
    \label{fig:representative-average-waiting-time-distributions-see-lab-appendix}
\end{figure}

\subsection{Experimental details for uncertainty quantification (Section \ref{sec:uncertainty_q_experiments})}
\label{sec:Additional experiments:uncertainty}


For this experiment, we train a transformer model with a maximum sequence length of 400 events and 13,000 uniform time bins. The model is trained on 40,000 simulated trajectories of an M/M/1 queue, each generated with arrival rates $\lambda$ and service rates $\nu$ drawn independently from a uniform prior. After training, we generate 10,000 trajectories from the trained model and analyze the resulting performance-metric distributions. We repeat the procedure under several different uniform priors for  $\lambda$ and $\nu$.  Across all prior settings and history lengths, the Transformer yields consistent accuracy, as illustrated in Figures~\ref{fig:uq_distribution_appendix} and~\ref{fig:M-M-1-time-distributions in Uncertainty-2-4-3-5-kl}.

\begin{figure}[t]
    \centering
    \centering\textbf{$\lambda \sim \text{Uniform}(1.5,2.5)$, $\nu \sim \text{Uniform}(3,6)$, $n=100$,  $N-n=300$ }\\
    \begin{minipage}[b]{0.325\textwidth}   
      \includegraphics[width=\textwidth, height=3cm]{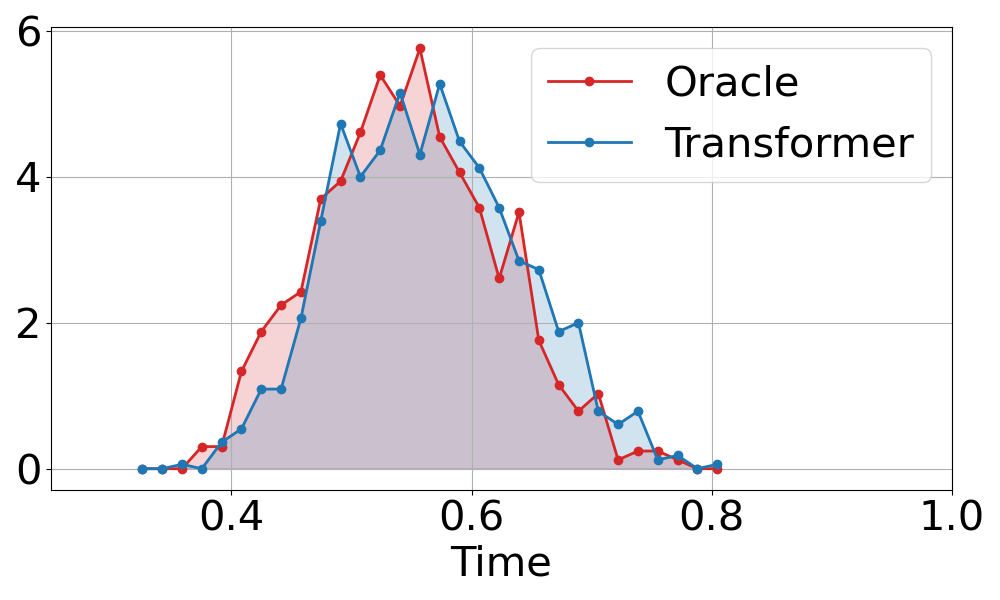}
    \end{minipage}
    \hfill
    \begin{minipage}[b]{0.325\textwidth}   
      \includegraphics[width=\textwidth, height=3cm]{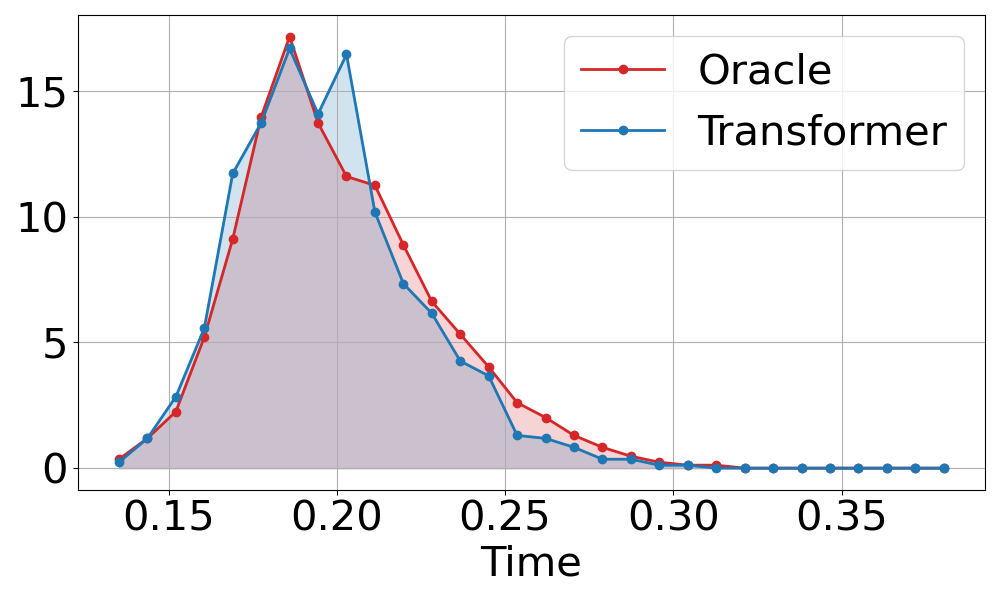}
    \end{minipage}
    \hfill
    \begin{minipage}[b]{0.325\textwidth}   
      \includegraphics[width=\textwidth, height=3cm]{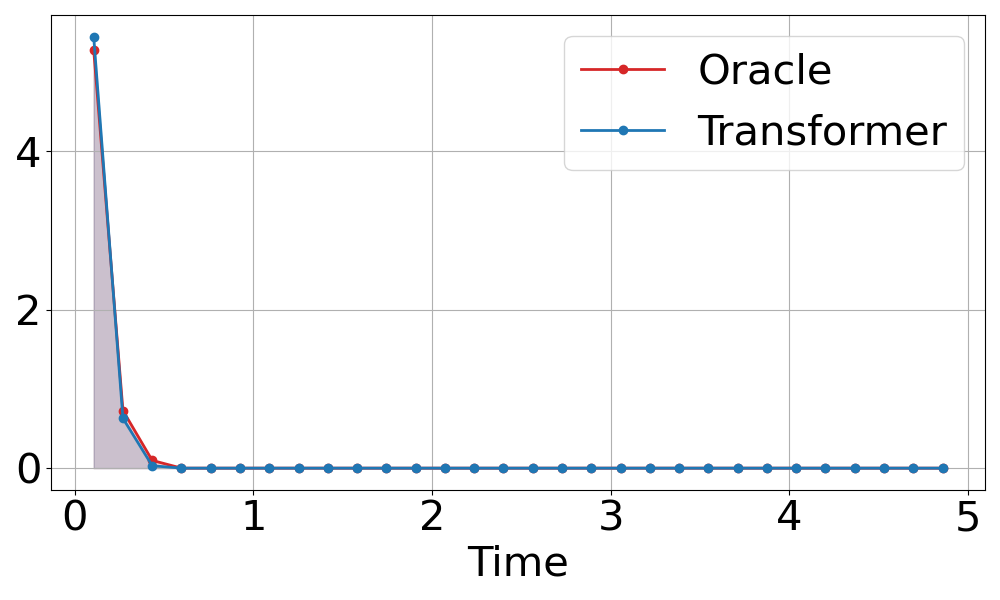}
    \end{minipage}

    \centering\textbf{$\lambda \sim \text{Uniform}(2,4)$, $\nu \sim \text{Uniform}(3,5)$, $n=100$,  $N-n=300$ }\\ 
    \begin{minipage}[b]{0.325\textwidth}   
      \includegraphics[width=\textwidth, height=3cm]{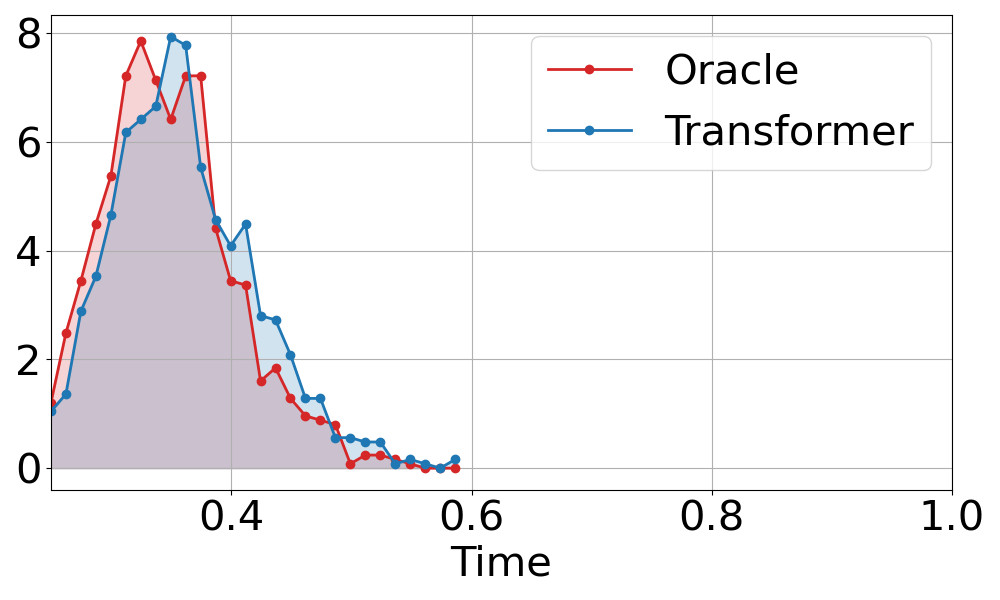}
    \end{minipage}
    \hfill
    \begin{minipage}[b]{0.325\textwidth}   
      \includegraphics[width=\textwidth, height=3cm]{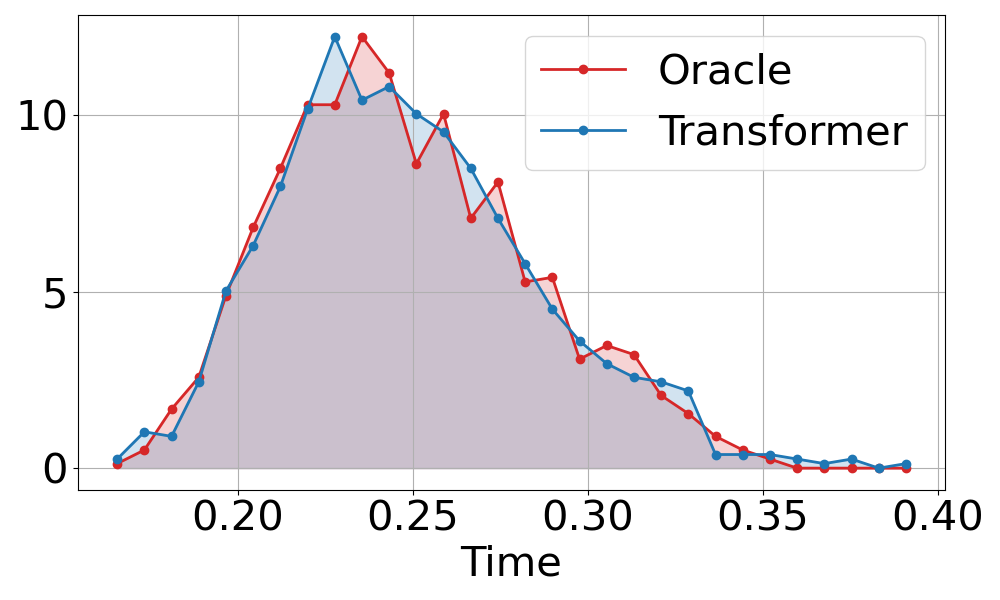}
    \end{minipage}
    \hfill
    \begin{minipage}[b]{0.325\textwidth}   
      \includegraphics[width=\textwidth, height=3cm]{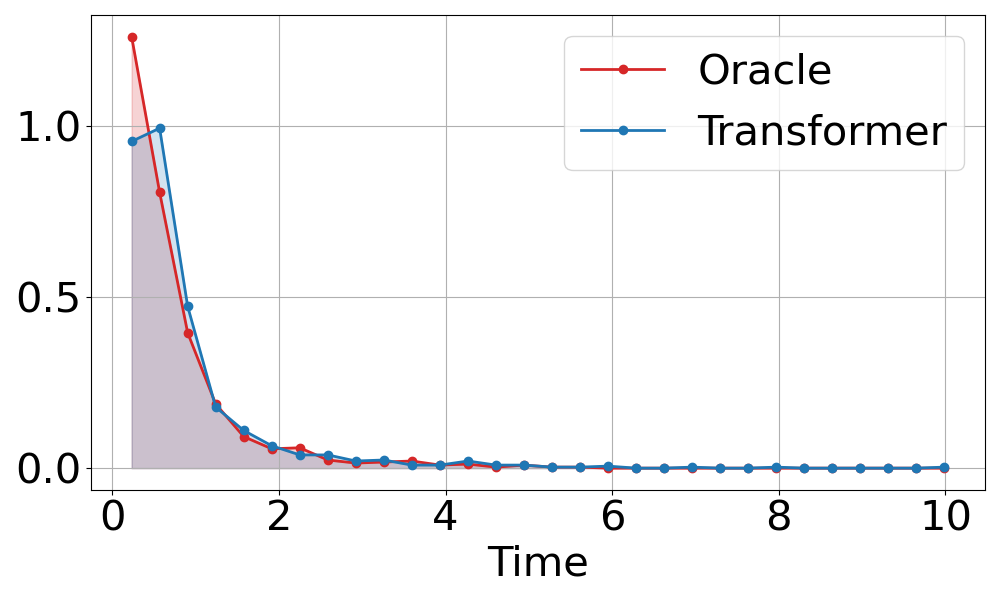}
    \end{minipage}

\centering\textbf{$\lambda \sim \text{Uniform}(2,4)$, $\nu \sim \text{Uniform}(3,5)$, $n=200$,  $N-n=200$ }\\
    \begin{minipage}[b]{0.325\textwidth}   
      \includegraphics[width=\textwidth, height=3cm]{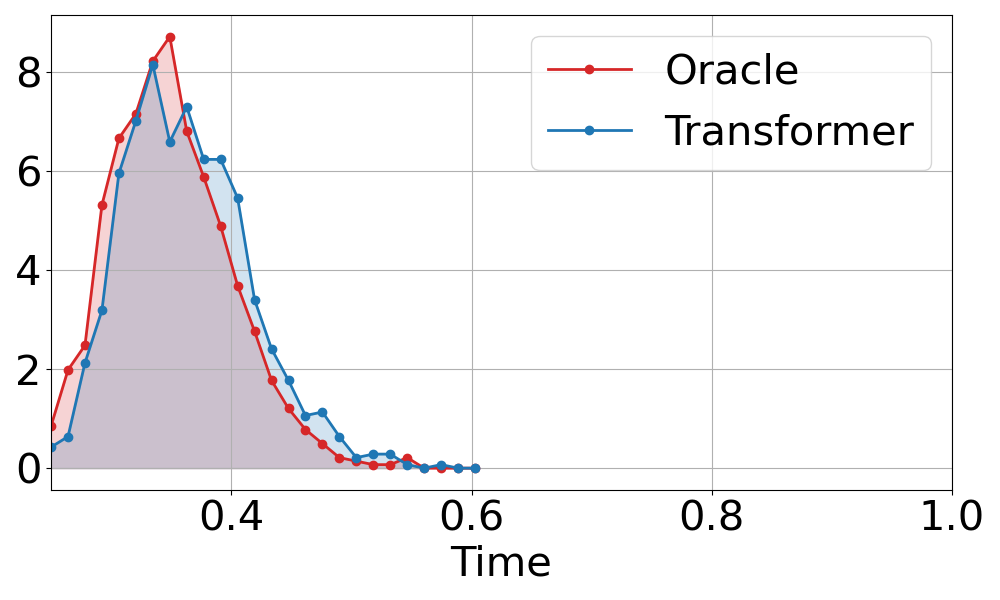}
      \centering\textbf{Inter-arrival time \\distribution}
    \end{minipage}
    \hfill
    \begin{minipage}[b]{0.325\textwidth}   
      \includegraphics[width=\textwidth, height=3cm]{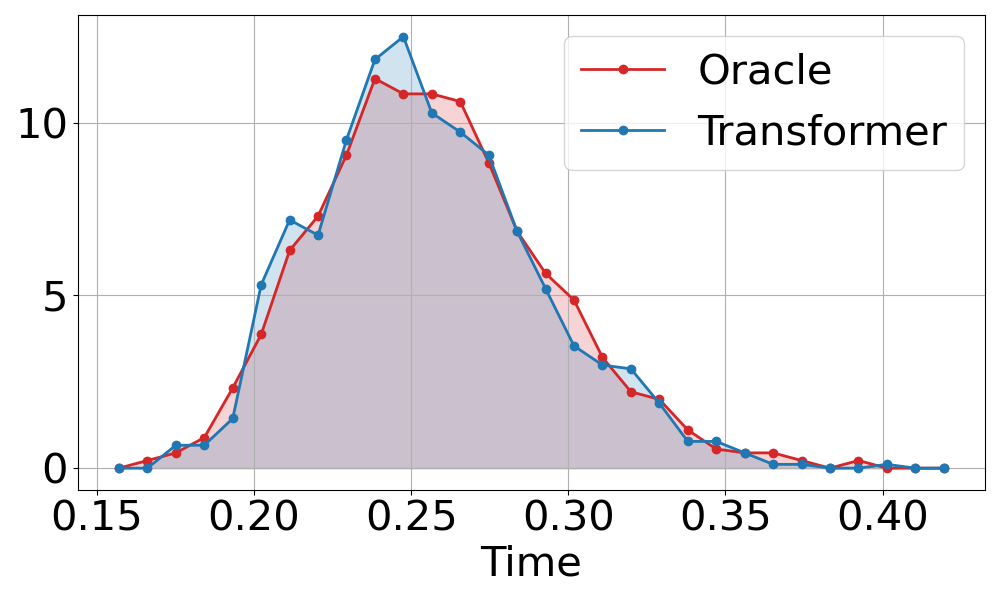}
      \centering\textbf{Service time \\distribution}
    \end{minipage}
    \hfill
    \begin{minipage}[b]{0.325\textwidth}   
      \includegraphics[width=\textwidth, height=3cm]{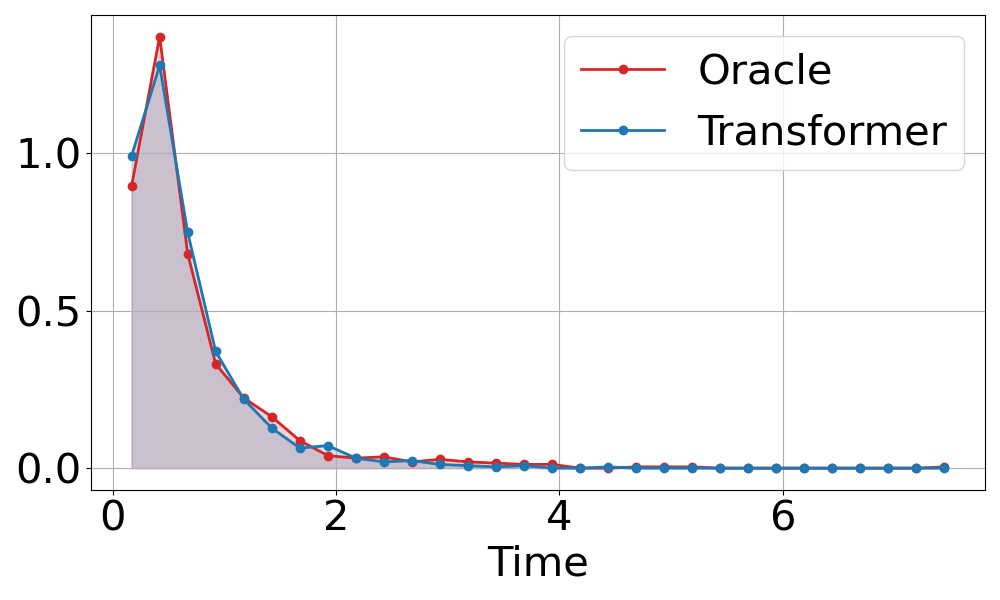}
      \centering\textbf{Waiting time \\distribution}
    \end{minipage}
\vspace{3mm}
    \caption{\textbf{Uncertainty Quantification for M/M/1 queue:} Comparison of performance measure distributions predicted by the sequence model (transformer) with those from the oracle.}
    \label{fig:uq_distribution_appendix}
  \end{figure}

  \begin{figure}[h!]
    \centering
    \centering\textbf{$\lambda \sim \text{Uniform}(1.5,2.5)$, $\nu \sim \text{Uniform}(3,6)$, $n=100$,  $N-n=300$ }\\
    \begin{minipage}[b]{0.325\textwidth}   
      \includegraphics[width=\textwidth, height=3.5cm]{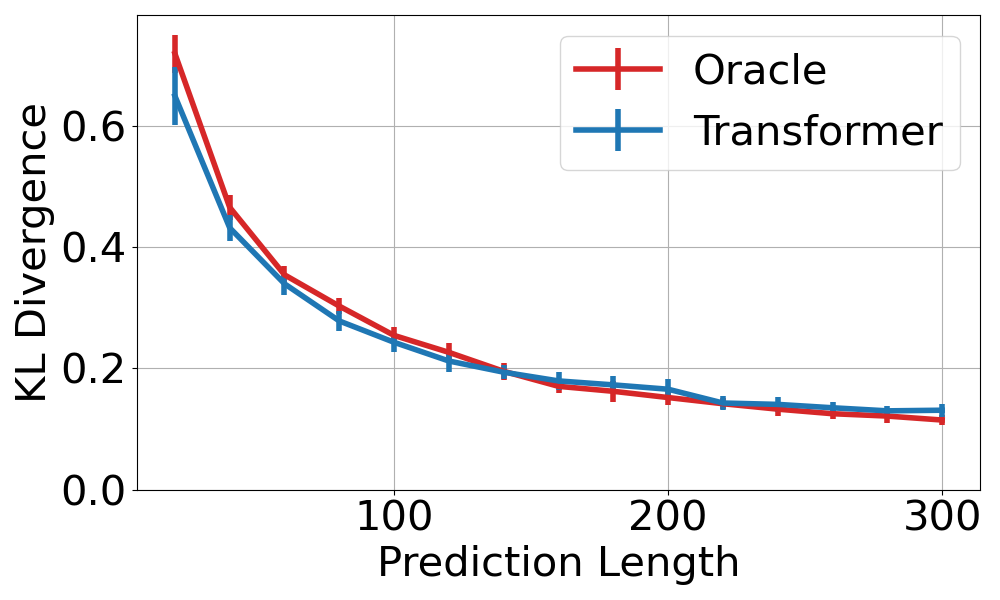}
    \end{minipage}
    \hfill
    \begin{minipage}[b]{0.325\textwidth}   
      \includegraphics[width=\textwidth, height=3.5cm]{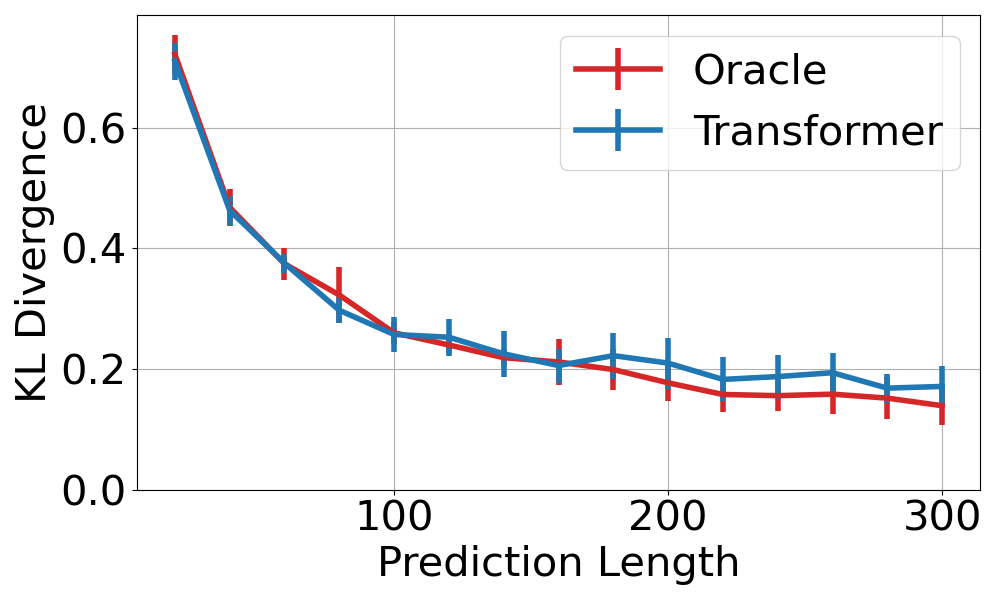}
    \end{minipage}
    \hfill
    \begin{minipage}[b]{0.325\textwidth}   
      \includegraphics[width=\textwidth, height=3.5cm]{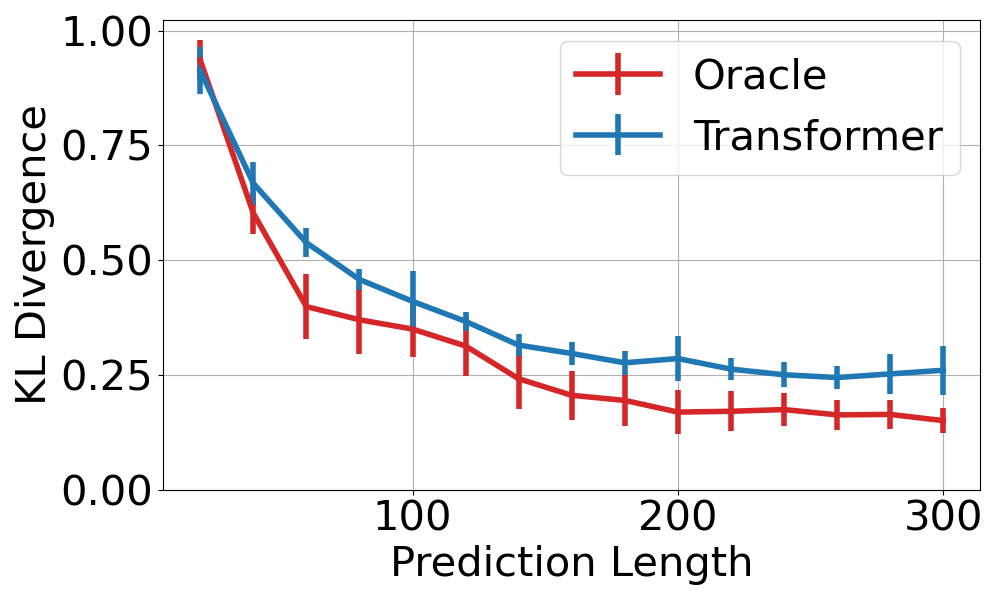}
    \end{minipage}

 \centering\textbf{$\lambda \sim \text{Uniform}(2,4)$, $\nu \sim \text{Uniform}(3,5)$, $n=100$,  $N-n=300$ }\\    
    \begin{minipage}[b]{0.325\textwidth}   
      \includegraphics[width=\textwidth, height=3.5cm]{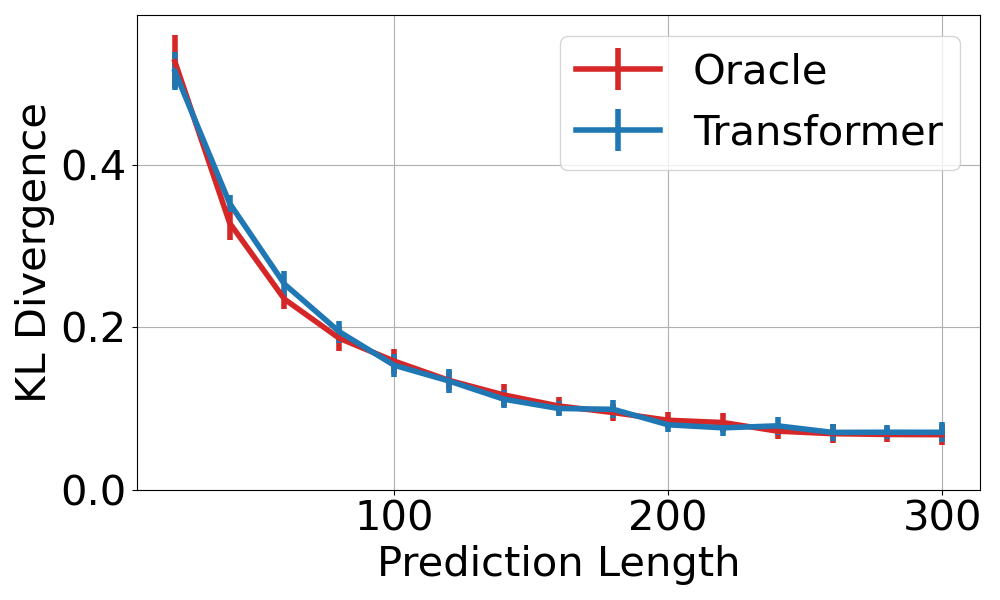}
    \end{minipage}
    \hfill
    \begin{minipage}[b]{0.325\textwidth}   
      \includegraphics[width=\textwidth, height=3.5cm]{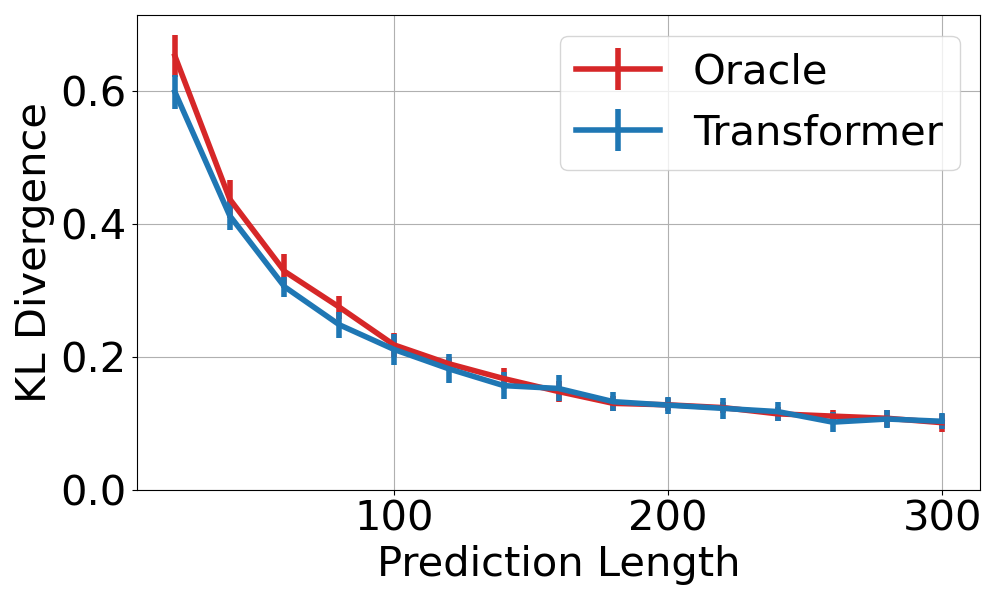}
    \end{minipage}
    \hfill
    \begin{minipage}[b]{0.325\textwidth}   
      \includegraphics[width=\textwidth, height=3.5cm]{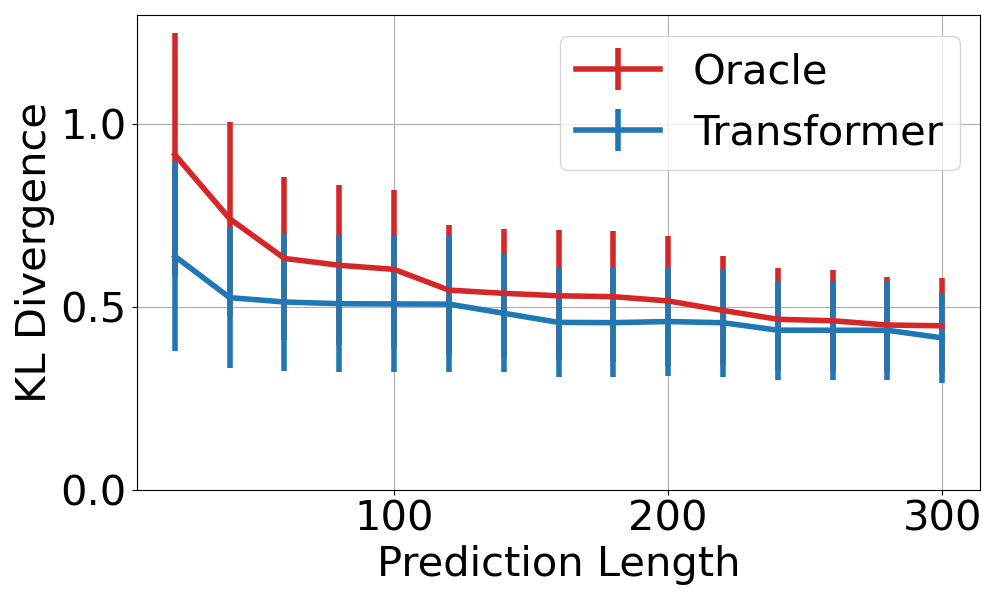}
    \end{minipage}

     \centering\textbf{$\lambda \sim \text{Uniform}(1.5,2.5)$, $\nu \sim \text{Uniform}(3,6)$, $n=100$,  $N-n=300$ }\\
    \begin{minipage}[b]{0.325\textwidth}   
      \includegraphics[width=\textwidth, height=3.5cm]{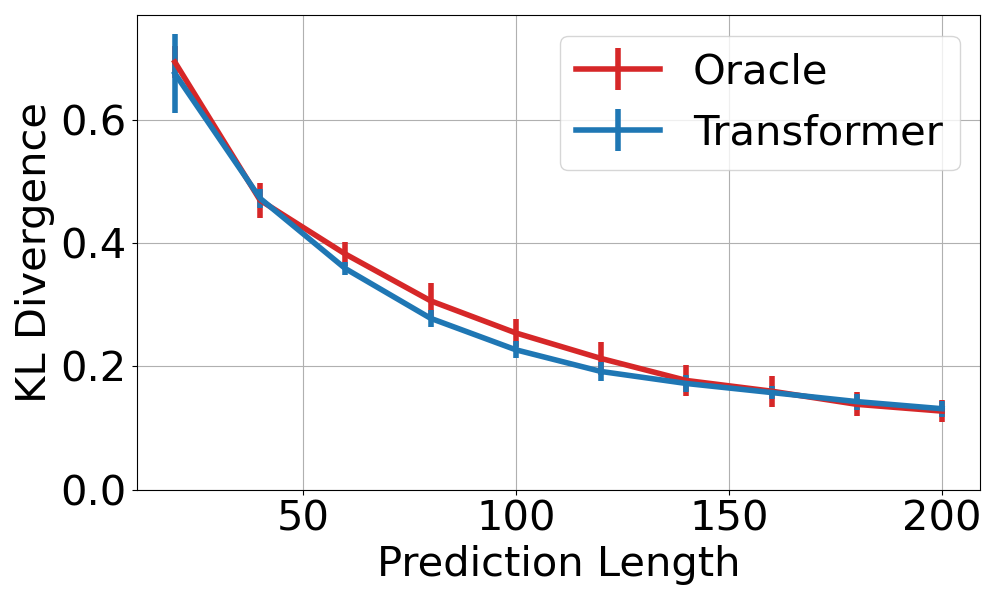}
      \centering\textbf{KL div. of inter-arrival time distribution}
    \end{minipage}
    \hfill
    \begin{minipage}[b]{0.325\textwidth}   
      \includegraphics[width=\textwidth, height=3.5cm]{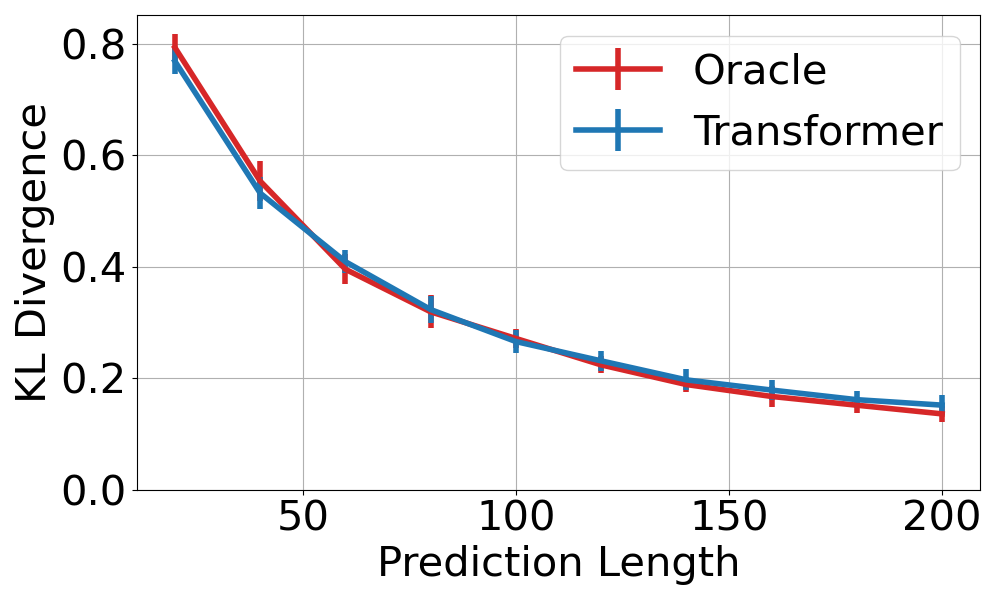}
      \centering\textbf{KL div. of service time distribution}
    \end{minipage}
    \hfill
    \begin{minipage}[b]{0.325\textwidth}   
      \includegraphics[width=\textwidth, height=3.5cm]{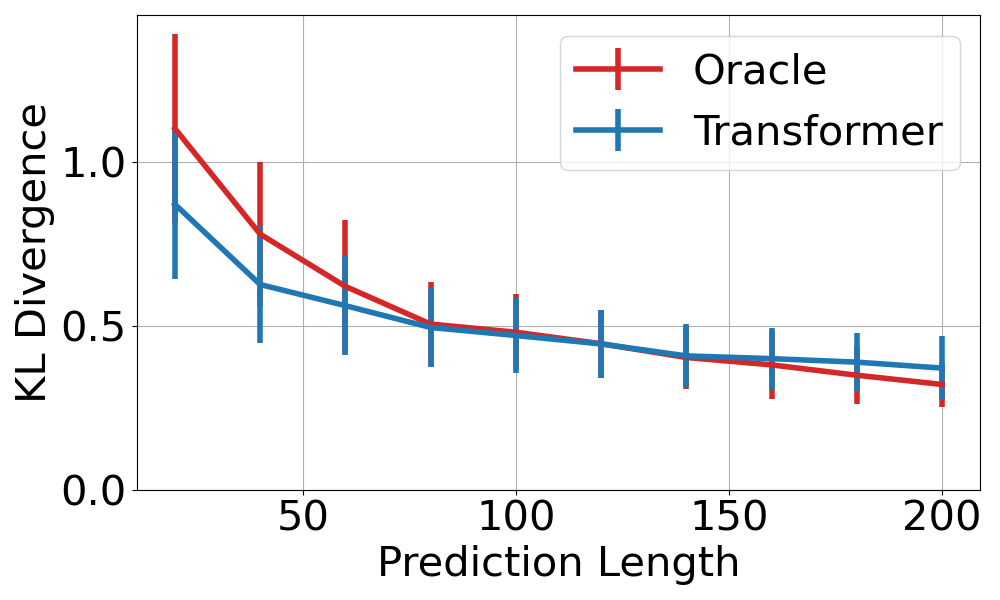}
      \centering\textbf{KL div. of waiting time distribution}
    \end{minipage}
\vspace{3mm}
    \caption{\textbf{Uncertainty Quantification for M/M/1 queue:} \textbf{Uncertainty Quantification:}  Comparison of (1) KL divergence between the Transformer (finite horizon $N$) and oracle benchmark (infinite horizon $N=\infty$) and (2) KL divergence between the oracle benchmark at finite horizon ($N$) and  infinite horizon. }
    \label{fig:M-M-1-time-distributions in Uncertainty-2-4-3-5-kl}
  \end{figure}

\subsection{Experimetal details for Simulating Counterfactuals (Section \ref{sec:simulating_counterfactuals})}
\label{sec:additional_experiments-counterfactuals}


This synthetic yet challenging scenario serves as a proxy for hospital operations, where patient arrivals fluctuate sharply over the course of a day.  We model a facility staffed by between two and twenty nurses, each working at a constant rate of 3.5 patients per hour.
Hourly arrivals follow the profile $[8,8,8,8,8,14,15,16,17,18,19,18]$ patients per hour. 

The Transformer is configured for sequences of up to 400 events, with time recorded in minutes and discretised into uniform 0.01-minute bins. Training is performed on 40,000 trajectories generated from the non-stationary arrival process described above.

\end{document}